\documentclass[11pt]{article}

\usepackage{subfigure}
\usepackage{float}
\usepackage[all]{xy}

\usepackage{amsmath,amssymb,amsthm}
\usepackage{latexsym}
\usepackage[dvipdfmx]{graphicx}
\usepackage[ruled]{algorithm2e}
\usepackage{algorithmic} 
\usepackage{authblk}
\usepackage{color}
\usepackage{relsize}
\usepackage{url}
\usepackage[colorlinks,linkcolor=blue,anchorcolor=green,citecolor=green]{hyperref}
\usepackage{caption} 
\usepackage{tikz}
\usepackage{yhmath}
\usepackage{mathrsfs,booktabs}

\usepackage[T1]{fontenc}
\usepackage[utf8]{inputenc}
\usepackage[font=small,labelfont=bf,tableposition=top]{caption}
\usepackage{booktabs}
\usepackage{threeparttable} 
\usepackage{multirow}

\makeatletter

\newcommand{\Rmnum}[1]{\expandafter\@slowromancap\romannumeral #1@}
\makeatother

\newtheorem{theorem}{Theorem}
\newtheorem{corollary}{Corollary}
\newtheorem{lemma}{Lemma}

\newtheorem{proposition}{Proposition}
\newtheorem{definition}{Definition}

\newtheorem{assumption}{Assumption}
\usepackage{color}

\usepackage{tikz,pgfplots}
\pgfplotsset{compat=newest}
\pgfplotsset{plot coordinates/math parser=false,trim axis left}
\newlength\figureheight
\newlength\figurewidth

\newcommand{\eins}{\boldsymbol{1}}
\DeclareMathOperator{\sign}{sign}

\newcommand{\argmax}{\operatornamewithlimits{arg \, max}}
\newcommand{\argmin}{\operatornamewithlimits{arg \, min}}

\pgfplotsset{
standard/.style={
	axis x line=middle,
	axis y line=middle,
	enlarge x limits=0.15,
	enlarge y limits=0.15,
	every axis x label/.style={at={(current axis.right of origin)},anchor=north west},
	every axis y label/.style={at={(current axis.above origin)},anchor=north east}
}
}

\author[1]{Hanyuan Hang}
\author[2]{Yuchao Cai}
\author[2]{Hanfang Yang}
\author[3]{Zhouchen Lin}

\date{\today}

\affil[1]{Department of Applied Mathematics, University of Twente, The Netherlands
}
\affil[2]{School of Statistics, Renmin University, China 
}
\affil[3]{School of EECS, Peking University, China
}

\begin{document}

\title{Under-bagging Nearest Neighbors for Imbalanced Classification}

%\author{\name Hanyuan Hang \email h.hang@utwente.nl \\
%	\addr Department of Applied Mathematics \\
%	University of Twente \\
%	7522 NB Enschede, The Netherlands
%}

%\editor{ }

\allowdisplaybreaks

\maketitle

%\tableofcontents
%\newpage

\begin{abstract}In this paper, we propose an ensemble learning algorithm called \textit{under-bagging $k$-nearest neighbors} (\textit{under-bagging $k$-NN}) for imbalanced classification problems. On the theoretical side, by developing a new learning theory analysis, we show that with properly chosen parameters, i.e., the number of nearest neighbors $k$, the expected sub-sample size $s$, and the bagging rounds $B$, optimal convergence rates for under-bagging $k$-NN can be achieved under mild assumptions w.r.t.~the arithmetic mean (AM) of recalls. Moreover, we show that with a relatively small $B$, the expected sub-sample size $s$ can be much smaller than the number of training data $n$ at each bagging round, and the number of nearest neighbors $k$ can be reduced simultaneously, especially when the data are highly imbalanced, which leads to substantially lower time complexity and roughly the same space complexity. On the practical side, we conduct numerical experiments to verify the theoretical results on the benefits of the under-bagging technique by the promising AM performance and efficiency of our proposed algorithm.
\end{abstract}

%\begin{keywords}
%	Imbalanced classification, 
%	under-sampling, 
%	bagging, 
%	$k$-nearest neighbors,
%	ensemble learning,
%	arithmetic mean measure, 
%	optimal convergence rates, 
%	learning theory
%\end{keywords}

\section{Introduction} \label{sec::intro}

Imbalanced classification has been encountered in multiple areas such as telecommunication managements \cite{babu2018enhanced, aljanabi2018intelligent}, bioinformatics \cite{bugnon2019deep, shanab2020optimally}, fraud detection \cite{li2021hybrid, somasundaram2019parallel}, and medical diagnosis \cite{zhao2020intelligent}, and has been considered one of the top ten problems in data mining research \cite{yang200610, rastgoo2016tackling}. In fact, the ratio of the size of the majority class to the minority class can be as huge as $10^6$ \cite{wu2008fast}. It is noteworthy that the imbalanced classification is emerging as an important issue in designing classifiers \cite{weiss2004mining, johnson2020effects}. Two observations account for this point: on the one hand, the imbalanced classification is pervasive in a large number of domains of great importance in the machine learning community, on the other hand, most popular classification learning algorithms are reported to be inadequate when encountering the imbalanced classification problem. These classification algorithms involve $k$-nearest neighbors \cite{wang2021review}, support vector machines \cite{raskutti2004extreme,pisner2020support}, random forest \cite{chawla2004special, sohony2018ensemble}, and neural networks \cite{johnson2019survey}. We refer the reader to \cite{sun2009classification, guo2017learning} for a general review on imbalanced classification.

$k$-nearest neighbors ($k$-NN) is included in the top ten most significant data mining algorithms \cite{wu2008top}, as it has always been preferred for its non-parametric working principle, and ease of implementation \cite{duda2001pattern}. To be specific, $k$-NN is a lazy method without model training since it simply tags the new data entry based learning from historical data. On the other hand, $k$-NN is known to be a simple yet powerful non-parametric machine learning algorithm widely used in fields such as genetics \cite{ayyad2019gene, arowolo2020pca}, data compression \cite{salvador2019compressed, deng2016efficient}, economic forecasting \cite{malini2017analysis, moldagulova2017using}, recommendation and rating prediction \cite{patro2020hybrid, al2021location}. In fact, companies like Amazon and Netflix use $k$-NN when recommending books to buy and movies to watch \cite{ahuja2019movie, belacel2020k}.

However, one of the major drawbacks of $k$-NN is that it uses only local prior probabilities to predict instance labels, ignoring the class distribution, and hence gives preference to the majority class for prediction, which results in undesirable performance on imbalanced data sets. Constant research is going on to cope with the shortcomings resulting in various improvements of the $k$-NN algorithm, which can be roughly grouped into four categories: resampling techniques, algorithm modification, cost-sensitive learning approaches, and ensemble learning methods. Resampling techniques aim to rebalance the training dataset by means of some mechanisms to generate a more balanced class distribution which is suitable for standard classifiers, see e.g., \cite{wilson1972asymptotic, zhang2003knn, kurniawati2018adaptive, nwe2019knn}. Different from resampling techniques, algorithm modification methods try to adjust the structure of standard classifiers to diminish the effect caused by class imbalance. There has been a flurry of work to design specific learning algorithms with $k$-NN for imbalanced classification, see e.g., \cite{tan2005neighbor, wang2008effective, li2011improving,dubey2013class,zhang2017krnn,liu2018improved, mullick2018adaptive,yuan2021novel}. Besides, cost-sensitive learning generally considers a higher cost for the minority class to compensate for the scarcity of minority data. Recently, \cite{zhang2020cost} for the first time designed two efficient cost-sensitive $k$-NN classification models by changing the distance function of the standard $k$-NN classification. As for the ensemble learning methods adapt to imbalance problem, \cite{guo2016bpso} implemented the powerful boosting framework Adaboost as the learning model, in which $k$-NN is chosen as the base classifier and boosting-by-resample method is used to generate the training set. However, these methods suffer from problems including the involvement of exhaustive search, the introduction of new parameters, and significant computational overhead, which hinder the scalability and easy implementation of $k$-NN. In addition, the effectiveness of these methods has not been investigated from a theoretical perspective.

In this study, we propose an ensemble learning algorithm named \textit{under-bagging $k$-nearest neighbors}, where the drawback of the standard $k$-NN classifier for imbalanced classification is eliminated with the help of the under-bagging technique. More precisely, the under-bagging $k$-NN classifier creates an ensemble of base predictors over bootstrap training samples independently drawn following the \textit{under-sampling} rule. In other words, at each \textit{bagging} round, we sample several subsets independently from the majority class such that the expected size of each subset equals to that of the minority class. After that, we build a $k$-NN classifier on the newly created \textit{balanced} training set. The under-bagging $k$-NN classifier is thus derived based on the averaged posterior probability function. It is worth pointing out that the under-bagging $k$-NN classifier enjoys three advantages.
Firstly, it preserves the simplicity of standard $k$-NN as a lazy learner without building any discriminative function from the training data, compared with other model-based classification algorithms. Secondly, the under-bagging $k$-NN classifier is an effective method targeting at imbalanced classification by drawing relatively balanced data subsets used at each bagging round. Last but not least, the bagging technique helps to improve the computational efficiency greatly by reducing the number of training samples at bagging rounds, while the standard $k$-NN is problematic to keep all the training examples in memory, so as to search for all the $k$ nearest neighbors for a test sample \cite{zhang2020cost}.

The main purpose of this paper is to conduct a theoretical analysis on the under-bagging $k$-NN for imbalanced classification from a learning theory perspective \cite{cucker2007learning,steinwart2008support}. In the analysis, we adopt the widely used AM measure, which is shown to be more general and powerful for imbalanced classification. In fact, as \cite{menon2013statistical} pointed out, since the Bayes classifier w.r.t.~the AM measure usually differs from that w.r.t.~the classification loss, standard algorithms expecting balanced class distribution are not consistent in terms of the AM measure. Therefore, we investigate the under-bagging $k$-NN classifier w.r.t.~the AM measure and demonstrate its effectiveness and efficiency. On the other hand, since bagging is known to reduce the variance of the base learners, in our analysis, we use Bernstein's concentration inequality \cite{massart2007concentration, vandervaart1996weak, Kosorok2008introduction} to establish the convergence rates, which allow for localization due to their specific dependence on the variance \cite{hang2017bernstein,hang2016learning}. In this way, our results hold in sense of ``with high probability'', which is more closely related to practical needs than ``in expectation'' and ``in probability'' addressed commonly in existing statistical analysis, see e.g., \cite{hall2005properties,biau2010rate,biau2015lectures}.

The contributions of this paper can be stated as follows.

\textit{(i)} We present the learning theory analysis on the under-sampling and under-bagging $k$-NN for imbalanced multi-class classification w.r.t.~the AM measure. We mention that the learning theory approach distinguishes our work from previous studies. Under the H\"{o}lder smoothness assumption and the margin condition, optimal convergence rates of both under-sampling $k$-NN and under-bagging $k$-NN are established w.r.t.~the AM measure with high probability. It is worth pointing out that our finite sample results demonstrate the explicit relationship among bagging rounds $B$, the number of nearest neighbors $k$, and the expected under-sampling size $s$.

\textit{(ii)} We conduct analysis on both time and space complexity of the under-bagging $k$-NN with parameter selection as shown in Theorem \ref{thm::bagknnclassund}. We show that under roughly the same space complexity, with the help of under-bagging, the time complexity of construction can be reduced from $\mathcal{O}(n \log n)$ (for the standard $k$-NN) to $\mathcal{O}((\rho n \log (\rho n))^{d/(2\alpha+d)})$, and the time complexity in the testing stage can be reduced from $\mathcal{O}((n \log n)^{2\alpha/(2\alpha+d)})$ to $\mathcal{O}(\log^2 (\rho n))$, where $\rho$ represents the imbalance ratio. These results indicate that under-bagging helps to enhance the computational efficiency, especially when the data are highly imbalanced.

\textit{(iii)} We conduct numerical experiments to verify the theoretical results. We first verify the relationship among parameters $k$, $s$, and $B$ in Theorem \ref{thm::bagknnclassund} based on synthetic datasets. Then on real datasets, we compare our under-bagging $k$-NN with the standard $k$-NN and the under-sampling $k$-NN. The results show that the under-bagging $k$-NN enjoys higher AM performance on imbalanced datasets. Moreover, the under-bagging technique can significantly reduce the running time of $k$-NN classifiers, which verifies our analysis on the time complexity of under-bagging $k$-NN. Finally, we can further reduce the running time by using fewer sub-samples with still competitive AM performance.

The rest of this paper is organized as follows. Section \ref{sec::preliminary} is a warm-up section for the introduction of some notations and definitions that are related to multi-class imbalanced classification. Then we propose the under-sampling and under-bagging $k$-NN for imbalanced classification in Section \ref{sec::prelim}. We provide basic assumptions and our main results on the convergence rates of both the under-sampling and under-bagging $k$-NN classifiers in Section \ref{sec::mainresults}. Some comments and discussions concerning the main results will also be provided in this section. The analysis on bounding error terms is presented in Section \ref{sec::erroranalysis}. We conduct numerical experiments to verify our theoretical findings for the under-bagging $k$-NN classifier and show the improvements of the under-bagging technique in Section \ref{sec::experiments}. All the proofs of Sections \ref{sec::mainresults} and \ref{sec::erroranalysis} can be found in Section \ref{sec::proofs}.

\section{Preliminaries} \label{sec::preliminary}

\subsection{Notations}

For $1 \leq p < \infty$, the $L_p$-norm of $x = (x_1, \ldots, x_d)$ is defined as $\|x\|_p := (|x_1|^p + \ldots + |x_d|^p)^{1/p}$, and the $L_{\infty}$-norm is defined as $\|x\|_{\infty} := \max_{i = 1, \ldots, d} |x_i|$. For any $x \in \mathbb{R}^d$ and $r > 0$, we denote $B_r(x) := B(x,r) := \{ x' \in \mathbb{R}^d : \|x' - x\|_2 \leq r \}$ as the closed ball centered at $x$ with radius $r$. We use the notation $a_n \lesssim b_n$ and $a_n=\mathcal{O}(b_n)$ to denote that there exists a positive constant $c$ independent of $n$ such that $a_n \leq c b_n$, for all $n \in \mathbb{N}$. Similarly, $a_n \gtrsim b_n$ denotes that there exists some positive constant $c \in (0, 1)$ such that $a_n \geq c^{-1} b_n$. 
%In addition, the notation $a_n \asymp b_n$ means that there exists some positive constant $c \in (0, 1)$, such that $a_n \geq c b_n$ and $a_n \leq c^{-1} b_n$, for all $n \in \mathbb{N}$. 
Finally, for a set $A \subset \mathbb{R}^d$, the cardinality of $A$ is denoted by $\#(A)$ and the indicator function on $A$ is denoted by $\eins_A$ or $\eins \{ A \}$.

\subsection{Imbalanced Classification} \label{sec::multiclassimbalance}

For a classification problem with $M$ classes, we observe data points $(x, y) \in \mathcal{X} \times \mathcal{Y}$ from an unknown distribution $\mathrm{P}$, where $x$ denotes the feature vector, $y$ is the corresponding label. Assume that $\mathcal{X} \subset \mathbb{R}^d$ and $\mathcal{Y} \subset [M] := \{ 1, \ldots, M \}$. Given $n$ independently observations $D_n := \{ (X_i, Y_i) : i = 1, \ldots, n \}$ drawn from $\mathrm{P}$, the conditional distribution $\mathrm{P}_{Y|X}$, i.e., \textit{posterior probability}, is defined as $\eta : \mathcal{X} \to [0,1]^M$, where 
\begin{align}\label{equ::etamx}
	\eta_m(x)
	= \mathrm{P}(Y = m | X = x),
	\qquad
	m = 1, \ldots, M.
\end{align}

To analyze the theoretical properties of the classifier, there is a need to introduce some more notations to evaluate the performance. To this end, for any measurable decision function $f : \mathcal{X} \to [M]$, a loss function $L : \mathcal{X} \times \mathcal{Y} \times \mathcal{Y} \to \mathbb{R}_+ := [0, \infty)$ defines a penalty incurred on predicting $f(x) \in \mathcal{Y}$ when the true label is $y$. The \textit{risk} is defined by
$\mathcal{R}_{L,\mathrm{P}}(f) 
:= \int_{\mathcal{X} \times \mathcal{Y}} L(x, y, f(x)) \, d\mathrm{P}(x,y)$.
The Bayes risk, which is the smallest possible risk w.r.t.~$\mathrm{P}$ and $L$, is given by
$\mathcal{R}_{L, \mathrm{P}}^*
:= \inf \{ \mathcal{R}_{L,\mathrm{P}}(f) \, | \, f : \mathcal{X} \to [M] \text{ measurable} \}$,
where the infimum is taken over all measurable functions $f : \mathcal{X} \to [M]$. In addition, a measurable function $f_{L,\mathrm{P}}^*$ satisfying $\mathcal{R}_{L,\mathrm{P}}(f_{L,\mathrm{P}}^*) = \mathcal{R}_{L,\mathrm{P}}^*$ is called a Bayes decision function. For example, $f_{L_{\mathrm{cl}},\mathrm{P}}^*(x) = \argmax_{m \in [M]} \eta_m(x)$ is the Bayes decision function w.r.t.~the \textit{classification loss} $L_{\mathrm{cl}}(x, y, f(x)):= \eins \{ f(x) \neq y \} $. We denote $\mathcal{R}^*_{L_{\mathrm{cl}},\mathrm{P}} = \mathcal{R}_{L_{\mathrm{cl}},\mathrm{P}} (f_{L_{\mathrm{cl}},\mathrm{P}}^*)$ as the corresponding Bayes risk and $\mathcal{R}_{L_{\mathrm{cl},\mathrm{P}}}(f) - \mathcal{R}^*_{L,\mathrm{P}}$ as the \textit{classification error} for a candidate classifier. For the sake of brevity, we write $\eta_{L_\mathrm{cl},\mathrm{P}}^*(x) := \eta_{f_{L_{\mathrm{cl}},\mathrm{P}}^*(x)}(x)$ in the following.

For imbalanced classification, we need to introduce some additional notations. To this end, for $m \in [M]$, let $D_{(m)} := \{ (x, y) \in D_n \, | \, y = m \}$ and $n_{(m)} := \#(D_{(m)})$. Throughout this paper, without loss of generality, we assume that $n_{(1)} \leq \ldots \leq  n_{(M)}$. In this case, $D_{(1)}$ is called the \textit{minority class}. In addition, let $\pi_m = \mathrm{P}(Y = m)$ denote the proportion of each category. Moreover, we denote $\overline{\pi} := \max_{1 \leq m \leq M} \pi_m$ and $\underline{\pi} := \min_{1 \leq m \leq M} \pi_m$. Furthermore, for $1 \leq m\leq M$, we define the \textit{weighted} posterior probability function
\begin{align}\label{equ::etamstarx}
	\eta^w_m(x)=\frac{\eta_m(x)/\pi_m}{\sum_{m=1}^M \eta_m(x)/\pi_m},
\end{align}	 
where we assign lower weights to the highly populated classes whereas assign the largest weight to the minority class. By this means, we attach the same importance of the minority class and other classes in the evaluation of the model. In particular, when the class distribution is balanced, i.e. $\pi_1 =\ldots=\pi_m$, the weighted posterior probability function \eqref{equ::etamstarx} is the same as \eqref{equ::etamx}. Additionally, we define the \textit{imbalance ratio} $\rho\in [0,1]$ as the ratio of the minority sample size $n_{(1)}$ to the averaged sample size in each class $n/M$, namely,
\begin{align}\label{equ::kappa}
	\rho := M n_{(1)}/n.
\end{align}
Note that when $\rho=1$, we have $n_{(1)}=\cdots=n_{(M)}=n/M$. In this case the problem is reduced to the balanced classification. It is easy to see that the smaller $\rho$, the higher level of class imbalance.

However, the usual misclassification loss is ill-suited as a performance measure for imbalanced classification, since it expects an equal misclassification cost on all classes \cite{he2009learning}. In fact, in the presence of imbalanced training data, samples of the minority class occur sparsely in the data space. As a result, given a test sample, the calculated $k$-nearest neighbors bear higher probabilities of samples from the other classes. Hence, test samples from the minority class are prone to be incorrectly classified \cite{sun2009classification,leevy2018survey}. In particular, in the highly imbalanced binary classification, a classifier can achieve good performance w.r.t.~classification error by predicting all test samples to the majority class. However, this results in undesirable performance on the minority class. To tackle this problem, a variety of performance measures have been proposed for evaluating multi-class classifiers in class-imbalance settings, see e.g., \cite{tallon2014data,flach2019performance,opitz2019macro}.

In this paper, we study the statistical convergence of the algorithms w.r.t.~one such performance measure, namely the arithmetic mean of the recall (AM), which was proposed in \cite{chan1998learning} and recently investigated in \cite{menon2013statistical}, which studied the consistency of algorithms proposed for imbalanced binary classification. We refer the reader to \cite{guo2017learning,flach2019performance,grandini2020metrics} for more details. The AM measure attempts to balance the errors on classes and is shown to be an effective performance measure for evaluating classifiers in imbalanced classification. We are confined to this measure since it can be reformulated as the sum of losses on individual samples as is illustrated in Section \ref{sec::erroranalysis} and thus is available for theoretical analysis.

For any candidate classifier $f : \mathcal{X} \to [M]$ and $m \in [M]$,  we first consider the \textit{recall} of the class $m$ defined by $r_m(f) = \mathrm{P}(f(x) = m \, | \, y = m)$. Larger recall indicates better prediction of samples in the class $m$. In particular, $r_m(f)=1$ means that every sample from the class $m$ is predicted correctly through the classifier $f$. Then we define the AM measure as the arithmetic mean of these values, that is,
\begin{align}\label{equ::ramf}
	r_{\mathrm{AM}}(f) 
	= \frac{1}{M} \sum_{m=1}^M r_m(f).
\end{align}
In particular, we define the optimal AM performance by
$$
r_{\mathrm{AM}}^*
:= \sup \{ r_{\mathrm{AM}}(f) \, | \, f : \mathcal{X} \to [M] \text{ measurable} \}.
$$
Moreover, we define the \textit{AM-regret} of $f$ by
\begin{align}\label{equ::amerror}
	\mathfrak{R}_{\mathrm{AM}}(f)
	= r_{\mathrm{AM}}^* - r_{\mathrm{AM}}(f).
\end{align}
It is generally considered that the usual classification error and the AM regret measure the goodness-of-fit of the classifier under different settings. More specifically, when the probability distribution is approximately balanced, the usual classification error is more suited when we focus on the prediction of each sample individually regardless of which class the sample is drawn from, let alone the class distribution. By contrast, in the imbalanced classification, it is argued by \cite{grandini2020metrics} that the AM regret could be a more reasonable choice since it is insensitive to imbalanced class distribution and it attaches equal importance to the recall of each class. In practice, the choice of the performance measure is usually decided by the type of data encountered. In the statistics and machine learning literature, the usual classification error has been studied extensively and understood well. In this study, we focus on the analysis of AM regret for the under-bagging algorithm in the imbalanced multi-class classification, which has not yet been well studied in the literature.

\section{Main Algorithm} \label{sec::prelim}

The usual Bayes optimal classifier that minimizes the classification error is not optimal w.r.t.~the AM regret \cite{menon2013statistical,narasimhan2015consistent}. In fact, according to \cite{chaudhuri2014rates, xue2018achieving}, the standard $k$-NN classifier is showed to converge to the Bayes error, and therefore it is not consistent w.r.t.~AM regret. Thus it is emerging as an important issue to design $k$-NN based classifiers for imbalanced classification with both solid theoretical guarantees w.r.t.~the AM measure and desirable practical performance. In this paper, we first consider the $k$-NN classifier built on the under-sampling data, namely \textit{under-sampling $k$-NN}. Moreover, to make full use of the information that might be overlooked by the under-sampling, we introduce the bagging technique and propose the \textit{under-bagging $k$-NN} classifier.

Before we proceed, we introduce the under-sampling strategy and the related probability measure. Specifically, suppose that an acceptance probability function $a(x, y) \in [0, 1]$ is given for every data point $(x, y)\in \mathcal{X}\times \mathcal{Y}$. Then each observation $(x_i, y_i)$, $i = 1, \ldots, n$, is \textit{independently} drawn from $D_n$ with probability $a(x_i, y_i)$. Mathematically speaking, the under-sampling strategy can be stated as follows:
\begin{enumerate}
	\item[(i)] 
	Sample $(X,Y)\sim \mathrm{P}$, where $\mathrm{P}$ denotes the distribution of the input data.
	\item[(ii)] 
	Generate $Z(X,Y)$ from the Bernoulli distribution with parameter $a(X,Y) \in (0,1]$ which will be specified in the following sections.
	\item[(iii)] 
	If $Z(X,Y)=1$, then \textit{accept} the candidate $(X,Y)$. Otherwise, \textit{reject} $(X,Y)$ and go to the beginning.
\end{enumerate}

After the strategy is repeated $n$ times, we obtain an under-sampling dataset $D_n^u=\{(X_i,Y_i):Z(X_i,Y_i)=1\}$ containing those \textit{accepted} samples. The joint probability of $\{Z(X_i,Y_i)\}_{i=1}^n$ conditional on $\{(X_i,Y_i)\}_{i=1}^n$ is denoted by $\mathrm{P}_Z$.

\subsection{Under-sampling $k$-NN Classifier} \label{sec::knnwithuniform}

The aim of under-sampling is to create more balanced data subsets from the class-imbalanced input data set, so that the multi-class classifiers expecting balanced class distribution can be adapted to imbalanced classification. To be specific, given the minority class $D_{(1)}$ and the other classes $D_{(m)}$, $2 \leq m \leq M$, the under-sampling method randomly subsamples ${D}^u_n := \{(X_1^u,Y_1^u), \ldots, (X_{s_u}^u, Y_{s_u}^u)\}$, $s_u = \#(D^u_n)$, from $D_n$ with \textit{acceptance probability}
\begin{align} \label{equ::aundxypar}
	a(x, y) = \sum_{m=1}^M (n_{(1)} / n_{(m)}) \eins \{ y = m \}.
\end{align}
In this case, all the samples from the minority class are in the set $D_n^u$.

Let $X_{(i)}^u(x)$ denote the $i$-th nearest neighbor of $x$ in the sub-sampling data  $D_n^u$ w.r.t.~the Euclidean distance and $Y_{(i)}^u(x)$ denote its label. We define the \textit{posterior probability estimate} $\widehat{\eta}^{k,u} : \mathcal{X} \to [0, 1]^M$ with the $m$-th entry by
\begin{align} \label{equ::widehatetakunder}
	\widehat{\eta}_m^{k,u}(x)
	:= \frac{1}{k} \sum_{i=1}^k \eins \{ Y_{(i)}^u(x) = m \}.
\end{align}
Then the \textit{under-sampling $k$-NN} is given by
\begin{align} \label{equ::standardknnunder}
	\widehat{f}^{k,u}(x)
	= \argmax_{m \in [M]} \widehat{\eta}_m^{k,u}(x).
\end{align}

\subsection{Under-bagging $k$-NN  Classifier} \label{equ::underbaggingknn}

The main drawback of under-sampling is that potentially useful information contained in the samples not appearing in $D_n^u$ is overlooked. Thus we use the bagging technique to further exploit the samples ignored by under-sampling, that is, samples in $D_n \setminus D_n^u$. More precisely, on the $b$-th round of bagging, we subsample $D_b^u := \{ (X_1^{b,u},Y_1^{b,u}), \ldots, (X_{s_b}^{b,u},Y_{s_b}^{b,u})\}$ with acceptance probability 
\begin{align}\label{equ::axymn}
	a(x,y)
	= \sum_{m=1}^M (s / (M n_{(m)})) \eins \{ y = m \},
\end{align}
where $1\leq s \leq  M n_{(1)}$ is the expected number of bootstrap samples. Then our classifier is built upon $D_b^u$, that is, we compute the $b$-th posterior probability estimate $\widehat{\eta}^{b,u} : \mathcal{X} \to [0,1]^M$ on the set $D_b^u$ for $1 \leq b \leq B$, respectively. To be specific,  the $m$-th entry of $\widehat{\eta}^{b,u}$ is defined by
\begin{align}\label{equ::widehatetakunderbag}
	\widehat{\eta}_m^{b,u}(x)
	:= \frac{1}{k} \sum_{i=1}^k \eins \bigl\{ Y_{(i)}^{b,u}(x) = m \bigr\}.
\end{align}
Then the average posterior probability estimate $\widehat{\eta}^{b,u}:\mathcal{X}\to [0,1]^M$ is 
\begin{align}\label{equ::baglocalunderbag}
	\widehat{\eta}^{B,u}_m(x) 
	= \frac{1}{B} \sum_{b=1}^B \widehat{\eta}^{b,u}_m(x)
	= \frac{1}{B} \sum_{b=1}^B \frac{1}{k} \sum^k_{i=1} \eins \bigl\{ Y_{(i)}^{b,u}(x) = m \bigr\}.
\end{align}
Finally, the \textit{under-bagging $k$-NN classifier} is defined by
\begin{align}\label{equ::fbuniformunder}
	\widehat{f}^{B,u}(x)
	= \argmax_{m \in [M]} \widehat{\eta}_m^{B,u}(x).
\end{align}

We summarize the under-bagging $k$-NN classifier in Algorithm \ref{alg::BNNimbalance}.

\begin{algorithm}[!h]
	% \SetAlgoNoLine
	\caption{Under-bagging $k$-NN Classifier for Imbalanced Classification}
	\label{alg::BNNimbalance}
	\KwIn{
		Minority class $D_{(1)}$ and the other classes $D_{(m)}$, $2 \leq m \leq M$. \\
		\quad\quad\quad\,\,\,\, Bagging rounds $B$ and the expected subsample size $s$;
		\\
		\quad\quad\quad\,\,\,\, Parameter $k \in \mathbb{N}_+$.
		\\
	} 
	\For{$b =1 \to B$}{
		Randomly sample ${D}_b^u$ from ${D}_n$ with acceptance probability chosen by \eqref{equ::axymn}.
		\\
		Reorder $D_{b}^u= \{(X_{(1)}^{b,u},Y_{(1)}^{b,u}), \ldots, (X_{(s_b)}^{b,u},Y_{(s_b)}^{b,u})\}$;
		\\
		Compute the $b$-th posterior probability estimate  $\widehat{\eta}^{b,u}$ on the set $D_{b}^u$ by \eqref{equ::widehatetakunderbag}.
	}
	Compute the bagged posterior probability estimate $\widehat{\eta}^{B,u}$ by \eqref{equ::baglocalunderbag}.
	\\
	\KwOut{The under-bagging $k$-NN classifier \eqref{equ::fbuniformunder}.
	}
\end{algorithm}

\section{Theoretical Results and Statements}
\label{sec::mainresults}

In this section, we present main results on the convergence rates of the AM regret of $\widehat{f}^{k,u}$ and $f^{B,u}$ under mild conditions in Section \ref{sec::knnmulticlassim} and \ref{sec::knnmulticlassimbag}. Then in Section \ref{sec::cnd} we also present some comments and discussions on the obtained main results.

Before we proceed, we need to introduce the following restrictions on the probability distribution to characterize which properties of a distribution most influence the performance of the classifier for imbalanced classification.

\begin{assumption}\label{ass:imbalance}
	We make the following assumptions on the probability measure $\mathrm{P}$.
	\begin{enumerate}
		\item[\textbf{(i)}] \textbf{[Smoothness]} 
		The posterior probability function $\eta$ defined by \eqref{equ::etamx} is assumed to be $\alpha$-H\"{o}lder continuous with a constant $c_L \in (0, \infty)$, that is, for any $x, x' \in \mathcal{X}$, we have
		\begin{align}\label{equ::alphacontinuous}
			|\eta(x') - \eta(x)|
			\leq c_L \|x' - x\|^{\alpha}.
		\end{align}
		\item[\textbf{(ii)}] \textbf{[Margin]} 
		For any $x \in \mathcal{X}$, let $\eta^w_{(m)}(x)$ denote the $m$-th largest element in $\{ \eta^w_m(x) \}_{m=1}^M$, where $\eta^w_m(x)$ is defined by \eqref{equ::etamstarx}. Assume that there exists a constant $\beta > 0$ and $c{_\beta} > 0$ such that for all $t > 0$, there holds
		\begin{align}\label{tsybakovim}
			{\mathrm{P}} \bigl( |\eta^w_{(1)}(x) - \eta^w_{(2)}(x)| \leq t \bigr)
			\leq c_{\beta} t^{\beta}.
		\end{align}
	\end{enumerate}
\end{assumption}

The $\alpha$-H\"{o}lder smoothness assumption \textit{(i)} is commonly adopted for $k$-nearest neighbors classification, see, e.g., \cite{chaudhuri2014rates,doring2017rate,xue2018achieving,khim2020multiclass}. In fact, since $\eta^w(x)$ is the weighted posterior probability function, $\eta^w$ is $\alpha$-H\"{o}lder continuous as long as $\eta$ is $\alpha$-H\"{o}lder continuous. Note that when $\alpha$ is small, the posterior probability function fluctuates more sharply, which results in the difficulty of estimating $\eta^w$ accurately and thus leads to a slower convergence rates for imbalanced classification. It is worth pointing out that the smoothness assumption endows our model with a global constraints, whereas the margin assumption only reflects the behavior of $\eta^w$ near the decision boundary.

The margin assumption \textit{(ii)} quantifies how well classes are separated on the decision boundary $\partial  := \{ x : \eta^w_{(1)}(x) = \eta^w_{(2)}(x) \}$, which was adopted for the weighted nearest neighbors for multi-class classification \cite{khim2020multiclass}. In particular, in the usual binary classification problems, when the probability distribution is balanced in the sense of $\pi_1=\pi_2$, this condition coincides with Tsybakov's margin condition \cite{audibert2007fast}. We mention that the adjusted posterior probability function $\eta^w(x)$ defined by \eqref{equ::etamstarx} attaches larger weights to the small classes. The restriction on the margin of $\eta^w(x)$ is more reasonable than $\eta(x)$ for imbalanced classification. To give a clear explanation, let us consider the binary classification problem, where the decision boundary can be expressed by $\partial G := \{ x : \eta_1(x) / \pi_1 = \eta_2(x)/\pi_2\}$. In this case, compared with the ordinary boundary $\partial G_0=\{x:\eta_1(x)=\eta_2(x)\}$, the decision boundary $\partial G$ move towards the majority class and more points would be categorized as the minority class. From \eqref{tsybakovim}, we clearly see that when $\beta$ is smaller,  $\eta^w_{(1)}(x)$ approaches $\eta^w_{(2)}(x)$ from above more steeply, which reflects a more complex behavior around the critical threshold $\partial G$. In this case, the imbalanced multi-class classification becomes more difficult. In particular, for $\beta=\infty$, $\eta^w_{(1)}(x)$ is far from $\eta^w_{(2)}(x)$ with a large probability, which makes the multi-class classification significantly easier. In general, the margin assumption (\textit{ii}) does not affect the the smoothness of the posterior probability functions in condition (\textit{i}) and vice versa.

In the following two sections, we present main results on the convergence rates for the under-sampling and under-bagging $k$-NN classifier w.r.t.~the AM measure of type ``with high probability''. It is worth pointing out that our result is built upon the techniques from the approximation theory \cite{cucker2007learning} and arguments from the empirical process theory \cite{vandervaart1996weak,Kosorok2008introduction}, which is essentially different from the previous work on the consistency of algorithms w.r.t.~the AM measure \cite{menon2013statistical, narasimhan2015consistent}, where several tools such as classification-calibrated losses \cite{bartlett2006convexity} and regret bounds for cost-sensitive classification \cite{scott2012calibrated} have been developed for the study.

\subsection{Results on Convergence Rates for the Under-sampling $k$-NN Classifier} \label{sec::knnmulticlassim}

Now we present the convergence rates for the under-sampling $k$-NN classifier w.r.t.~the AM measure under the above assumptions.

\begin{theorem}\label{thm::knnundersample}
	Let $\widehat{f}^{k,u}$ be the under-sampling $k$-NN classifier defined as in \eqref{equ::standardknnunder}, where the acceptance probability is chosen as in \eqref{equ::aundxypar}. Assume that $\mathrm{P}$ satisfies Assumptions \ref{ass:imbalance} and $\mathrm{P}_X$ is the uniform distribution on $[0,1]^d$. Then there exists an $N_1^*\in \mathbb{N}$, which will be specified in the proof, such that for all $n\geq N_1^*$, by choosing 
	\begin{align} \label{k::undersampling}
		k = s_u^{2\alpha/(2\alpha+d)} (\log s_u)^{d/(2\alpha+d)}
	\end{align}
	where $s_u = \#(D_n^u)$, there holds
	\begin{align} \label{rate::udersampling}
		\mathfrak{R}_{\mathrm{AM}}(\widehat{f}^{k,u})
		\lesssim (\log n/n)^{\alpha(\beta+1)/(2\alpha+d)}
	\end{align}
	with probability $\mathrm{P}_Z \otimes \mathrm{P^n}$ at least $1-4/n^2$.
\end{theorem}

Compared with the standard $k$-NN where $k$ is of order $n^{2\alpha/(2\alpha+d)}$ up to a logarithm factor, in Theorem \ref{thm::knnundersample} we prove that $k$ is of order $(\rho n)^{2\alpha/(2\alpha+d)}$ up to a logarithm factor when under-sampling is introduced. Especially when the data is highly imbalanced, i.e., $\rho$ is very small, the value of $k$ can be significantly reduced by the under-sampling technique.

The following Theorem shows that up to a logarithm factor, the convergence rate \eqref{rate::udersampling} of the under-sampling $k$-NN classifier $\widehat{f}^{k,u}$ is minimax optimal w.r.t.~the AM regret in the case $\alpha \beta < d$.

\begin{theorem} \label{thm::lowerbound}
	Let $\mathcal{F}$ be the set of all measurable functions $f_n:(\mathbb{R}^d\times \mathbb{R})^n\times \mathbb{R}^d\to \mathbb{R}$ and $\mathcal{P}$ be the set of all probability distributions satisfying Assumption \ref{ass:imbalance} with $\alpha\beta<d$. Then we have
	\begin{align*}
		\inf_{f_n\in \mathcal{F}} \sup_{\mathrm{P}\in \mathcal{P}} \mathfrak{R}_{\mathrm{AM}}(f_n)
		\gtrsim n^{-\alpha(\beta+1)/(2\alpha+d)}.
	\end{align*}
\end{theorem}

The lower bound in Theorem \ref{thm::lowerbound} coincides with that for standard classification w.r.t.~the classification error \cite{audibert2007fast}, although the class of probability distribution considered in Theorem \ref{thm::lowerbound} is different from the one considered in \cite{audibert2007fast} as stated in the beginning of Section \ref{sec::mainresults}.

\subsection{Results on Convergence Rates for the Under-bagging $k$-NN Classifier} \label{sec::knnmulticlassimbag}

We now state our main results on the convergence of the under-bagging $k$-NN classifier w.r.t.~the AM measure.

\begin{theorem}\label{thm::bagknnclassund}
	Let $\widehat{f}^{B,u}(x)$ be the under-bagging $k$-NN classifier defined as in \eqref{equ::fbuniformunder}. Assume that $\mathrm{P}$ satisfies Assumption \ref{ass:imbalance} and $\mathrm{P}_X$ is the uniform distribution on $[0,1]^d$. Furthermore, let $\rho$ be the imbalance ratio defined by \eqref{equ::kappa}. Then there exists an $N_2^*\in \mathbb{N}$, which will be specified in the proof, such that for all $n\geq N_2^*$, by choosing 
	\begin{align}
		s & \gtrsim
		\begin{cases}
			(\rho n)^{d/(2\alpha+d)} (\log (\rho n))^{2\alpha/{(2\alpha+d)}}, & \text{ if } d > 2 \alpha, 
			\\
			(\rho n \log (\rho n))^{1/2}, & \text{ if } d \leq 2 \alpha,
		\end{cases}
		\label{s::underbagging}
		\\
		B & = \rho n/s,
		\label{B::underbagging}
		\\
		k & = s (\log (\rho n)/\rho n)^{d/{(2\alpha+d)}},
		\label{k::underbagging}
	\end{align}
	there holds
	\begin{align} \label{rate::underbagging}
		\mathfrak{R}_{\mathrm{AM}}(\widehat{f}^{B,u})
		\lesssim (\log n/n)^{\alpha(\beta+1)/(2\alpha+d)}
	\end{align}
	with probability $\mathrm{P}_{Z}^B \otimes \mathrm{P}^n$ at least $1 - 5 / n^2$.
\end{theorem}

Theorem \ref{thm::bagknnclassund} together with Theorem \ref{thm::lowerbound} implies that up to a logarithm factor, the convergence rate \eqref{rate::underbagging} of the under-bagging $k$-NN classifier $\widehat{f}^{B,u}$ turns out to be minimax optimal w.r.t.~the AM measure, if we choose the expected sub-sample size $s$, the bagging rounds $B$, and the number of nearest neighbors $k$ according to \eqref{s::underbagging}, \eqref{B::underbagging}, and \eqref{k::underbagging}, respectively. In other words, when the bagging technique is combined with the under-sampling $k$-NN classifier, the convergence rates of $\widehat{f}^{B,u}$ is not only obtainable, but also the same with that of $\widehat{f}^{k,u}$.

Notice that for a given dataset, \eqref{s::underbagging} and \eqref{B::underbagging} yield that $k$ and $B$ is proportional to $s$ and $s^{-1}$, respectively. Therefore, only a few independent bootstrap samples are required to obtain the estimate $\widehat{\eta}_m^{b,u}$ in \eqref{equ::widehatetakunderbag} for the posterior probability function at each bagging round. As a result, $k$ is reduced to $\mathcal{O}(\log (\rho n))$ in \eqref{k::underbagging}, instead of $\mathcal{O}((\rho n)^{2\alpha/(2\alpha+d)} (\log (\rho n))^{d/(2\alpha+d)})$ in \eqref{k::undersampling} for the under-sampling $k$-NN.

In particular, we show in Corollary \ref{thm::bag1nnund} that $k$ can be further reduced to a constant order and present the convergence rates of under-bagging $1$-NN classifier w.r.t.~the AM measure.

\begin{corollary}\label{thm::bag1nnund}
	Let $\widehat{f}^{B,u}(x)$ be the bagged $1$-NN classifier defined by Algorithm \ref{alg::BNNimbalance} with $k=1$. Furthermore, assume $\mathrm{P}$ satisfies Assumption \ref{ass:imbalance} and $\mathrm{P}_X$ is the uniform distribution on $[0,1]^d$. Moreover, let $\rho$ be the imbalance ratio defined by \eqref{equ::kappa}. Then there exists an $N_3^*\in \mathbb{N}$, which will be specified in the proof, such that for all $n\geq N_3^*$, with probability $\mathrm{P}_{Z}^B \otimes \mathrm{P}^n $ at least $1 - 5 / n^2$,  the  following two statements hold:
	\begin{enumerate}
		\item[(i)] 
		If $d > 2 \alpha$, by choosing $s = (\rho n)^{\frac{d}{2\alpha+d}} (\log (\rho n))^{\frac{2\alpha-d}{2\alpha+d}}$ and $B = (\rho n)^{\frac{2\alpha}{2\alpha+d}}(\log (\rho n))^{\frac{d-2\alpha}{2\alpha+d}}$, we have
		\begin{align} \label{rate::underbagging1NN}
			\mathfrak{R}_{\mathrm{AM}}(\widehat{f}^{B,u})
			\lesssim (\log^2 n/n)^{\alpha(\beta+1)/(2\alpha+d)}.
		\end{align}
		\item[(ii)] 
		If $d \leq 2 \alpha$, by choosing $s = (\rho n \log (\rho n))^{1/2}$ and $B = (\rho n / \log (\rho n))^{1/2}$, we have
		\begin{align*}
			\mathfrak{R}_{\mathrm{AM}}(\widehat{f}^{B,u})
			\lesssim \max \bigl\{ (\log n/n)^{\alpha/(2d)}, (\log^3 n/n)^{1/4} \bigr\}^{\beta+1}.
		\end{align*}
	\end{enumerate}
\end{corollary}

Again, Theorem \ref{thm::lowerbound} yields that up to a logarithm factor, the rate \eqref{rate::underbagging1NN} of under-bagging $1$-NN classifier is minimax optimal when $d > 2 \alpha$, $\alpha \in [0, 1]$, which is usually the case.

\subsection{Comments and Discussions} \label{sec::cnd}

The section presents some comments on the obtained theoretical results on the convergence rates of the under-sampling classifier $\widehat{f}^{k,u}$ and $\widehat{f}^{B,u}$, and compares them with related findings in the literature.

\subsubsection{Comments on Convergence for Imbalanced Classification}

In this paper, we focus on the imbalanced classification problem. As pointed out in Sections \ref{sec::intro} and \ref{sec::multiclassimbalance}, in the context of imbalanced classification, the AM regret is a widely used performance measure instead of the usual classification error. \cite{xue2018achieving, khim2020multiclass} show that the standard $k$-NN classifier converges to the Bayes risk for multi-class classification. Therefore, in general, it can not be consistent w.r.t.~the AM regret, which explains the undesirable performance of standard $k$-NN classification on imbalanced data from the theoretical perspective. To tackle this problem, in this study, we consider the under-sampling and under-bagging $k$-NN classifiers. Both of them only retain the scalability and easy implementation of $k$-NN method but also have the optimal convergence rates w.r.t.~the AM measure (Theorems \ref{thm::knnundersample} and \ref{thm::bagknnclassund}). In our analysis, we only make the $\alpha$-H\"{o}lder continuity and the margin assumption for the posterior probability function $\eta$. As is pointed out in Section \ref{sec::mainresults}, for imbalanced classification, the decision boundary is determined by the weighted posterior probability function $\eta^w$ by endowing smaller classes with larger weights. Moreover, it is worth pointing out that the results are of type ``with high probability'' by using Bernstein's concentration inequality that takes into account the variance information of the random variables within a learning theory framework \cite{cucker2007learning,steinwart2008support}.

As mentioned in Section \ref{sec::intro}, in the literature, despite many classifiers designed to address imbalance have been proposed, theoretical studies on these methods w.r.t.~the AM measure are relatively limited. \cite{menon2013statistical} proved the statistical consistency of two families of algorithms for imbalanced classification, where the first family of algorithms applies a suitable threshold to a class probability estimate obtained by minimizing an appropriate strongly proper loss, and the second one minimizes a suitably weighted form of an appropriate classification-calibrated loss, i.e., the cost-sensitive learning algorithms. However, it is well-known that consistency only measures the infinite-sample property of a classifier, while its finite-sample bounds of convergence rates can hardly be guaranteed. Besides, consistency did not directly reflect any degree of regularity or smoothness of the underlying posterior probability function.

\subsubsection{Comparison with Previous Bagged $k$-NN Algorithms and Analysis}

We also compare our results with previous theoretical analysis of the $k$-NN algorithm combined with bagging techniques. \cite{hall2005properties} demonstrated the consistency properties of bagged nearest neighbor classifiers to the Bayes classifier. \cite{biau2010rate} studied the rate of convergence of the bagged nearest neighbor estimate w.r.t.~the mean squared error. They derived the optimal rate $\mathcal{O}(n^{-2/(2+d)})$ under the assumption that the regression function is Lipschitz. \cite{samworth2012optimal} regarded the bagged nearest neighbor classifier as a weighted nearest neighbor classifier, and showed that the ``infinite simulation'' case of bagged nearest neighbors (with infinite bagging rounds) can attain the optimal convergence rate. It is worth pointing out that our analysis of the under-bagging $k$-NN presents in this study is essentially different from that in the previous works.

First of all, we highlight that different from previous statistical analysis, our theoretical analysis is conducted from a learning theory perspective \cite{cucker2007learning,steinwart2008support} using techniques such as approximation theory and empirical process theory \cite{vandervaart1996weak,Kosorok2008introduction}.

Secondly, previous works only take into account the uniform resampling method based on the $1$-NN classifier, where the weights of the bagged estimate have an explicit probability distribution, whereas our work aims at providing a theoretical analysis of the under-sampling $k$-NN algorithm (Theorem \ref{thm::knnundersample}). To this end, we have to explore the more complex Generalized Pascal distribution (Section \ref{sec::proofbagapprox}).

Thirdly, previous works consider the ``infinite simulation'' case of bagged $k$-NN when the number of bagging round $B \to \infty$, where the results fail to explain the success of bagging with finite resampling times in practice. By contrast, we provide results of convergence rate with finite $B$ by exploiting arguments such as Bernstein's concentration inequality from the empirical process theory, which enable us to derive the trade-off between the number of bagging rounds $B$ and the expected sub-sample size $s$ (Theorem \ref{thm::bagknnclassund}).
In fact, \eqref{s::underbagging} and \eqref{B::underbagging} imply that $B \lesssim (\rho n)^{2\alpha/(2\alpha+d)}$ up to a logarithm factor, which turns out to be relatively small, especially when $\rho$ is small or when $d$ is large. 
Furthermore, we show that when the smaller acceptance probability is adopted for under-sampling, more bagging rounds are required to achieve optimal convergence rates.

Last but not least, results in \cite{biau2010rate} hold ``in expectation'' w.r.t.~both the resampling distribution and input data, and results in \cite{samworth2012optimal} hold ``in probability'', while results in our study hold ``with high probability'', which is a stronger claim since it gives us a confidence about how well the method has learned for a given data set $D$ of fixed size $n$ \cite{steinwart2008support}. In other words, Theorem \ref{thm::knnundersample} and \ref{thm::bagknnclassund} imply that for most datasets sampled from $\mathrm{P}^n$ the classifiers $\widehat{f}^{k,u}$ and $\widehat{f}^{B,u}$ have an almost optimal performance whenever $n$ is large.

\subsubsection{Comments on Complexity} \label{sec::complexity}

As a commonly-used algorithm, $k$-$d$ tree \cite{bentley1975multidimensional} is used to search the nearest neighbors in NN-based methods. In what follows, we show that under-bagging helps reducing the time complexity of both the construction and search stages, whereas maintaining roughly the same space complexity.

\cite{friedman1977algorithm} shows that $k$-$d$ tree has a time complexity $\mathcal{O}(nd \log n)$ and a space complexity $\mathcal{O}(nd)$ for the tree construction. By Theorem \ref{thm::bagknnclassund}, we see that it suffices to choose $B = \rho n/s$ when $s = (\rho n)^{d/(2\alpha+d)} (\log (\rho n))^{2\alpha/{(2\alpha+d)}}$. Therefore, compared with the standard $k$-NN whose complexity is $\mathcal{O}(nd\log n)$, the time complexity of construction the $k$-$d$ tree in our algorithm can be reduced to $\mathcal{O}((\rho n)^{d/(2\alpha+d)} d\log (\rho n)^{(2\alpha+d)/(4\alpha+d)})$ with parallel computing fully employed. Considering the bagging rounds $B$, the space complexity of our algorithm turns out to be $\mathcal{O}(Bsd)=\mathcal{O}(\rho n d)$, whereas the space complexity of the standard $k$-NN is $\mathcal{O}(nd)$.

In the search of the $k$-th nearest neighbor for a test sample, the time complexity is $\mathcal{O}(k \log n)$  \cite{friedman1977algorithm}. For the standard $k$-NN, since the number of nearest neighbors is $\mathcal{O}(n^{2\alpha/(2\alpha+d)})$ \cite{chaudhuri2014rates, zhao2021minimax}, the time complexity of the search stage turns out to be $\mathcal{O}(n^{2\alpha/(2\alpha+d)}\log n))$. According to Theorem \ref{thm::bagknnclassund}, thanks to the under-sampling technique, for each base learner, we merely need to search $\mathcal{O}(\log (\rho n))$ neighbors among $s = (\rho n)^{d/(2\alpha+d)} (\log (\rho n))^{2\alpha/{(2\alpha+d)}}$ samples. Thus, the time complexity of the search stage can be reduced to $\mathcal{O}(\log^2(\rho n))$.

In summary, the bagging technique can enhance the computational efficiency to a considerable amount when parallel computation is fully employed. When the dimension gets higher, we typically require more samples in the input space, i.e., larger $n$, and thus an algorithm requires more time. This phenomenon is often referred to as the curse of dimensionality. We mention that the under-bagging technique can actually alleviate this problem by enjoying smaller time complexity. Furthermore, by adopting the under-sampling rule, the expected number of samples in each class is equal to the sample size of the minority class, and thus the size of training samples at each bagging round can be greatly reduced when the data distribution is highly imbalanced, reflected in a very small value of $\rho$.

\section{Error Analysis} \label{sec::erroranalysis}

In this section, we conduct error analysis for the under-sampling and under-bagging $k$-NN classifier respectively by establishing its convergence rates, which are stated in the above section in terms of the AM measure. The downside of using the AM measure is that it does not admit an exact bias-variance decomposition and the usual techniques for classification error estimation may not apply directly. Nonetheless, if we introduce the \textit{balanced} version of the classification loss,
\begin{align}\label{equ::lbal}
	L_{\mathrm{bal}}(x, y, f(x))
	& := \sum_{m=1}^M \eins \{ y = m \} L_{\mathrm{cl}}(x, y, f(x)) / (M \pi_m)
	\nonumber\\
	& = \sum_{m=1}^M \eins \{ y = m \} \eins \{ f(x) \neq y \} / (M \pi_m),
\end{align}
where a wrong classification of an instance from the minority class is punished stronger than a wrong classification of an instance from the majority class, we are able to reduce the problem of analyzing the AM regret to the problem of analyzing the expectation or sum of a loss on individual samples. In fact, the balanced loss is a useful tool to study the statistical consistency of algorithms for ranking and imbalanced classification \cite{kotlowski2011bipartite, menon2013statistical}. According to Proposition 6 in \cite{narasimhan2015consistent}, the Bayes classifier w.r.t.~the balanced loss can be expressed as
\begin{align}
	f_{L_{\mathrm{bal}}, \mathrm{P}}^*(x)
	& = \argmin_{m \in [M]} \sum_{i=1}^n \eta_i(x) L_{\mathrm{bal}}(x, m, i)
	= \argmin_{m \in [M]} \sum_{i=1}^n \eta_i(x){\eins \{ i \neq m \}} / \pi_i
	\nonumber\\
	& = \argmin_{m \in [M]} \biggl( \sum_{i=1}^n \eta_i(x) / \pi_i - \eta_m(x) / \pi_m \biggr)
	= \argmax_{m \in [M]} \eta^w_m(x).
	\label{equ::bayes_bal}
\end{align}
For brevity, we write $\eta_{L_{\mathrm{bal}},\mathrm{P}}^*(x) := \eta_{f_{L_{\mathrm{bal}}, \mathrm{P}}^*(x)}(x)$ in the following. From \eqref{equ::bayes_bal} we see that the Bayes classifier in terms of the balanced loss depends on the weighted posterior probability function $\eta^w$ instead of $\eta$. To explain, let us consider the binary classification problem $f:\mathcal{X}\to \{+1,-1\}$ with $\pi_{+1}\geq \pi_{-1}$. Then \eqref{equ::bayes_bal} takes the following form
\begin{align*}
	f^*_{L_{\mathrm{bal}},\mathrm{P}}(x)
	& = \sign \bigl( \eta_{+1}(x) / \pi_{+1} - \eta_{-1}(x) / \pi_{-1} \bigr)
	\\
	& = \sign \bigl( \eta_{+1}(x) / \pi_{+1} - (1 - \eta_{+1}(x)) / \pi_{-1} \bigr)
	= \sign(\eta_{+1}(x)-\pi_{+1}),
\end{align*}
where $\sign(x)=1$ if $x>0$ and $\sign(x)=-1$ otherwise. It is easy to see that the decision boundary changes from $1/2$ for usual classification to $\pi_{+1}$, which expands the region where the prediction is the minority class. In particular, if $\pi_m = 1/M$ for $1 \leq m \leq M$, then we have $L_{\mathrm{bal}}(x, y, f(x)) = L_{\mathrm{cl}}(x, y, f(x))$, then the balanced loss is equal to the classification loss, which leads to the same Bayes classifier.

With these preparations, we can present the next proposition which indicates that to analyze the AM-regret of a classifier, it suffices to analyze its balance risk.

\begin{proposition}\label{pro::eqiambalance}
	For any classifier $f:\mathcal{X}\to [M]$, we have
	$r_{\mathrm{AM}}(f)=1-\mathcal{R}_{L_{\mathrm{bal}}, \mathrm{P}}(f)$.
\end{proposition}

The above proposition directly yields that $r_{\mathrm{AM}}^* = 1 - \mathcal{R}_{L_{\mathrm{bal}},\mathrm{P}}^*$, which implies that the AM-regret is equal to the excess balanced error. As a result, the AM-regret in \eqref{equ::amerror} can be re-expressed as
\begin{align}\label{equ::R_amf}
	\mathfrak{R}_{\mathrm{AM}}(f)
	= \mathcal{R}_{L_{\mathrm{bal}}, \mathrm{P}}(f) - \mathcal{R}_{L_{\mathrm{bal}},\mathrm{P}}^*.
\end{align}

To bound the right-hand side of \eqref{equ::R_amf}, the main idea here is to build a new probability distribution to convert the excess balanced error into the excess classification error so that the approximation theory and the Bernstein's concentration inequality for multi-class classification can be applied. To this end, note that \eqref{equ::bayes_bal} implies that the Bayes classifier in terms of the balanced loss depends on $\eta^w$, which inspires us to consider a new probability distribution $\mathrm{P}^w$ with the posterior probability function $\eta^w$. To be specific, let $\mathrm{P}(X,Y)$ be the probability distribution of the samples, then we define the \textit{balanced} probability distribution $\mathrm{P}^w(X,Y)$ whose marginal distribution satisfies
\begin{align} \label{equ::tildepy}
	\pi^w_m
	:=\mathrm{P}^w(Y = m)
	= 1/M
	\qquad
	\text{ for } 1 \leq m \leq M, 
\end{align}
and the conditional density function satisfies
\begin{align}\label{equ::tildefxy}
	f^w(x | Y = m)
	= f(x | y = m)
	= \eta_m(x) f_X(x)/\pi_m.
\end{align}
Consequently, combining \eqref{equ::tildepy} and \eqref{equ::tildefxy}, we obtain the marginal density $f^w_{{X}}(x)$ given by
\begin{align}\label{equ::tildefx}
	f^w_X(x)
	= \sum_{m=1}^M \pi^w_m f^w(x | y = m)
	= \sum_{m=1}^M \eta_m(x) f_X(x) / (M \pi_m).
\end{align}
Thus, for the probability measure $\mathrm{P}^w(x, y)$, the Bayes classifier w.r.t.~the classification loss is 
\begin{align}\label{equ::f01fbal}
	f_{L_{\mathrm{cl}},\mathrm{P}^w}^*(x)
	= \argmax_{m \in [M]} \eta^w_m(x)
	= f_{L_{\mathrm{bal}}, \mathrm{P}}^*(x).
\end{align}

With these preparations, we present the next  theorem showing the equivalence between the excess balanced error w.r.t.~$\mathrm{P}$ and the excess classification error w.r.t.~$\mathrm{P}^w$ defined as above, which supplies the key to the proof of the convergence rates of the candidate classifier w.r.t.~the AM measure.

\begin{theorem}\label{thm::connection}
	Let  $\mathrm{P}^w$ be the probability measure defined by \eqref{equ::tildepy} and \eqref{equ::tildefxy}. Then for any classifier $f : \mathcal{X} \to \mathcal{Y}$, we have
	$\mathcal{R}_{L_{\mathrm{bal}},\mathrm{P}}(f)-\mathcal{R}^*_{L_{\mathrm{bal}},\mathrm{P}}
	= \mathcal{R}_{L_{\mathrm{cl}}, \mathrm{P}^w}({f}) - \mathcal{R}_{L_{\mathrm{cl}}, \mathrm{P}^w}^*$.
\end{theorem}

Combining Theorem \ref{thm::connection} with \eqref{equ::R_amf}, we see that the standard techniques 
can also be applied to analyzing the classification error w.r.t.~the balanced probability distribution $\mathrm{P}^w$ to derive the convergence rates of AM regret.

The following Lemma enables us to reduce the problem of bounding the excess multi-class classification error to the problem of bounding the estimation error of the posterior probability function.

\begin{proposition}\label{pro::regressionerror}
	Let $\widehat{\eta}:\mathcal{X}\to [0,1]^M$ be an estimate of $\eta^w$ and $\widehat{f}(x)=\argmax_{m\in[M]}\widehat{\eta}_m(x)$. Suppose that there exists $\phi_n$ such that
	$(\mathrm{P}^w)^n ( \|\widehat{\eta}(x) - \eta^w(x)\|_{\infty} \leq \phi_n)
	\geq 1 - \delta$.
	Then with probability $(\mathrm{P}^w)^n$ at least $1 - \delta$, there holds
	$\mathcal{R}_{L_{\mathrm{cl}},\mathrm{P}^w}(\widehat{f})-\mathcal{R}_{L_{\mathrm{cl}},\mathrm{P}^w}^*\leq c_\beta(2\phi_n)^{\beta+1}$.
\end{proposition}

Therefore, to further our analysis, we first need to bound the $L_{\infty}$-distance $\|\eta^{k,u}-\eta^w\|_{\infty}$ and $\|\eta^{B,u}-\eta^w\|$, where $\eta^{k,u}$ and $\eta^{B,u}$ are the posterior probability function estimator for the under-sampling and under-bagging $k$-NN defined by \eqref{equ::widehatetakunder} and \eqref{equ::baglocalunderbag}, respectively. We present the error analysis for the under-sampling and under-bagging $k$-NN classifiers in the following two sections.

\subsection{Analysis for the Under-Sampling $k$-NN Classifier} \label{sec::errordec}

In this section, we conduct error analysis for the $L_{\infty}$-distance between  $\eta^{k,u}$ and $\eta^w$ by establishing its convergence rates. The downside of under-sampling strategy is that it changes the probability distribution of the training data, that is, the under-sampling subset $D_n^u$ in Section \ref{sec::knnwithuniform} dose not have the distribution $\mathrm{P}$. In the sequel, let $\mathrm{P}^u$ denote the probability distribution of the accepted samples by the under-sampling strategy discussed in Section \ref{sec::knnwithuniform} and $\eta^u(x)$ be the corresponding posterior probability function. By Lemma \ref{lem::pu} in Section \ref{proofsecerror}, for $1\leq m\leq M$, $\eta_m^u(x)$ can be expressed as 
\begin{align}\label{equ::etaund}
	\eta_m^u(x)
	= \frac{\eta_m(x) / n_{(m)}}{\sum_{m=1}^M \eta_m(x) / n_{(m)}}.
\end{align}
It thus follows that
\begin{align}\label{equ::etaku}
	\|\widehat{\eta}^{k,u}-\eta^w\|_{\infty}\leq \|\widehat{\eta}^{k,u}-\eta^u\|_{\infty}+\|\eta^u-\eta^w\|_{\infty}.
\end{align}
It is easy to see that the first term of the right-hand side of \eqref{equ::etaku} represents the error for applying $k$-NN on the subset $D_n^u$ and thus it admits the usual decomposition for error estimation whereas the second term, namely \textit{under-sampling error} is brought about by the under-sampling strategy from the training data. To bound the first term $\|\widehat{\eta}^{k,u}-\eta^u\|_{\infty}$, we need to define $\overline{\eta}^{k,u} : \mathcal{X} \to [0,1]^M$, where its $m$-th entry 
\begin{align}\label{equ::bareta}
	\overline{\eta}_m^{k,u}(x)
	= \mathbb{E}[ \widehat{\eta}_m^{k,u}(x)|D_n^u]
	= \frac{1}{k}\sum_{i=1}^k \eta_m^u(X_{(i)}^u(x)).
\end{align}
In other words, $\overline{\eta}^{k,u}$ denotes the conditionally expectation of $\widehat{\eta}^{k,u}$ on the under-sampling data $D_n^u$. Thus we obtain the error decomposition for the posterior probability function w.r.t.~the under-sampling $k$-NN classifier as follows:
\begin{align}\label{equ::errordec}
	\|\widehat{\eta}^{k,u}-\eta^w\|_{\infty}\leq \|\eta^u-\eta^w\|_{\infty} + \|\widehat{\eta}^{k,u}-\overline{\eta}^{k,u}\|_{\infty}+\|\overline{\eta}^{k,u}-\eta^u\|_{\infty}.
\end{align}
Apart from the under-sampling error mentioned above, the second term on the right-hand side of \eqref{equ::errordec} is called the \textit{sample error} since it is associated with the empirical measure $D_n^u$ and the last term of \eqref{equ::errordec} is called \textit{approximation error} since it indicates how the error is propagated by the under-sampling $k$-NN algorithm.

\subsubsection{Bounding the Sample Error Term} \label{sec::sampleunder}

We now establish the oracle inequality for the under-sampling posterior probability function $\widehat{\eta}^{k,u}$ under $L_{\infty}$-norm. This oracle inequality will be crucial in establishing the convergence results of the estimator.

\begin{proposition}\label{prop::BnnEstimation}
	Let $\widehat{\eta}^{k,u}$ and $\overline{\eta}^{k,u}$ be defined by \eqref{equ::widehatetakunder} and \eqref{equ::bareta}, respectively. Then there exists an $N_1\in \mathbb{N}$, which will be specified in the proof, such that for all $n>N_1$, with probability $\mathrm{P}^n\otimes \mathrm{P}_Z$ at least $1-1/n^2$,
	there holds
	\begin{align}\label{equ::error212}	
		\|\widehat{\eta}^{k,u}(x)-\overline{\eta}^{k,u}(x)\|_{\infty}
		\lesssim \sqrt{{\log s_u/k}}.
	\end{align}
\end{proposition}

\subsubsection{Bounding the Approximation Error Term} \label{sec::approxunder}

The result on bounding the deterministic error term shows that the $L_{\infty}$-distance between $\overline{\eta}^{k,u}$ and $\eta^u$ can be small by choosing $k$ appropriately.

\begin{proposition}\label{prop::knnundersample}
	Let $\widehat{f}^{k,u}$ be the under-sampling $k$-NN classifier defined by \eqref{equ::standardknnunder}. Assume that $\mathrm{P}_X$ is the uniform distribution on $[0,1]^d$ and Assumption \ref{ass:imbalance} is satisfied. Then there exists an $N_2\in \mathbb{N}$, which will be specified in the proof, such that for all $n\geq N_2$, there holds
	\begin{align*}
		\|\overline{\eta}^{k,u} - \eta^u \|_{\infty}
		\lesssim (k/s_u)^{\alpha/d}
	\end{align*}
	with probability $\mathrm{P}^n\otimes \mathrm{P}_Z$ at least $1-1/n^2$.
\end{proposition}

\subsubsection{Bounding the Under-sampling Error Term} \label{sec::undersamplingerror}

The next proposition shows that the $L_{\infty}$-norm distance between $\eta^u$ and $\eta^w$, which possess a polynomial decay w.r.t.~the number of the training data. The following result is crucial in our subsequent analysis on the converge rates of both under-sampling and under-bagging $k$-NN classifier.

\begin{proposition}\label{thm::punderund}
	Let $\eta^w_m(x)$ and  $\eta_m^u(x)$ be defined by \eqref{equ::etamstarx} and \eqref{equ::etaund} respectively. Then there exists an $N_3\in \mathbb{N}$, which will be specified in the proof, such that for all $n\geq N_3$, there holds
	\begin{align} \label{equ::con2}
		\|\eta^u - \eta^w\|_{\infty}
		\lesssim \sqrt{\log n/n}
	\end{align}
	with probability $\mathrm{P}^n$ at least $1-1/n^2$.
\end{proposition}

\subsection{Analysis for the Under-bagging $k$-NN Classifier} \label{sec::error_bagging}

We now proceed with the estimation of the posterior probability error term $\|\widehat{\eta}^{B,u}-\eta^w\|$ within the learning theory framework. To this end, we first show that the under-bagging $k$-NN classifier can be re-expressed as a weighted $k$-NN, which is amenable to statistical analysis. To be specific, let $X_{(i)}(x)$ be the $i$-th nearest neighbor of $x$ in $D_n$ w.r.t.~the Euclidean distance and $Y_{(i)}(x)$ denote its label. Then for $1\leq b\leq B$, we re-express the posterior probability estimate $\widehat{\eta}^{b,u} : \mathcal{X} \to [0, 1]^M$ on the under-sampling set $D_b^u$ with its $m$-th entry defined by
$\widehat{\eta}_m^{b,u}(x) 
= \sum_{i=1}^n V_i^{b,u}(x) \eins \bigl\{ Y_{(i)}(x) = m \bigr\}$.
Here, $V_i^{b,u}(x)$ equals 
$1/k$ if $\sum_{j=1}^i Z^b(X_{(j)}(x),Y_{(j)}(x))\leq k$ and $0$ otherwise,
where $Z^b(x,y)$, $1\leq b\leq B$, are i.i.d.~Bernoulli random variables with parameter $a(x, y)$. Then the posterior probability estimate \eqref{equ::baglocalunderbag} can be re-expressed as
\begin{align}\label{equ::baglocalunder}
	\widehat{\eta}_m^{B,u}(x) 
	= \frac{1}{B} \sum_{b=1}^B \widehat{\eta}^{b,u}_m(x)
	= \frac{1}{B} \sum_{b=1}^B \sum_{i=1}^n V_i^{b,u}(x) \eins \bigl\{ Y_{(i)}(x) = m \bigr\}.
\end{align}
To bound $\|\widehat{\eta}^{B,u}-\eta^w\|_{\infty}$, we need to consider the bagged posterior probability function estimator, that is, we repeat under-sampling an infinite number of times, and take the average of the individual outcomes. To be specific, we define 
$\widetilde{\eta}^{B,u} : \mathcal{X} \to [0,1]^M$, with the $m$-th entry 
\begin{align}\label{equ::tildefmlund}
	\widetilde{\eta}_m^{B,u}(x)
	= \mathbb{E}_{\mathrm{P}_Z}^B [ \widehat{\eta}_m^{B,u}(x) |  \{(X_i,Y_i)\}_{i=1}^n ]:=\sum_{i=1}^n \overline{V}_i^u(x) \eins \bigl\{ Y_{(i)}(x) = m \bigr\},
\end{align}
where 
\begin{align}\label{equ::defbarviu}
	\overline{V}_i^u(x)
	= \mathbb{E}_{\mathrm{P}_Z} \bigl[ V_i^{b,u}(x) \big|  \{(X_i,Y_i)\}_{i=1}^n \bigr].
\end{align}
In fact, for any $x\in \mathcal{X}$, by the law of large numbers, we have $\widehat{\eta}_m^{B,u}(x) \to \widetilde{\eta}_m^{B,u}(x)$ almost surely as $B\to \infty$. Then we define the population version of the bagged estimator $\widetilde{\eta}_m^{B,u}(x)$ as follows:
\begin{align}\label{equ::ebaglocalund}
	\overline{\eta}_m^{B,u}(x)
	= \mathbb{E}[ \widetilde{\eta}_m^{B,u} | X_1,\ldots,X_n]
	= \sum_{i=1}^n \overline{V}_i^u(x) \eta_m^u(X_{(i)}(x)),
\end{align}
where the conditional expectation is taken w.r.t.~$(\mathrm{P}^u)_{Y|X}^n$. With these preparations, we are able to make the following error decomposition:
\begin{align}\label{equ::errorbag}
	\|\widehat{\eta}^{B,u}-\eta^w\|_{\infty}\leq
	\|\widehat{\eta}^{B,u}-\widetilde{\eta}^{B,u}\|_{\infty}
	+ \|\overline{\eta}^{B,u}-	\widetilde{\eta}^{B,u}\|_\infty
	+ \|\overline{\eta}^{B,u}-	{\eta}^u\|_\infty+
	\|\eta^u-\eta^w\|_\infty.
\end{align}
Compared with the analysis for the under-sampling $k$-NN classifier in \eqref{equ::errordec}, there are four terms on the right hand side of \eqref{equ::errorbag}. Since we are not able to repeat the sampling strategy an infinite number of times, the bagging procedure brings about the error term $\|\widehat{\eta}^{B,u} - \widetilde{\eta}^{B,u}\|_{\infty}$, which is called \textit{bagging error} in what follows. In addition, the second and the third term on the right hand-side of \eqref{equ::errorbag} can be viewed as the \textit{bagged sample error} and the \textit{bagged approximation error} for the bagged posterior probability function estimator $\widetilde{\eta}_m^{B,u}$ by similar arguments in the analysis for under-sampling $k$-NN classifier. Finally, the last term of  \eqref{equ::errorbag} is the \textit{under-sampling error} as mentioned in \eqref{equ::errordec}.

\subsubsection{Bounding the Bagging Error Term} \label{sec::bagsample}

The next proposition shows that the bagging error term can be bounded in term of the number of the nearest neighbors $k$ and the number of bagging rounds $B$.

\begin{proposition}\label{pro::fml*funder}
	Let $\widehat{\eta}^{B,u}$ and $ \widetilde{\eta}^{B,u}$ be defined by \eqref{equ::baglocalunder} and \eqref{equ::tildefmlund}, respectively. Suppose that $9 k B \geq 2(2d+3)\log n$. Then there exists an $N_4 \in \mathbb{N}$, which will be specified in the proof, such that for all $n\geq N_4$, there holds
	\begin{align*}
		\|\widetilde{\eta}^{B,u}-\widehat{\eta}^{B,u}\|_{\infty}
		\lesssim \sqrt{\log n / (kB)}
	\end{align*}
	with probability $\mathrm{P}_Z^B\otimes \mathrm{P}^n$ at least $1-1/n^{2}$.
\end{proposition}

\subsubsection{Bounding the Bagged Approximation Error Term}\label{sec::bagapprox}

We now show that the $L_{\infty}$ distance between $\overline{\eta}^{B,u}$ and $\eta^u$ can be bounded by two terms. The first term is determined by the ratio $k/s$ and the smoothness of the posterior probability function whereas the second term results from the under-sampling strategy, which possess an exponential decay w.r.t.~$(s/n)^2$.

\begin{proposition}\label{pro::tildefmlund}
	Let $\overline{\eta}^{B,u}$ be defined by \eqref{equ::ebaglocalund} with $k \geq \lceil 48(2d+9)\log n\rceil$ and $\eta^u$ be defined by \eqref{equ::etaund}. Assume that $\mathrm{P}$ satisfies Assumption \ref{ass:imbalance} and $\mathrm{P}_X$ is the uniform distribution on $[0, 1]^d$. Moreover, suppose that $s\exp( - (s/M-k)^2/(2n)) \leq M\underline{\pi}/2$. Then there exists an $N_5\in \mathbb{N}$, which will be specified in the proof,  such that for all $n\geq N_5$, there holds 
	\begin{align*}
		\|\overline{\eta}^{B,u} - \eta^u\|_{\infty}
		\lesssim (k/s)^{\alpha/d} + \exp \bigl( - (s - k)^2 / (2n) \bigr)
	\end{align*}
	with probability $\mathrm{P}^n$ not less than $1 - 1 / n^2$. 
\end{proposition}

\subsubsection{Bounding the Bagged Sample Error Term}\label{sec::bagsampleerror}

We now establish the oracle inequality for the bagged under-bagging posterior probability function $\widetilde{\eta}^{B,u}$ in terms of $L_{\infty}$ norm. The oracle inequality will be crucial in establishing the convergence results for the under-bagging $k$-NN classifier.

\begin{proposition}\label{pro::tildeffunder}
	Let  $\widetilde{\eta}^{B,u}$ and $\overline{\eta}^{B,u}$ be defined by \eqref{equ::tildefmlund} and \eqref{equ::ebaglocalund}, respectively.  Then there exists an $N_6\in \mathbb{N}$, which will be specified in the proof, such that for all $n\geq N_6$, there holds
	\begin{align*}
		\|\widetilde{\eta}^{B,u} - \overline{\eta}^{B,u}\|_{\infty}
		\lesssim \sqrt{s \log n / (k M n_{(1)})}
	\end{align*}
	with probability $\mathrm{P}_{Z}^B \otimes \mathrm{P}^n$ at least $1 - 1 / n^2$.
\end{proposition}

\section{Experiments}\label{sec::experiments}

\subsection{Performance Evaluation Metrics}

Compared with typical supervised learning, imbalanced learning pays more attention to the classification performance of the minority classes. Therefore, instead of the overall accuracy, we use the AM measure defined by \eqref{equ::ramf} to evaluate the performance of different classifiers.

\subsection{Hyper-parameter Analysis}

There are three hyper-parameters in the under-bagging $k$-NN for imbalanced classification: the bagging rounds $B$, the number of nearest neighbors $k$, and the expectation of subsample size $s$. Before we conduct parameter through synthetic experiments, we first introduce the data generation procedure. We use a simple toy dataset containing two interleaving half circles. In detail, we use two half moon functions with Gaussian noises added to synthesize the samples, where each half represents an unique class. The standard deviation of the Gaussian noises is $0.2$. We generate $20,000$ positive samples and $200$ negative samples in each run for training, and $200,000$ positive samples and $2,000$ negative samples for testing. We repeat the synthetic experiments for 100 times and record the averaged recall score. One visualization of synthetic dataset is shown in Figure \ref{fig:syndata}, with the ratio of the majority class (marked in blue) and the minority class (marked in orange) $100:1$.

\begin{figure}[!h]
	\centering
	\includegraphics[width=0.58\textwidth]{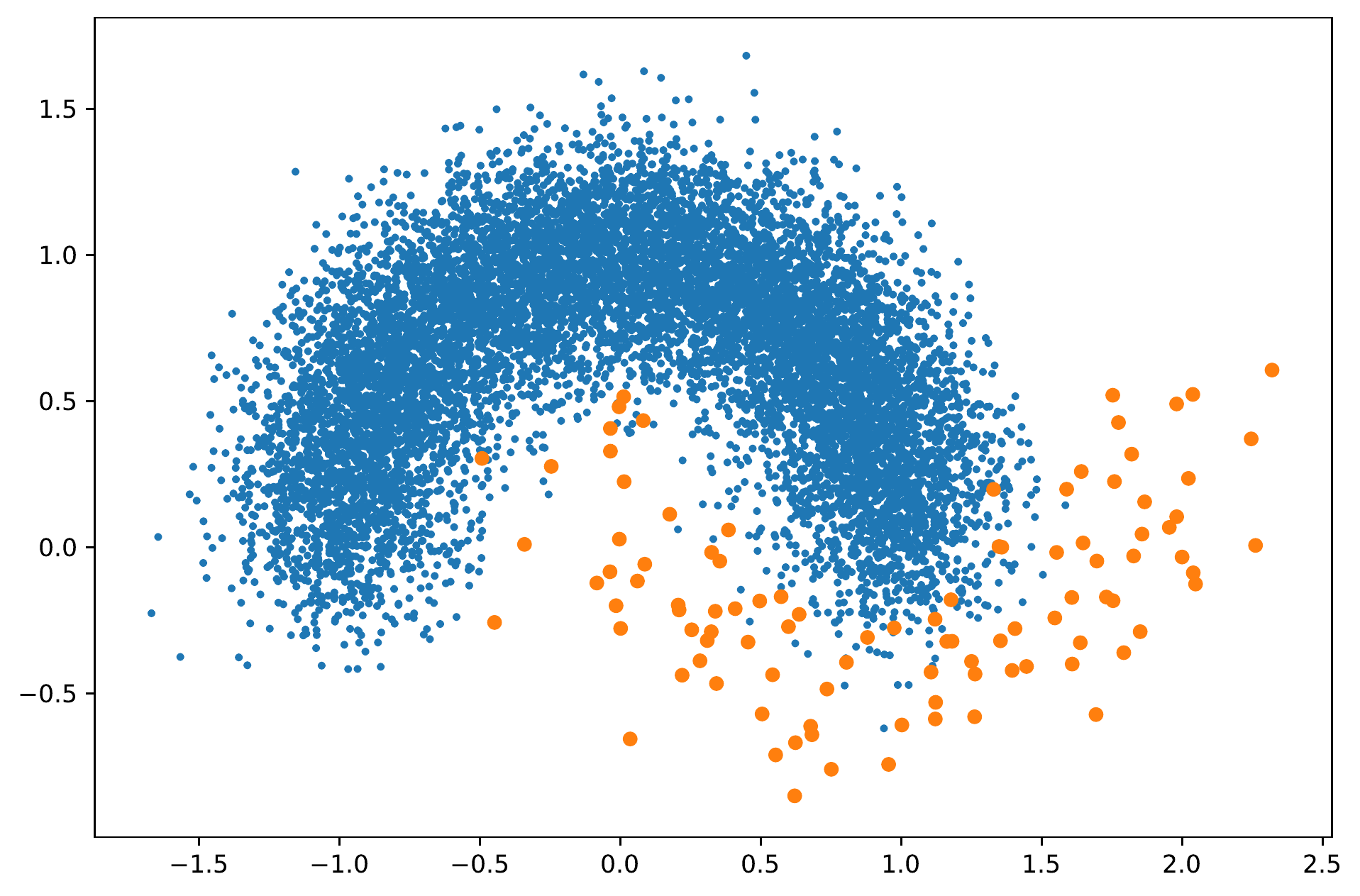}
	\captionsetup{justification=centering}
	\caption{One Visualization of the Generated Synthetic Dataset}
	\label{fig:syndata}
\end{figure}

\vspace{0.1cm}

\noindent\textbf{Parameter Analysis on $B$ and $k$.}
To study these two hyper-parameters $B$ and $k$, we fix the expected sub-sample size $s=Mn_{(1)}$, and vary the bagging rounds $B \in \{1, 2, 5, 10, 20, 50\}$ and the number of neighbors $k \in \{1, 2, \ldots, 30\}$. The averaged AM measure among different $B$ and $k$ are shown in Figure \ref{fig:param-Bk}. As expected, we can choose a sufficient large $B$ and the optimal number of neighbors $k$ for a good performance. Moreover, as the bagging rounds $B$ increases, the performance of a bagging classifier is robust under a wide range of hyper-parameter $k$, which reduces the difficulty of selecting the optimal hyper-parameter $k$. In addition, from the practical perspective, we can use a relatively large $B$ to achieve a high AM measure performance with a low running time, since we can easily save running time under the parallelism of bagging rounds $B$.

\begin{figure}[!h]
	\centering
	\includegraphics[width=0.58\textwidth]{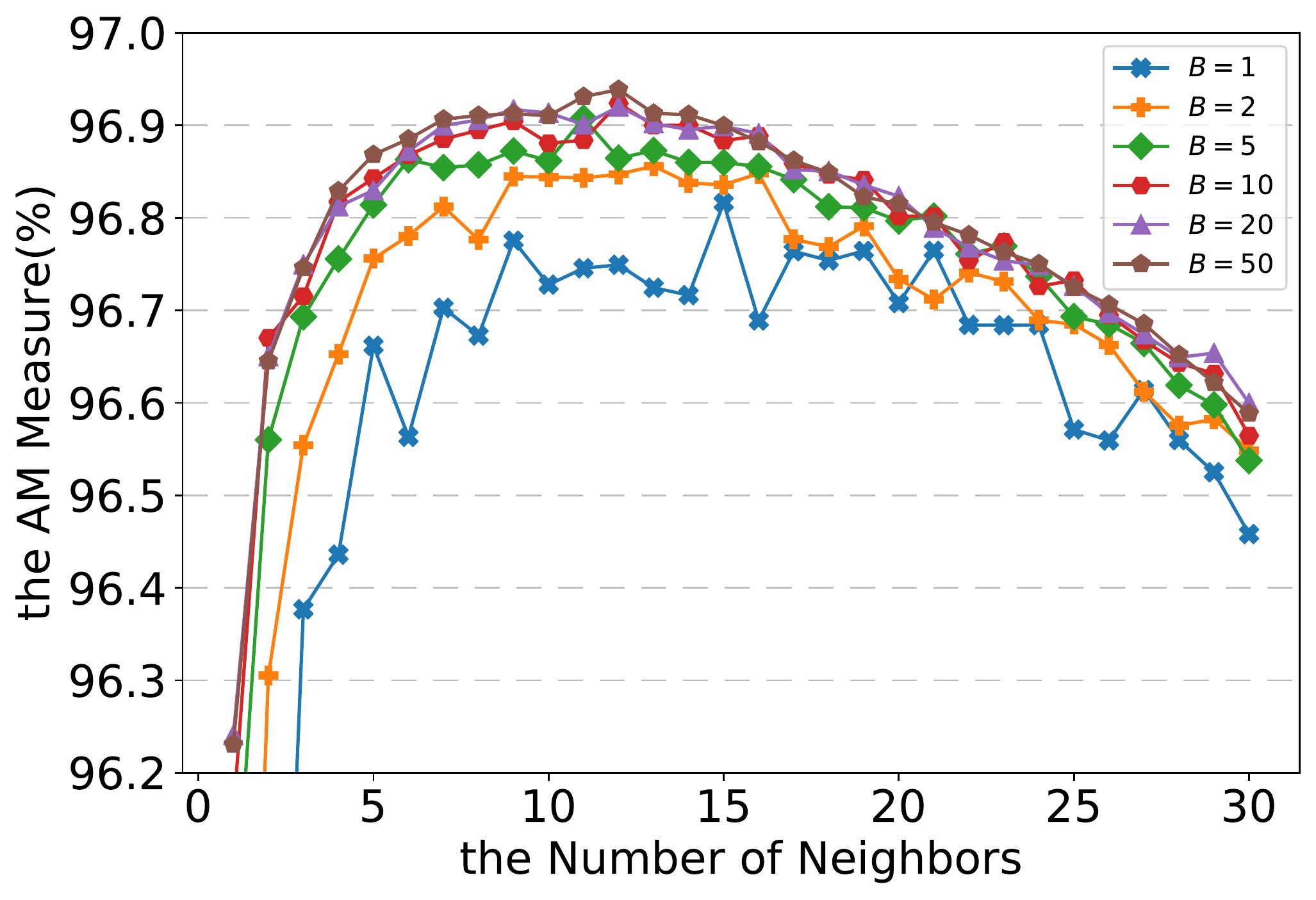}
	\captionsetup{justification=centering}
	\caption{Parameter Analysis between $B$ and $k$.}
	\label{fig:param-Bk}
\end{figure}

\vspace{0.1cm}

\noindent\textbf{Parameter Analysis on the Expected Sub-Sample Size $s$.}
To study the empirical performance of under-bagging $k$-NN with different expected sub-sample size $s$, we fix the bagging rounds $B=20$, and explore the performance under different sub-sample size $s=aMn_{(1)}$ with $a \in \{0.2, 0.4, \ldots, 1.0\}$. In fact, according to \eqref{equ::aundxypar}, $a$ represents the acceptance probability of the samples in the minority class. As is shown in Figure \ref{fig:param-sk}, as the expected sub-sample size becomes larger, the results of the AM measure become better, since more samples are taken into consideration in each round of bagging. However, the difference in the best AM measure gets smaller while the expected sub-sample size is close to $Mn_{(1)}$, which means that when coping with massive imbalanced data, under-bagging $k$-NN achieves competitive empirical performance with a relatively small sub-sample size $s$ w.r.t.~$n$. Moreover, the optimal number of nearest neighbors $k$ required increases with the expected sub-sample size $s$, which coincides with the theoretical results that $k = s (\log (\rho n)/\rho n)^{d/{(2\alpha+d)}}$ in Theorem \ref{thm::bagknnclassund}. In practice, we can use a relatively small $a$, and train on less samples with a relatively small hyper-parameter $k$ to achieve competitive empirical performance under only a small computational burden.

\begin{figure}[!h]
	\centering
	\includegraphics[width=0.58\textwidth]{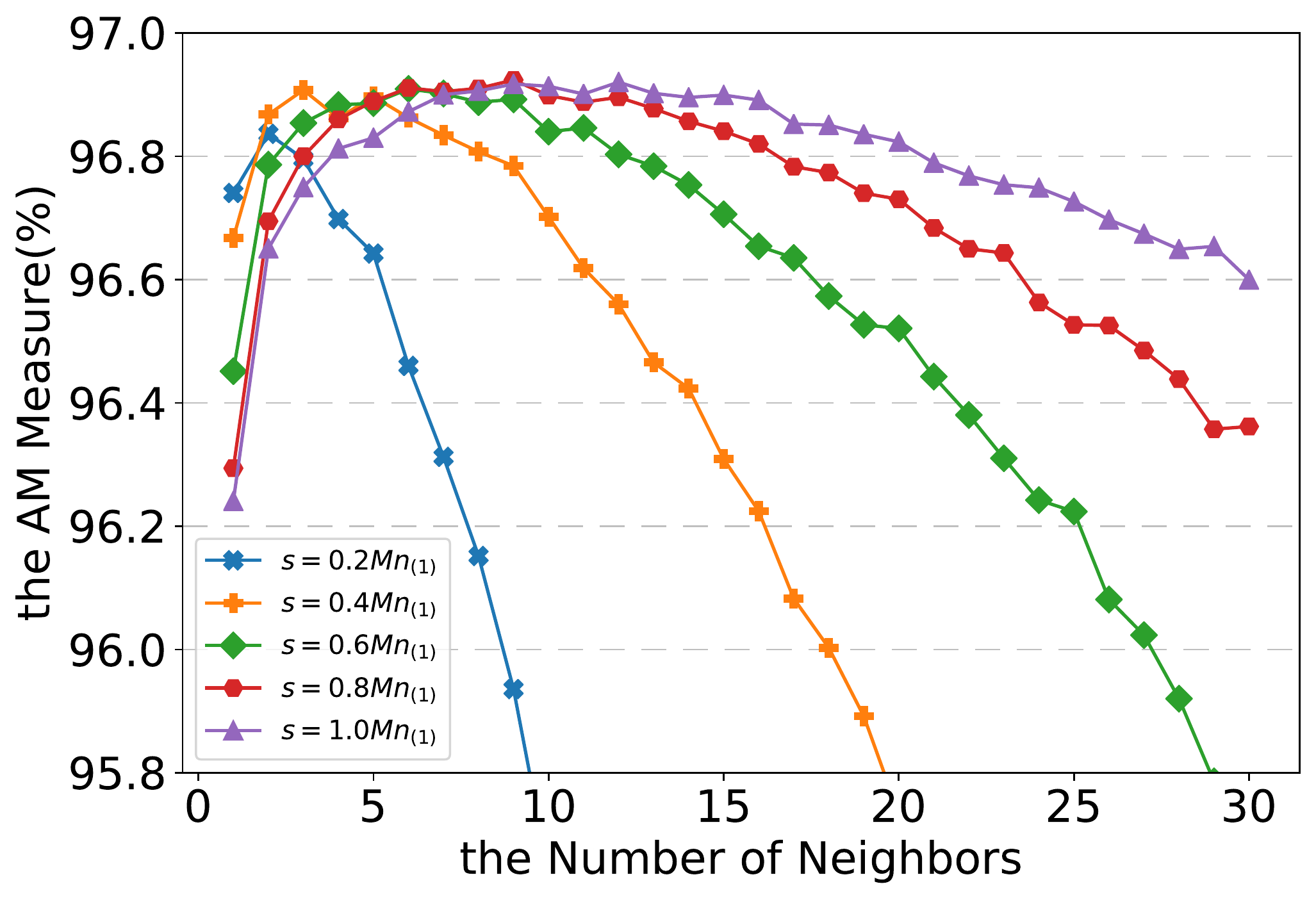}
	\captionsetup{justification=centering}
	\caption{Parameter analysis of the expected sub-sample size $s$.}
	\label{fig:param-sk}
\end{figure}

\subsection{Numerical Comparison on Real-world Imbalanced Datasets}

\subsubsection{Experimental Settings}

To verify the effectiveness of our proposed under-bagging $k$-NN for imbalanced classification, we conduct extensive experiments on nine real-world imbalanced datasets, including binary-class and multi-class data sets. These imbalanced datasets comes from the UCI Machine Learning Repository \cite{Dua:2019}. As some data sets contain missing values, we impute the missing values of numerical features and categorical features with the mean value and the most frequent value of the non-missing values for those features respectively. The details of datasets, including size and dimension are listed in Table \ref{tab::dataset-descrpitions}. Besides, we show the proportion of the majorities and minorities in each data set, and then calculate the imbalance ratio $\rho$. We mention that the number of features in Table \ref{tab::dataset-descrpitions} are provided after the preprocessing of the one-hot encoding for categorical features. We use the one-hot encoding to transform a categorical feature with $k$ categories into $k$ binary features. We apply standardization of datasets by scaling features to the range of $[0, 1]$. We apply 2 times 10-fold cross-validation with a total of 20 runs for repeated experiments, and then report the average results of AM measure and running time. All experiments are conducted on a 64-bit machine with 24-cores Intel Xeon 2.0GHz CPU (E5-4620) and 64GB main memory.

\begin{table}[!h]
	\centering
	\captionsetup{justification=centering}
	\caption{Description of Real-World Data Sets}
	\resizebox{\textwidth}{!}{
		\begin{tabular}{lrrrrrr}
			\toprule
			Datasets & \#Instances & \#Features & \#Classes & \%Majority &  \%Minority & $\rho$ \hspace{3mm} \\
			\midrule
			{\tt APSFailure}  & 76,000 & 170 & 2 & 98.19 & 1.81 & 3.62\%\\
			{\tt ActivityRecognition}  & 22,646 & 9 & 4 & 90.69 & 1.48 & 5.92\%\\
			{\tt Adult}  & 48,842 & 105 & 2 & 76.07 & 23.93 & 47.86\%\\
			{\tt Bitcoin}  & 2,916,697 & 6 & 2 & 98.58 & 1.42 & 2.84\%\\
			{\tt BuzzInSocialMedia}  & 140,707 & 77 & 2 & 99.16 & 0.84 & 1.68\%\\
			{\tt CencusIncomeKDD}   & 299,285 & 503 & 2 & 93.8 & 6.2 & 12.4\%\\
			{\tt CreditCardClients}    & 30,000 & 32 & 2 & 77.88 & 22.12 & 44.24\%\\
			{\tt OccupancyDetection}    & 20,560 & 5 & 2 & 76.9 & 23.1 & 46.2\%\\
			{\tt p53Mutants}   & 31,420 & 5408 & 2 & 99.52 & 0.48 & 0.96\%\\
			\bottomrule
		\end{tabular}
	}
	\label{tab::dataset-descrpitions}
\end{table}

We verify the effectiveness of our under-bagging technique for the $k$-NN classifier, we compare the following cases:
\begin{itemize}
	\item[\textit{(i)}] 
	The standard $k$-NN classifier;
	\item[\textit{(ii)}] 
	The $k$-NN classifier with under-sampling of the majority classes, corresponding to our under-bagging $k$-NN classifier with $B=1$ and $a=1$;
	\item[\textit{(iii)}] 
	Our under-bagging $k$-NN classifier with $B=5$ and $s=Mn_{(1)}$, which means that we use all samples of the minority class;
	\item[\textit{(iv)}] 
	Our under-bagging $k$-NN classifier with $B = 5$ and $s=0.5Mn_{(1)}$, which means that the expected sub-sample size turns out to be a half of that in \textit{(iii)} for efficient computing.
\end{itemize}
We use the {\tt scikit-learn}  and {\tt imbalanced-learn} implementations in {\tt Python} and tune the number of neighbors $k$ by cross-validation.

\subsubsection{Descriptions of Datasets}

The datasets are from the UCI machine learning repository \cite{Dua:2019}.
\begin{itemize}
	\item {\tt APSFailure}: 
	The \textit{APS Failure at Scania Trucks Data Set} \cite{Biteus2017} contains 76,000 samples and 170 attributes. The dataset is collected for the prediction of failures in the Air Pressure System of Scania heavy trucks.
	\item {\tt ActivityRecognition}: 
	The goal of the dataset \textit{Activity Recognition with Healthy Older People Using a Batteryless Wearable Sensor Data Set} \cite{torres2013sensor} is to predict the activity type using a batteryless, wearable sensor on top of elder people's clothing. There are two room settings (S1 and S2), and we choose the S2 setting, where the sample size is 22,646 and the feature size is 9.
	\item {\tt Adult}: 
	The \textit{Adult Data Set} is to predict whether the income is more than 50 thousands per year based on census data. It contains 48,842 data points with six numerical attributes and eight categorical attributes. We preprocess the categorical attributes by one-hot encodings, and the feature size of the preprocessed dataset is 105.
	\item {\tt Bitcoin}: 
	The \textit{Bitcoin Heist Ransomware Address Data Set} \cite{akcora2019bitcoinheist} has 2,916,697 data points and 6 predictive attributes. This dataset aims at identifying ransomware payments. The original labels are multiclass, containing the `white' label (i.e., not known to be ransomware) and a series of ransomware types. As the sample size of each ransomware type is small, we combine all ransomware types into a `ransomware' label  to build a binary-class imbalanced dataset.
	\item {\tt BuzzInSocialMedia}: 
	The \textit{Buzz in social media Data Set} \cite{kawala2013Predictions} contains two different social networks: Twitter and Tom's Hardware, and we choose the Twitter part. There are 140,707 samples with 77 features representing time-windows. The classification task is to predict whether these time-windows are followed by buzz events or not.
	\item {\tt CensusIncomeKDD}: 
	The \textit{Census-Income (KDD) Data Set} contains weighted census data extracted from the 1994 and 1995 current population surveys conducted by the U.S. Census Bureau. There are 299,285 instances with 40 categorical and numerical attributes. We use one-hot encodings to preprocess the categorical attributes, leading to 503 features in the end.
	\item {\tt CreditCardClients}: 
	The response variable of the \textit{Default of Credit Card Clients Data Set} \cite{yeh2009} is the default payment, with 23 numerical and categorical explanatory variables. We one-hot encode the categorical variables, and there are finally 32 attributes in total. The number of instances is 30,000.
	\item {\tt OccupancyDetection}: 
	The \textit{Occupancy Detection Data Set} \cite{candanedo2016accurate} uses temperature, humidity, light, and $\mathrm{CO}_2$ to predict room occupancy. The sample size is 20,560 and the number of predictive variables is 5.
	\item {\tt p53Mutants}: 
	The classification task of the \textit{p53 Mutants Data Set} \cite{danziger2006functional} is to predict the transcriptional activity (active vs inactive) of p53 proteins. There are 31,420 samples with 5,408 attributes.
\end{itemize}

\subsubsection{Experimental Results}

Table \ref{tab::real-recall} summarizes the averaged performance of the AM measure, and Table \ref{tab::real-runningtime} shows the computational performance w.r.t.~the averaged running time. The best results are marked in bold, the second best marked in \underline{underline}, and the standard deviations are shown in parentheses.

\begin{table}[!h]
	\centering
	\captionsetup{justification=centering}
	\caption{Average AM measure Score among Different Methods on Real-World Data Sets}
	\scriptsize
	\resizebox{\textwidth}{!}{
		\begin{tabular}{l|r|r|r|r}
			\toprule
			\multirow{2}{*}{Datasets} & \multirow{2}{*}{$k$-NN} & \multicolumn{2}{c|}{Under-bagging with $s=Mn_{(1)}$} &  \multicolumn{1}{c}{$s = 0.5Mn_{(1)}$} \\
			& & $B=1$ & $B=5$ & $B=5$ \\
			\midrule
			{\tt APSFailure} & 0.7917 (0.0226) & 0.9401 (0.0098) & \textbf{0.9464 (0.0100)} & \underline{0.9444 (0.0100)} \\
			{\tt ActivityRecognition} & 0.8439 (0.0228) & \underline{0.8499 (0.0181)} & \textbf{0.8571 (0.0203)} & 0.8485 (0.0210) \\
			{\tt Adult} & 0.7491 (0.0054) & 0.8020 (0.0067) & \textbf{0.8047 (0.0070)} & \underline{0.8046 (0.0057}) \\
			{\tt Bitcoin} & 0.5638 (0.0062) & 0.6630 (0.0031) & \textbf{0.6740 (0.0024)} & \underline{0.6694 (0.0032)} \\
			{\tt BuzzInSocialMedia} & 0.7959 (0.0215) & 0.9411 (0.0125) & \textbf{0.9433 (0.0122)} & \underline{0.9425 (0.0124)} \\
			{\tt CensusIncomeKDD} & 0.6799 (0.0058) & 0.8412 (0.0042) & \textbf{0.8459 (0.0038)} & \underline{0.8435 (0.0033)} \\
			{\tt CreditCardClients} & 0.6373 (0.0084) & 0.6829 (0.0095) & \textbf{0.6892 (0.0102)} & \underline{0.6868 (0.0094}) \\
			{\tt OccupancyDetection} & 0.9930 (0.0024) & \underline{0.9941 (0.0018)} & \textbf{0.9942 (0.0019)} & 0.9938 (0.0020) \\
			{\tt p53Mutants} & 0.7733 (0.0601) & 0.8824 (0.0449) & \textbf{0.9006 (0.0408)} & \underline{0.8911 (0.0460)} \\
			\bottomrule
		\end{tabular}
	}
	%\begin{tablenotes}
	%\footnotesize
	%\item  The best results are marked in \textbf{bold} and the second best marked in \underline{underline}. The standard deviation is reported in the parenthesis beside each value. 
	%\end{tablenotes}
	\label{tab::real-recall}
\end{table}

\begin{table}[!h]
	\centering
	\captionsetup{justification=centering}
	\caption{Average Running Time among Different Methods on Real-World Data Sets}
	\scriptsize
	\resizebox{\textwidth}{!}{
		\begin{tabular}{l|r|r|r|r}
			\toprule
			\multirow{2}{*}{Datasets} & \multirow{2}{*}{$k$-NN} & \multicolumn{2}{c|}{Under-bagging with $s=Mn_{(1)}$} &  \multicolumn{1}{c}{$s=0.5Mn_{(1)}$} \\
			& & $B=1$ & $B=5$ & $B=5$ \\
			\midrule
			{\tt APSFailure} & 5.33 (0.07) & \textbf{0.48 (0.07)} & 0.84 (0.07) & \underline{0.52 (0.05)} \\
			{\tt ActivityRecognition} & 0.56 (0.10) & \textbf{0.06 (0.02)} & 0.17 (0.04) & \underline{0.13 (0.02)} \\
			{\tt Adult} & 4.37 (0.15) & \underline{2.05 (0.16)} & 2.92 (0.14) & \textbf{1.59 (0.10)} \\
			{\tt Bitcoin} & 16885.49 (834.14) & \underline{10.21 (1.42)} & 11.88 (1.16) & \textbf{8.11 (0.86)} \\
			{\tt BuzzInSocialMedia} & 21.86 (5.36) & \underline{0.69 (0.10)} & 1.08 (0.13) & \textbf{0.65 (0.10)} \\
			{\tt CensusIncomeKDD} & 107.43 (0.41) & \underline{21.11 (1.01)} & 37.22 (0.60) & \textbf{19.00 (0.81)} \\
			{\tt CreditCardClients} & 1.51 (0.20) & \underline{0.71 (0.05)} & 1.02 (0.06) & \textbf{0.65 (0.04)} \\
			{\tt OccupancyDetection} & 0.49 (0.03) & \textbf{0.24 (0.05)} & 0.40 (0.06) & \underline{0.25 (0.03)} \\
			{\tt p53Mutants} & 2.72 (0.00) & \textbf{1.51 (0.07)} & 3.20 (0.08) & \underline{2.22 (0.07)} \\
			\bottomrule
		\end{tabular}
	}
	%\begin{tablenotes}
	%\footnotesize
	%\item  The best results are marked in \textbf{bold} and the second best marked in \underline{underline}. The standard deviation is reported in the parenthesis beside each value. 
	%\end{tablenotes}
	\label{tab::real-runningtime}
\end{table}

The numerical experiments verify the theoretical results in the following ways:

\textit{(i)} For imbalanced classification, under-sampling $k$-NN and under-bagging $k$-NN classifier significantly outperform the standard $k$-NN classifier by a large margin, which empirically verifies the results in Theorems \ref{thm::knnundersample} and \ref{thm::bagknnclassund} that under-sampling and under-bagging $k$-NN classifiers converge to the optimal classifier w.r.t.~the AM measure, whereas the standard $k$-NN turns out to be inconsistent w.r.t.~this measure.

\textit{(ii)} The theoretical relationship between the expected sub-sample size $s$ w.r.t.~$n$ shown in Theorem \ref{thm::bagknnclassund} is also numerically verified.
When we compare our under-bagging $k$-NN with $s=Mn_{(1)}$ to our under-bagging $k$-NN with $s=0.5Mn_{(1)}$, we find in Tables \ref{tab::real-recall} and \ref{tab::real-runningtime} that the AM performance trained with smaller sub-sample size does not degrade too much, and sometimes can be even slightly better, whereas the running time of our under-bagging $k$-NN with $s=0.5Mn_{(1)}$ is much smaller than that with $s=Mn_{(1)}$. 
This indicates that we can use a relative small sub-sample size to speed up the under-bagging $k$-NN and still keep a good performance in terms of the AM measure. 

\textit{(iii)}
The running time of our under-bagging $k$-NN classifier demonstrates its computational efficiency compared with the standard $k$-NN classifier, especially when the sample size is large.
As our complexity analysis in Section \ref{sec::complexity} show, the time complexity of construction can be reduced from $\mathcal{O}(n \log n)$ (for the standard $k$-NN) to $\mathcal{O}((\rho n \log (\rho n))^{d/(2\alpha+d)})$, and the time complexity in the testing stage can be reduced from $\mathcal{O}((n \log n)^{2\alpha/(2\alpha+d)})$ to $\mathcal{O}(\log^2 (\rho n))$.
To verify the complexity results, we compare the running time of the under-bagging $k$-NN ($B=1$) with that of the standard $k$-NN, and we observe that the under-bagging technique significantly reduces the running time.
In Table \ref{tab::real-runningtime}, the under-bagging technique can reduce at least half the running time of the standard $k$-NN.
In particular, on {\tt Bitcoin}, where the sample size is up to $3$ million and the imbalance ratio is about $3\%$, the running time can even be reduced by $99.4\%$.
Besides, when we adopt more bagging rounds $B=5$, the performance of the under-bagging $k$-NN w.r.t.~the AM measure further enhances with a mild increase of the running time.

To conclude, these results show the importance of the combination of bagging and under-sampling:
the under-sampling technique not only improves the AM performance in imbalanced classification, but also saves computational time;
while the bagging technique can further improve the performance without too much sacrifice of running time, since the bagging technique is a native parallel-friendly method.

\section{Proofs}\label{sec::proofs}

In this section, we first prove the fundamental results related to the AM measure in Section \ref{sec::erroranalysis}, which play an essential role in establishing the convergence rates for both under-sampling and under-bagging $k$-NN classifiers. Then in Section \ref{sec::proofunder} and Section \ref{sec::proofunderbagging}, we present  the proofs of the theoretical results related to the under-sampling $k$-NN in Section \ref{sec::errordec} and the under-bagging $k$-NN in Section \ref{sec::error_bagging}, respectively.

\vspace{0.2cm}

\begin{proof}[Proof of Proposition \ref{pro::eqiambalance}]
	By the definition of the balanced loss in \eqref{equ::lbal}, we have
	\begin{align*}
		1 - \mathcal{R}_{L_{\mathrm{bal}}, \mathrm{P}}(f)
		& =	1 - \sum_{i=1}^M \mathrm{P}(y = i, f(x) \neq y) / (M \pi_i)
		= \frac{1}{M} \sum_{i=1}^M (\pi_i - \mathrm{P}(y = i, f(x) \neq y)) / \pi_i
		\\
		& = \frac{1}{M} \sum_{i=1}^M \mathrm{P}(y = i, f(x) = i) / \pi_i
		= \frac{1}{M} \sum_{i=1}^M \mathrm{P}(f(x) = i | y = i)
		= r_{\mathrm{AM}}(f).
	\end{align*}
	This completes the proof of Proposition \ref{pro::eqiambalance}.
\end{proof}

The following lemma is needed in the proof of Theorem \ref{thm::connection}, which provides a new formulation of the excess risk w.r.t.~the balanced loss.

\begin{lemma}\label{equ::excessbalrisk}
	Let the balanced loss be defined in \eqref{equ::bayes_bal}. Then we have
	\begin{align*}
		\mathcal{R}_{L_{\mathrm{bal}}, \mathrm{P}}(f) - \mathcal{R}_{L_{\mathrm{bal}},
			\mathrm{P}}^*
		=  \sum^n_{i=1} (\eta_m(x) / (M \pi_m)) \mathbb{E}_{\mathrm{P}_X} \bigl[ {\eta^w_{L_{\mathrm{bal}}, \mathrm{P}}(x) - \eta^w_{f(x)}(x)} \bigr],
	\end{align*}
	where we write $\eta^w_{L_{\mathrm{bal}}, \mathrm{P}}(x) := \eta^w_{f_{L_{\mathrm{bal}}, \mathrm{P}}^*(x)}(x)$.
\end{lemma}

\begin{proof}[Proof of Lemma \ref{equ::excessbalrisk}]
	By the definition of the balanced loss \eqref{equ::bayes_bal}, we have
	\begin{align*}
		\mathcal{R}_{L_{\mathrm{bal}}, \mathrm{P}}(f)
		& = \mathbb{E}_{\mathrm{P}_X} \biggl[ \mathbb{E}_{\mathrm{P}_{Y|X}} \biggl[ \sum_{m=1}^M \eins \{ Y = m \} \eins \{ f(x) \neq Y \} / (M \pi_m) \biggr] \bigg| X = x \biggr]
		\\
		& = \mathbb{E}_{\mathrm{P}_X}  \biggl[ \mathbb{E}_{\mathrm{P}_{Y|X}}  \biggl[ \sum_{j=1}^M \biggl( \eins \{ f(x) = j \} \sum_{m=1}^M \eins \{ Y = m \} \eins \{ Y \neq j \} / (M \pi_m) \biggr) \biggr] \bigg| X = x \biggr]
		\\
		& = \mathbb{E}_{\mathrm{P}_X}  \biggl[ \sum_{j=1}^M \biggl( \eins \{ f(x) = j \} \mathbb{E}_{\mathrm{P}_{Y|X}}  \biggl[ \sum_{m=1}^M \eins \{ Y = m \} \eins \{ Y \neq j \} / (M \pi_m) \biggr] \biggr) \bigg| X = x \biggr] 
		\\
		& = \mathbb{E}_{\mathrm{P}_X}  \biggl[ \sum_{j=1}^M \biggl( \eins \{ f(x) = j \} \sum_{m \neq j}^M \eta_m(x) /(M \pi_m) \biggr) \biggr]
		\\
		& = \mathbb{E}_{\mathrm{P}_X}  \biggl[ \sum_{m \neq f(x)}^M \eta_m(x) / (M \pi_m) \biggr].
	\end{align*}
	Consequently we find
	\begin{align*}
		\mathcal{R}_{L_{\mathrm{bal}}, \mathrm{P}}(f) - \mathcal{R}_{L_{\mathrm{bal}}, \mathrm{P}}^*
		& = \mathbb{E}_{\mathrm{P}_X} \biggl[ \sum_{m \neq f(x)}^M \eta_m(x) / (M \pi_m) \biggr] 
		- \mathbb{E}_{\mathrm{P}_X} \biggl[ \sum_{m \neq f_{L_{\mathrm{bal}}, \mathrm{P}}^*(x)}^M \eta_m(x) / (M \pi_m) \biggr]
		\nonumber\\
		& = \sum^n_{i=1} (\eta_m(x) / (M \pi_m)) \mathbb{E}_{\mathrm{P}_X} \bigl[ {\eta^w_{L_{\mathrm{bal}}, \mathrm{P}}(x) - \eta_{f(x)}^w(x)} \bigr],
	\end{align*}
	where $\eta^w_{L_{\mathrm{bal}}, \mathrm{P}}(x) = \eta^w_{f_{L_{\mathrm{bal}}, \mathrm{P}}^*(x)}(x)$. Thus we obtain the assertion.
\end{proof}

\begin{proof}[Proof of Theorem \ref{thm::connection}]
	An elementary calculation yields 
	\begin{align*}
		\mathcal{R}_{L_{\mathrm{cl}}, \mathrm{P}^w}({f}) - \mathcal{R}_{L_{\mathrm{cl}},{\mathrm{P}^w}}^*
		& = \mathbb{E}_{\mathrm{P}^w_X} \bigl[ \mathbb{E}_{\mathrm{P}^w_{Y|X}} [ \eins \{ f(x) \neq Y \} - \eins \{ f^*_{L_{\mathrm{cl}},\mathrm{P}^w}(x) \neq Y \} ] \big| X = x \bigr]
		\\
		& = \mathbb{E}_{\mathrm{P}^w_X} \bigl[ \eta^{w,*}_{L_{\mathrm{cl}},\mathrm{P}^w}(x) 
		- \eta^w_{f(x)}(x) \bigr],
	\end{align*}
	where we write $\eta^{w,*}_{L_{\mathrm{cl}},\mathrm{P}^w}(x) := \eta^w_{{f}_{L_{\mathrm{cl}},\mathrm{P}^w}^*(x)}(x)$. By \eqref{equ::f01fbal}, we find
	\begin{align*}
		\mathcal{R}_{L_{\mathrm{cl}}, \mathrm{P}^w}({f}) - \mathcal{R}_{L_{\mathrm{cl}},\mathrm{P}^w}^*
		= \mathbb{E}_{\mathrm{P}^w_X} \bigl[ \eta^{w,*}_{L_{\mathrm{bal}},{\mathrm{P}}}(x) 
		- \eta^w_{f(x)}(x) \bigr],
	\end{align*}
	where we write $\eta^{w,*}_{L_{\mathrm{cl}},\mathrm{P}^w} := \eta^w_{f_{L_{\mathrm{bal}},\mathrm{P}}^*(x)}	(x)$. This together with \eqref{equ::tildefx} implies
	\begin{align*}
		\mathcal{R}_{L_{\mathrm{cl}}, \mathrm{P}^w}(f) - \mathcal{R}_{L_{\mathrm{cl}}, \mathrm{P}^w}^* 
		= \sum^n_{i=1} (\eta_m(x) / (M \pi_m)) \cdot \mathbb{E}_{\mathrm{P}_X} \bigl[ {\eta^w_{L_{\mathrm{bal}}, \mathrm{P}}(x) - \eta_{f(x)}^w(x)} \bigr].
	\end{align*}
	Combining this with Lemma \ref{equ::excessbalrisk}, we obtain the assertion.
\end{proof}

To prove Proposition \ref{pro::regressionerror}, we need the following Lemmas \ref{lem::convergereg} and \ref{lem::regexcessrisk}, which reduce the problem of analyzing the excess classification risk to the problem of analyzing the error estimation of posterior probability function.

\begin{lemma}\label{lem::convergereg}
	Let $\widehat{\eta}:\mathcal{X}\to [0,1]^M$ be an estimate of $\eta^w$ and $\widehat{f}(x) = \argmax_{m\in[M]} \widehat{\eta}_m(x)$. Suppose that there exists $\phi$ such that $(\mathrm{P}^w)^n ( \|\widehat{\eta}(x) - \eta^w(x)\|_{\infty} \leq \phi) \geq 1 - \delta$. Then with probability $(\mathrm{P}^w)^n$ at least $1 - \delta$, 
	there holds $\|\eta^{w,*}_{L_{\mathrm{cl}},\mathrm{P}^w}(x)-  \eta^w_{\widehat{f}(x)}(x)\|_{\infty} \leq 2 \phi$,
	where $\eta^{w,*}_{L_{\mathrm{cl}},\mathrm{P}^w}(x) = {\eta}_{{f}_{L_{\mathrm{cl}},\mathrm{P}^w}^*(x)}(x)$. 
\end{lemma}

\begin{proof}[Proof of Lemma \ref{lem::convergereg}]
	Fix an $x \in \mathcal{X}$ with $\|\widehat{\eta}(x) - \eta^w(x)\|_{\infty} \leq \phi$. Let 
	$m^* = \argmax_{m \in [M]} \eta^w_m(x)$ and $m = \widehat{f}(x)$. 
	Then we have $\eta^w_{m^*}(x) \leq \widehat{\eta}_{m^*}(x) + \phi$ and $\eta^w_{m}(x) \geq \widehat{\eta}_m(x) - \phi$. Thus, we find
	\begin{align*}
		\eta^w_{m^*}(x) - \eta^w_{m}(x)
		\leq (\widehat{\eta}_{m^*}(x) + \phi) - (\widehat{\eta}_m(x) - \phi)
		= (\widehat{\eta}_{m^*}(x) - \widehat{\eta}_m(x)) + 2 \phi.
	\end{align*}
	Since $m = \widehat{f}(x)$ is the maximum entry of $\widehat{\eta}_m(x)$, we have $\widehat{\eta}_{m^*}(x) \leq \widehat{\eta}_m(x)$ and consequently $\eta^w_{m^*}(x) - \eta^w_{m}(x) \leq 2 \phi$.
	In other words, we show that the event $\{ x \in \mathcal{X} : \|\widehat{\eta}(x) - \eta^w(x)\|_{\infty} \leq \phi \}$ is contained in $\{ \eta^w_{m^*}(x) - \eta^w_{m}(x) \leq 2 \phi\}$. This implies that for all $x\in \mathcal{X}$, with probability $(\mathrm{P}^w)^n$ at least $1 - \delta$, there holds $\eta^{w,*}_{L_{\mathrm{cl}},\mathrm{P}^w}(x) - \eta^w_{\widehat{f}(x)}(x) \leq 2 \phi$, which finishes the proof.
\end{proof}

\begin{lemma}\label{lem::regexcessrisk}
	Let $\widehat{\eta}:\mathcal{X}\to [0,1]^M$ be an estimate of $\eta^w$ and $\widehat{f}(x) = \argmax_{m\in[M]} \widehat{\eta}_m(x)$. Moreover, let Assumption \ref{ass:imbalance} hold. Suppose that 
	$\|\eta^{w,*}_{L_{\mathrm{cl}},\mathrm{P}^w}-  \eta^w_{\widehat{f}(x)}\|_{\infty}
	\leq 2 \phi$ holds for some $\phi > 0$ with probability $(\mathrm{P}^w)^n$ at least $1 - \delta$, where $\eta^{w,*}_{L_{\mathrm{cl}},\mathrm{P}^w}(x) = \eta^w_{{f}_{L_{\mathrm{cl}},\mathrm{P}^w}^*(x)}(x)$. Then with probability $(\mathrm{P}^w)^n$ at least $1 - \delta$, there holds
	\begin{align*}
		\mathcal{R}_{L_{\mathrm{cl}},\mathrm{P}^w}(\widehat{f})
		- \mathcal{R}_{L_{\mathrm{cl}},\mathrm{P}^w}^*
		\leq (c_{\beta}/(M \underline{\pi}) (2 \phi)^{\beta+1}
	\end{align*}
	
\end{lemma}

\begin{proof}[Proof of Lemma \ref{lem::regexcessrisk}]
	Let $m=f^*_{L_{\mathrm{cl}},\mathrm{P}^w}(x)$. By the definition of the excess risk, we have
	\begin{align*}
		\mathcal{R}_{L,\mathrm{P}^w}(\widehat{f}(x)) - \mathcal{R}_{L,\mathrm{P}^w}^*
		= \mathbb{E}_{\mathrm{P}^w_X} \bigl[ \eta^w_{m}(x) - \eta^w_{\widehat{f}(x)}(x) \bigr].
	\end{align*}
	Let $\eta^w_{(m)}(x)$ denote the $m$-th largest entry of the vector $\eta^w(x) = (\eta^w_1(x), \ldots, \eta^w_M(x))^{\top}$ and define $\delta(x) := \eta^w_{(1)}(x) - \eta^w_{(2)}(x)$. Moreover, let $\Delta_n := \{ x : \mathcal{X} : \delta(x) \geq 2 \phi\}$. Then we have
	\begin{align*}
		\mathbb{E}_{\mathrm{P}^w_X} [ \eta^w_{m}(x) - \eta^w_{\widehat{f}(x)}(x)]
		= \int_{\Delta_n} (\eta^w_{m}(x)-\eta^w_{\widehat{f}(x)}(x)) \, d\mathrm{P}^w_X(x)
		+ \int_{\Delta_n^c} (\eta^w_{m}(x)-\eta^w_{\widehat{f}(x)}(x)) \, d\mathrm{P}^w_X(x).
	\end{align*}
	Now we consider the two regions $\Delta_n$ and $\Delta_n^c$ separately to bound the error. If $\delta(x) \geq 2 \phi$, since $\eta^w_{(1)}(x)=\eta^w_{m}(x)$, we have
	$\eta^w_{m}(x)-\eta^w_{\widehat{f}(x)}(x) < 2 \phi
	\leq \eta^w_{(1)}(x) - \eta^w_{(2)}(x)$.
	In other words, we find $\eta^w_{\widehat{f}(x)}(x)$ is larger than $\eta^w_{(2)}(x)$, which yields $\eta^w_{\widehat{f}(x)}(x) = \eta^w_{(1)}(x) = \eta^w_{m}(x)$. Consequently, we obtain
	\begin{align}\label{equ::errterm1}
		\int_{\Delta_n} (\eta^w_{m}(x) - \eta^w_{\widehat{f}(x)}(x)) \, d\mathrm{P}^w_X(x) 
		= 0.
	\end{align}
	Otherwise if $\delta(x) \leq 2\phi$, then by \eqref{equ::tildefx}, we have
	\begin{align*}
		\int_{\Delta_n^c} (\eta^w_{m}(x) - \eta^w_{\widehat{f}(x)}(x)) \, d\mathrm{P}^w_X(x)
		& =	\int_{\Delta_n^c} \sum^M_{m=1} (\eta_m(x) / (M \pi_m)) (\eta^w_{m}(x) - \eta^w_{\widehat{f}(x)}(x)) \, d\mathrm{P}_X(x)
		\\
		& \leq \frac{1}{M \underline{\pi}}\int_{\Delta_n^c} (\eta^w_{m}(x) - \eta^w_{\widehat{f}(x)}(x)) \, d\mathrm{P}_X(x)
		\\
		& \leq \mathrm{P}_X(\delta(x) \leq 2\phi) / (M \underline{\pi}).
	\end{align*}
	Using Condition $(i)$ in Assumption \ref{ass:imbalance}, we get
	\begin{align}\label{equ::errterm2}
		\int_{\Delta_n^c} (\eta^w_{m}(x) - \eta^w_{\widehat{f}(x)}(x)) \, d\mathrm{P}^w_X(x)
		\leq (c_{\beta} / (M \underline{\pi})) (2 \phi)^{\beta+1}.
	\end{align}
	Combining \eqref{equ::errterm1} and \eqref{equ::errterm2}, we obtain the assertion.
\end{proof}

\begin{proof}[Proof of Proposition \ref{pro::regressionerror}] 
	Proposition \ref{pro::regressionerror} is a straightforward consequence of Lemma \ref{lem::convergereg} and Lemma \ref{lem::regexcessrisk}.
\end{proof}

Before we proceed, we list two lemmas that will be used frequently in the proofs. Lemma \ref{lem::hoffeding} is Hoeffding's inequality, which was established in \cite{hoeffding1994probability} and Lemma \ref{lem::bernstein} is Bernstein's inequality, which was introduced in \cite{bernstein1946theory}. Both concentration inequalities can be found in many statistical learning textbooks, see e.g., \cite{massart2007concentration, cucker2007learning, steinwart2008support}.

\begin{lemma}[Hoeffding's inequality] \label{lem::hoffeding}
	Let $a<b$ be two real numbers, $n\geq 1$ be an integer, and $\xi_1,\ldots,\xi_n$ be independent random variables satisfying $\xi_i\in [a,b]$, for $1\leq i\leq n$. Then, for all $\tau>0$, we have
	\begin{align*}
		\mathrm{P}\biggl(\frac{1}{n}\sum_{i=1}^n (\xi_i-\mathbb{E}_{\mathrm{P}}\xi_i)\geq (b-a)\sqrt{\frac{\tau}{2n}}\biggr)\leq e^{-\tau}.
	\end{align*}
\end{lemma}

\begin{lemma}[Bernstein's inequality] \label{lem::bernstein}
	Let $B>0$ and $\sigma>0$ be real numbers, and $n\geq 1$ be an integer. Furthermore, let $\xi_1,\ldots,\xi_n$ be independent random variables satisfying $\mathbb{E}_{\mathrm{P}}\xi_i=0$, $\|\xi_i\|_\infty\leq B$, and $\mathbb{E}_{\mathrm{P}}\xi^2\leq \sigma^2$ for all $i=1,\ldots,n$. Then for all $\tau>0$, we have
	\begin{align*}
		\mathrm{P}\biggl(\frac{1}{n}\sum_{i=1}^n\xi_i\geq \sqrt{\frac{2\sigma^2\tau}{n}}+\frac{2B\tau}{3n}\biggr)\leq e^{-\tau}.
	\end{align*}
\end{lemma}

\subsection{Proofs Related to the Under-sampling $k$-NN Classifier} \label{sec::proofunder}

In this section, we first present in Sections \ref{proofsecerror}-\ref{sec::proofundersampling} the proof of the theoretical results on bounding the sample error in Section \ref{sec::sampleunder}, the approximation error in Section \ref{sec::approxunder}, and the under-sampling error in Section \ref{sec::undersamplingerror}, respectively. Then, in Section \ref{sec::proofknn}, we prove the main result on the convergence rates of the under-sampling $k$-NN classifier and the minimax lower bound, i.e., Theorems \ref{thm::knnundersample} and \ref{thm::lowerbound} in Section \ref{sec::knnmulticlassim}.

\subsubsection{Proofs Related to Section \ref{sec::sampleunder}} \label{proofsecerror}

The following lemma providing an explicit expression for the under-sampling distribution discussed in Section \ref{sec::errordec}, which supplies the key to the proof of Lemma \ref{lem::punder} and Proposition \ref{thm::punderund}.

\begin{lemma}\label{lem::pu}
	Let $\mathrm{P}^u$ be the probability distribution of the accepted samples by the under-sampling strategy in Section \ref{sec::knnwithuniform}.
	Then we have
	\begin{align*}
		\mathrm{P}^u(X \in A, Y = m)	
		= \frac{\int_A \eta_m(x) f_X(x) \, dx / n_{(m)}}{\sum_{m=1}^M \pi_m / n_{(m)}}.
	\end{align*}
	Moreover, the marginal distribution can be expressed as 
	\begin{align}\label{equ::pimu}
		\pi_m^u
		:= \mathrm{P}^u(Y = m)
		= \frac{\int_{\mathcal{X}} \eta_m(x) f_X(x) \, dx / n_{(m)}}{\sum_{m=1}^M \pi_m / n_{(m)}}
		= \frac{\pi_m / n_{(m)}}{\sum_{m=1}^M \pi_m / n_{(m)}}.
	\end{align}
	In addition, the conditional density function is 
	$f^u(x | y = m) = \eta_m(x) f_X(x) / \pi_m$
	and the posterior probability function is given by
	\begin{align}\label{equ::etaundlem}
		\eta_m^u(x)
		:= \mathrm{P}^u(Y = m | X = x)
		= \frac{\pi_m^u f^u(x | y = m)}{\sum_{m=1}^M \pi_m^u f^u(x | y = m)}
		= \frac{\eta_m(x) / n_{(m)}}{\sum_{m=1}^M \eta_m(x) / n_{(m)}}
	\end{align}
	Furthermore, the marginal distribution $f^u_X(x)$ can be expressed as
	\begin{align}\label{equ::fundxx}
		f^u_X(x)
		= \sum_{m=1}^M \pi_m^u(x) f(x | y = m).
	\end{align}
\end{lemma}

\begin{proof}[Proof of Lemma \ref{lem::pu}]
	Let ${\mathrm{P}}_{X,Y,Z} = \mathrm{P}_{X,Y} \times \mathrm{P}_{Z|(X,Y)}$ denote the joint probability measure. Then we can calculate the probability of  $Z(X,Y) = 1$, that is, $(X,Y)$ is accepted in the under-sampling strategy, as follows:
	\begin{align*}
		\mathrm{P}_{X,Y,Z}(Z(X, Y) = 1)
		& = \int_{\mathcal{X} \times \mathcal{Y}} \mathrm{P}(Z(X, Y) = 1 | (X, Y) = (x, y)) \, d\mathrm{P}(x, y)
		\\
		& = \sum_{m=1}^M \pi_m \int_{\mathcal{X}} \mathrm{P}(Z(X, Y) = 1 | (X, Y) = (x, m)) \, d\mathrm{P}(x | y = m)
		\\
		& = \sum_{m=1}^M \pi_m \int_{\mathcal{X}} a(x, m) f(x | y = m) \, dx.
	\end{align*}
	Thus, for any measurable set $A$ of $\mathcal{X}$, an elementary calculation yields 
	\begin{align*}
		& \mathrm{P}_{X,Y,Z}(X \in A,Y = m, Z(X, Y) = 1)
		\\
		& = \int_{\mathcal{X} \times \mathcal{Y}} \eins \{ x \in A \} \eins \{ Y = m \} \mathrm{P}({Z(x,y) = 1} | (X, Y) = (x, y)) \, d\mathrm{P}(x, y)
		\\
		& = \pi_m \int_{\mathcal{X}} \eins \{ x \in A \} \mathrm{P}(Z(x, m) = 1 | (X, Y) = (x, m)) \, d\mathrm{P}(x | y = m)
		\\
		& = \pi_m \int_{\mathcal{X}} \eins \{ x \in A \} a(x, m) f(x | y = m) \, dx
		\\
		& = \pi_m \int_A a(x, m) f(x | y = m) \, dx.
	\end{align*}
	Consequently, the {distribution function of the accepted samples} is given by
	\begin{align*}
		\mathrm{P}^u(X \in A, Y = m)
		& = \mathrm{P}_{X,Y,Z}(X \in A, Y = m | Z(X, Y) = 1)
		\\
		& = \frac{\mathrm{P}^u(X \in A, Y = m, Z(X, Y) = 1)}{\mathrm{P}^u(Z(X, Y) = 1)}
		\\
		& = \frac{\pi_m \int_A a(x, m) f(x | y = m) \, dx}{\sum_{m=1}^M \pi_m \int_{\mathcal{X}} a(x,m) f(x | y = m) \, dx}.
	\end{align*}
	Combining this with \eqref{equ::axymn}, we find
	\begin{align}\label{equ::puxbymnew}
		\mathrm{P}^u(X \in A, Y = m)
		= \mathrm{P}(X \in A, Y = m | Z(X, Y) = 1)
		= \frac{\int_A \eta_m(x) f_X(x) \, dx / n_{(m)}}{\sum_{m=1}^M \pi_m / n_{(m)}}.
	\end{align}
	Thus we finish the proof with straightforward application of the joint probability measure in \eqref{equ::puxbymnew} for calculating $\pi_m^u$, $f^u(x|y=m)$, $\eta_m^u$ and $f^u_X(x)$.
\end{proof}

To prove Proposition \ref{prop::BnnEstimation}, we need to bound the number of reorderings of the data. To be specific, let $X_1,\ldots,X_n\in \mathbb{R}^d$ be some vectors and $X_{(i)}(x)$ denote the $i$-th nearest neighbor of $x$, then there exits a permutation $(\sigma_1,\ldots,\sigma_n)$ such that $X_{(i)}(x) = X_{\sigma_i}(x)$. Then we define the inverse of the permutation, namely the rank $\Sigma_i$ by $\Sigma_i := \{ 1 \leq \ell \leq n : X_{\sigma_{\ell}} = X_i \}$.
Let $\mathcal{S} = \{(\Sigma_1,\ldots,\Sigma_n),x\in \mathbb{R}^d\}$ be the set of all rank vectors one can observe by moving $x$ around in space and we use the notation $|\mathcal{S}|$ to represent the cardinality of $\mathcal{S}$.

The next lemma provides the upper bound for the number of reorderings, which plays a crucial rule to derive the uniform bound for the proof of Propositions \ref{prop::BnnEstimation}, \ref{pro::fml*funder} and \ref{pro::tildeffunder}.

\begin{lemma}\label{lem::numberreorder}
	For any $d\geq 1$  and all $n \geq 2d$, there holds $|\mathcal{S}| \leq (25/d)^dn^{2d}$.
\end{lemma}

\begin{proof}[Proof of Lemma \ref{lem::numberreorder}]
	The hyperplane $\|x-X_i\|^2 = \|x-X_j\|^2$ generates a sign function
	\begin{align*}
		p_{ij}(x)=
		\begin{cases}
			1 & \text{ if } \|x-X_i\|^2 > \|x-X_j\|^2, \\
			0 &  \text{ if } \|x-X_i\|^2 = \|x-X_j\|^2, \\
			-1 &  \text{ if } \|x-X_i\|^2 < \|x-X_j\|^2.
		\end{cases}
	\end{align*}
	The collection of the signs $\{ p_{ij}(x), 1 \leq i \leq j \leq n \}$, called the sign pattern, determines the ordering of $\|x-x_i\|^2$ and identifies all ties.  \cite{milnor1964betti,warren1968lower,pollack1993number} prove that the maximal number of sign pattern is not larger than $(25/d)^d n^{2d}$ for any $d\geq 1$ and $n\geq 2d$. Thus we show that $|\mathcal{S}|\leq (25/d)^dn^{2d}$.
\end{proof}

The next lemma states that the under-sampling distribution $\mathrm{P}^u$ is relatively close to the balanced distribution $\mathrm{P}^w$ with high probability. To be specific, \eqref{equ::con1} bounds the sub-sample size of each class $n_{(m)}$ and the marginal density function $f_X^u(x)$. In addition, \eqref{equ::con3} implies that if $\eta_m(x)$ is assumed to be $\alpha$-H\"{o}lder continuous, $\eta_m^u(x)$ remains to be $\alpha$-H\"{o}lder continuous. This lemma will be used several times in the sequel, which is crucial in the proof of Propositions \ref{prop::BnnEstimation} and \ref{prop::knnundersample}.

\begin{lemma}\label{lem::punder}
	Let $\eta^w_m(x)$, $\eta_m^u(x)$ and $f^u_X(x)$ be defined as in \eqref{equ::etamstarx}, \eqref{equ::etaund}, and \eqref{equ::fundxx}, respectively. Assume that $\mathrm{P}$ satisfies Assumption \ref{ass:imbalance} and $\mathrm{P}_X$ is the uniform distribution on $[0,1]^d$. Then there exists an $n_1\in \mathbb{N}$ such that for all $1 \leq m \leq M$ and all $n\geq n_1$, there hold
	\begin{align}\label{equ::con1}
		n \pi_m/2 \leq n_{(m)}
		\qquad 
		\text{and} 
		\qquad
		1/(2 M \overline{\pi})
		\leq f_X^u(x) 
		\leq 2/(M \underline{\pi})
	\end{align}
	with probability $\mathrm{P}^n$ at least $1-2M/n^3$. Moreover, for $x, x' \in \mathcal{X}$, we have
	\begin{align}\label{equ::con3}
		|\eta_m^u(x') - \eta_m^u(x)|
		\leq 4 c_L \|x - x'\|^{\alpha}.
	\end{align}
\end{lemma}

\begin{proof}[Proof of Lemma \ref{lem::punder}]
	For any $1 \leq m \leq M$, let $\zeta_i := \eins \{ Y_i = m \} - \pi_m$. Then $\zeta_i$'s are independent random variables such that $\mathbb{E}_{\mathrm{P}}[\zeta_i] = 0$ and $\mathbb{E}_{\mathrm{P}} \zeta_i^2 \leq 1/4$ for $1 \leq i \leq n$. Using Bernstein's inequality in Lemma \ref{lem::bernstein}, we obtain that for any $\tau > 0$, there holds
	\begin{align*}
		\mathrm{P}^n \biggl( \frac{1}{n} \sum_{i=1}^n \eins \{ Y_i = m \} \geq \pi_m + \sqrt{\frac{\tau}{2n}} + \frac{2\tau}{3n} \biggr)
		\leq e^{-\tau}.
	\end{align*}
	Setting $\tau := 3 \log n$, we get
	\begin{align}\label{equ::nmnpim}
		\mathrm{P}^n \bigl( n_{(m)} / n \geq \pi_m + 2 \sqrt{\log n / n} \bigr)
		\leq 1/n^3.
	\end{align}
	On the other hand, let $\zeta_i' := \eins \{ Y_i \neq m \} - (1 - \pi_m)$, then $\zeta_i'$'s are independent random variables such that $\mathbb{E}_{\mathrm{P}}[\zeta_i'] = 0$ and $\mathbb{E}_{\mathrm{P}} \zeta_i'^2 \leq 1/4$ for $1 \leq i \leq n$. Again, by using Bernstein's inequality in Lemma \ref{lem::bernstein}, we obtain that for any $\tau > 0$, there holds
	\begin{align*}
		\mathrm{P}^n \biggl( 1 - \frac{1}{n} \sum_{i=1}^n \eins \{ Y_i = m \} \geq 1 - \pi_m + \sqrt{\frac{\tau}{2n}} + \frac{2\tau}{3n} \biggr)
		\leq 1/n^3.
	\end{align*}
	Setting $\tau := 3 \log n$, we obtain
	\begin{align}\label{equ::pimnmn}
		\mathrm{P}^n \bigl( \pi_m \geq n_{(m)} / n + 2 \sqrt{\log n/n} \bigr)
		\leq 1/n^3.
	\end{align}
	Now, \eqref{equ::nmnpim} together with \eqref{equ::pimnmn} and a union bound argument yields
	\begin{align}\label{equ::pna}
		\mathrm{P}^n \Bigl( 1 - (2/\underline{\pi}) \sqrt{\log n / n} 
		\leq n_{(m)} / (n \pi_m) 
		\leq 1 + (2/\overline{\pi}) \sqrt{\log n/n}, \, \forall 1 \leq m \leq M \Bigr)
		\geq 1 - 2 M / n^3,
	\end{align}
	where $\overline{\pi} = \max_{1 \leq m \leq M} \pi_i$ and $\underline{\pi} = \min_{1 \leq m \leq M} \pi_i$. Let the event $E$ be defined by
	\begin{align*}
		E := \Bigl\{ 1 - (2/\underline{\pi}) \sqrt{\log n / n} 
		\leq n_{(m)} / (n \pi_m) 
		\leq 1 + (2/\overline{\pi}) \sqrt{\log n/n}, \, \forall 1 \leq m \leq M \Bigr\}
	\end{align*}
	and the integer $n_1 \in \mathbb{N}_+$ satisfy $\log n_1/n_1 \leq \min \{ \overline{\pi}^2/4, \underline{\pi}^2/16 \}$. The following arguments will be made on the event $E$ if $n > n_1$. 
	
	\textit{(i)} 
	Since $n > n_1$, the definition of the event $E$ implies that $1/2 < n_{(m)}/(n\pi_m)$ for $1 \leq m \leq M$. Consequently we have $n \pi_m / 2 \leq n_{(m)} \leq n$.
	
	\textit{(ii)} 
	By \eqref{equ::pimu}, we have for $1 \leq m \leq M$,
	\begin{align*}
		|\pi_m^u - 1/M|
		= \biggl| \frac{n \pi_m / n_{(m)}}{\sum_{m=1}^M n \pi_m / n_{(m)}} - 1/M \biggr|
		= \frac{|M n \pi_m / n_{(m)} - \sum_{m=1}^M n \pi_m / n_{(m)}|}{M \sum_{m=1}^M n \pi_m / n_{(m)}}.
	\end{align*}
	On the event $E$, there holds
	\begin{align*}
		|\pi_m^u - 1/M|
		\leq (4/(M \underline{\pi})) \sqrt{\log n/n} 
		\cdot \bigl( 1 + (2/\overline{\pi}) \sqrt{\log n/n} \bigr) 
		\cdot \bigl( 1 - (2/\underline{\pi}) \sqrt{\log n/n} \bigr)^{-1}.
	\end{align*}
	Since $n > n_1$, we have $|\pi_m^u - 1/M| \leq (16/(M \underline{\pi}) \sqrt{\log n/n}$ and consequently
	\begin{align*}
		1/(2 M \overline{\pi})
		\leq \sum_{m=1}^M \eta_m(x) / (2M\pi_m)
		\leq \sum_{m=1}^M \pi_m^u \eta_m(x) / \pi_m
		\leq \sum_{m=1}^M 2 \eta_m(x)/ (M \pi_m)
		\leq 2/(M \underline{\pi}).
	\end{align*}
	This together with $f_X^u(x) = \sum_{m=1}^M \pi_m^u(x) f(x | y = m) = f_X(x) \sum_{m=1}^M \pi_m^u(x) \eta_m(x) / \pi_m$
	yields $(1/(2 M \overline{\pi})) f_X(x) \leq f_X^u(x) \leq (2 / (M \underline{\pi})) f_X(x)$.
	
	\textit{(iii)} 
	For any $1 \leq m \leq M$ and $x, x' \in \mathcal{X}$, by \eqref{equ::alphacontinuous} and \eqref{equ::etaund}, we have
	\begin{align*}
		|\eta^u_m(x')-\eta^u_m(x)|
		\leq c\|x'-x\|^{\alpha}\cdot \frac{n/n_{(m)}}{\sum^M_{m=1}n\eta_m(x)/n_{(m)}}.
	\end{align*}
	On the event $E$, for $n > n_1$, there holds
	\begin{align*}
		|\eta_m^u(x') - \eta^u_m(x)|
		& \leq  c \|x' - x\|^{\alpha} \bigl( 1 - (2 / \underline{\pi}) \sqrt{\log n / n} \bigr)^{-1} \biggl( \sum^M_{m=1} n \eta_m(x) / n_{(m)} \biggr)^{-1}
		\\
		&\leq c \|x' - x\|^{\alpha} \bigl( 1 - (2 / \underline{\pi}) \sqrt{\log n / n} \bigr)^{-1} \bigl( 1 + (2 / \overline{\pi}) \sqrt{\log n/n} \bigr)
		\leq 4 c \|x' - x\|^{\alpha}.
	\end{align*}
	
	Thus, \eqref{equ::pna} implies that \eqref{equ::con1} and \eqref{equ::con3} hold with probability $\mathrm{P}^n$ at least $1-2M/n^3$ if $n > n_1$, which completes the proof.
\end{proof}

The following technical lemma is needed in the proof of Proposition \ref{prop::BnnEstimation}.

\begin{lemma}\label{lem::lemconcen}
	Let $Z_1,\ldots,Z_n$ be a sequence of independent zero-mean real-valued random variables with $|Z_i| \leq C$ for some constant $C>0$. Let $(v_{1},\ldots,v_{n})$ be a weight vector, with $v_{\max} = \max_i |v_i|>0$. Then for all $\varepsilon>0$, we have
	\begin{align*}
		\mathrm{P}\biggl(\sum_{i=1}^nv_iZ_i\geq\varepsilon\biggr)\leq 2\exp\biggl(-\frac{\varepsilon^2}{2C^2v_{\max}\sum_{i=1}^n |v_i|}\biggr).
	\end{align*}
\end{lemma}

\begin{proof}[Proof of Lemma \ref{lem::lemconcen}]
	For $1\leq i\leq n$, we have $|v_iZ_i|\leq Cv_i$. Applying Hoeffding's inequality in Lemma \ref{lem::hoffeding}, we get
	$\mathrm{P}\bigl(\sum_{i=1}^nv_iZ_i\geq \varepsilon\bigr)
	\leq 2\exp\bigl(- 2 \varepsilon^2 / \sum_{i=1}^n (2Cv_i)^2\bigr)$.
	This together with the inequality 
	$\sum_{i=1}^n(2Cv_i)^2
	=4C^2\sum_{i=1}^nv_i^2
	\leq 4C^2 v_{\max}\sum_{i=1}^n |v_i|$
	yields the assertion.
\end{proof}

\begin{proof}[Proof of Proposition \ref{prop::BnnEstimation}]
	By the definition of $\widehat{\eta}^{k,u}$ and $\overline{\eta}^{k,u}$, we have
	\begin{align*}
		\widehat{\eta}^{k,u}_m(x) - \overline{\eta}^{k,u}_{m}(x)
		= \frac{1}{k}\sum_{i=1}^k \bigl( \eins \{ Y^u_{(i)}(x) = m \} - \eta_m^u(X^u_{(i)}(x)) \bigr).
	\end{align*}
	Conditional on $D_b^u$, the random variables $\eins \{ Y^u_{(1)}(x) = m \} - \eta_m^u(X^u_{(1)}(x)), \ldots, \eins \{ Y^u_{(s_u)}(x) = m \} - \eta_m^u(X^u_{(s_u)}(x))$ are independent with zero mean and $|\eins \{ Y^u_{(1)}(x) = m \} - \eta_m^u(X^u_{(1)}(x))|\leq 1$. Applying Lemma \ref{lem::lemconcen}, we get
	$(\mathrm{P}^u_{Y|X})^{s_u}\bigl( |\widehat{\eta}^{k,u}_m(x) - \overline{\eta}^{k,u}_{m}(x)| \geq \varepsilon | D_n^u\bigr)
	\leq 2 \exp (- \varepsilon^2 k / 2)$.
	Setting $\varepsilon := \sqrt{2(2d + 3)\log s_u/k}$, we get
	\begin{align}\label{equ::pointwise}
		(\mathrm{P}^u_{Y|X})^{s_u}
		\bigl( |\widehat{\eta}_m^{k,u}(x) - \overline{\eta}_{m}^{k,u}(x)| \geq \varepsilon | D_n^u \bigr)
		\leq 2 s_u^{-(2d+3)}.
	\end{align}
	Note that this inequality holds only for a fixed $x$. To derive the uniform upper bound over $\mathcal{X}$, let
	$\mathcal{S} := \bigl\{ (\sigma_1, \ldots, \sigma_{s_u}) : \text{all permutations of } (1, \ldots, s_u) \text{ obtainable by moving } x \in \mathbb{R}^d \bigr\}$. Then we have
	\begin{align*}
		&	(\mathrm{P}^u_{Y|X})^{s_u}\biggl( \sup_{x \in \mathbb{R}^d} \bigl( |\widehat{\eta}^{k,u}_m(x) - \overline{\eta}_{m}^{k,u}(x)| - \varepsilon \bigr) > 0 \bigg| D_n^u\biggr)
		\\
		& \leq(\mathrm{P}^u_{Y|X})^{s_u}\biggl( \bigcup_{(\sigma_1, \ldots, \sigma_{s_u}) \in \mathcal{S}} \biggl| \sum_{i=1}^k k^{-1} (\eins \{ Y_{\sigma_i} = m \} - \eta_m^u(X_{\sigma_i})) \biggr| > \varepsilon \bigg| D_n^u\biggr)
		\\
		& \leq \sum_{(\sigma_1, \ldots, \sigma_{s_u}) \in \mathcal{S}} (\mathrm{P}^u_{Y|X})^{s_u}\biggl( \biggl| \sum_{i=1}^k k^{-1} (\eins \{ Y_{\sigma_i} = m \} - \eta^u_m(X_{\sigma_i})) \biggr| > \varepsilon \bigg| D_n^u \biggr).
	\end{align*}
	For any $(\sigma_1,\ldots,\sigma_{s_u})\in \mathcal{S}$, by \eqref{equ::pointwise}, we have
	\begin{align*}
		(\mathrm{P}^u_{Y|X})^{s_u}\biggl( \biggl| \sum_{i=1}^k k^{-1} \bigl( \eins \{ Y_{\sigma_i} = m \} - \eta_m^u(X_{\sigma_i}) \bigr) \biggr| > \varepsilon \bigg| D_n^u\biggr)
		\leq 2 / s_u^{2d+3}.
	\end{align*}
	This together with Lemma \ref{lem::numberreorder} implies
	\begin{align*}
		(\mathrm{P}^u_{Y|X})^{s_u} \Bigl( \sup_{x \in \mathbb{R}^d} (|\widehat{\eta}_m(x) - \overline{\eta}_{m}(x)| - \varepsilon) > 0 \Big| D_n^u \Bigr)
		\leq 2 (25/d)^d / s_u^3.
	\end{align*}
	when $s_u\geq 2d$. Then a union bound argument with $c_1 := 2 (2d + 3) $ yields
	\begin{align*}
		(\mathrm{P}^u_{Y|X})^{s_u} \bigl(
		\|\widehat{\eta}^{k,u} - \overline{\eta}^{k,u}\|_{\infty}
		\leq \sqrt{c_1 \log s_u / k} \big| D_n^u \bigr)
		\geq 1 - 2 M (25/d)^d / s_u^3.
	\end{align*}
	By \eqref{equ::con1} in Lemma \ref{lem::punder}, if $n \geq n_1$, then
	we have $s_u\geq n_{(1)}\geq n\underline{\pi}/2$ with probability $\mathrm{P}^n$ at least $1-2M/n^3$. 
	Consequently, if $n > \max \{ n_1, \lceil 4d / \underline{\pi} \rceil\}$, there holds $\|\widehat{\eta}^{k,u}(x)-\overline{\eta}^{k,u}(x)\|_{\infty} \leq \sqrt{c_1 \log s_u / k}$
	with probability $\mathrm{P}^n\otimes \mathrm{P}_Z$ at least $1-(2M+16M(25/d)^d/\underline{\pi}^3)/n^3$. 
	Therefore, if $n\geq N_1:=\max\{n_1,\lceil 4d/\underline{\pi}\rceil,\lceil 2M+16M(25/d)^d/\underline{\pi}\rceil\}$, then we have $\|\widehat{\eta}^{k,u}(x)-\overline{\eta}^{k,u}(x)\|_{\infty}
	\leq \sqrt{c_1 \log s_u / k}$ with probability $\mathrm{P}^n\otimes \mathrm{P}_Z$ at least $1-1/n^2$, which completes the proof.
\end{proof}

\subsubsection{Proofs Related to Section \ref{sec::approxunder}} \label{sec::proofsecappro}

To conduct our analysis, we first need to recall the definitions of \textit{VC dimension} (\textit{VC index}) and \textit{covering number}, which are frequently used in capacity-involved arguments and measure the complexity of the underlying function class \cite{vandervaart1996weak,Kosorok2008introduction,gine2021mathematical}.

\begin{definition}[VC dimension] \label{def::VC dimension}
	Let $\mathcal{B}$ be a class of subsets of $\mathcal{X}$ and $A \subset \mathcal{X}$ be a finite set. The trace of $\mathcal{B}$ on $A$ is defined by $\{ B \cap A : B \subset \mathcal{B}\}$. Its cardinality is denoted by $\Delta^{\mathcal{B}}(A)$. We say that $\mathcal{B}$ shatters $A$ if $\Delta^{\mathcal{B}}(A) = 2^{\#(A)}$, that is, if for every $A' \subset A$, there exists a $B \subset \mathcal{B}$ such that $A' = B \cap A$. For $n \in \mathrm{N}$, let
	\begin{align}\label{equ::VC dimension}
		m^{\mathcal{B}}(n) := \sup_{A \subset \mathcal{X}, \, \#(A) = n} \Delta^{\mathcal{B}}(A).
	\end{align}
	Then, the set $\mathcal{B}$ is a Vapnik-Chervonenkis class if there exists $n<\infty$ such that $m^{\mathcal{B}}(n) < 2^n$ and the minimal of such $n$ is called the VC dimension of $\mathcal{B}$, and abbreviate as $\mathrm{VC}(\mathcal{B})$.
\end{definition} 

Since an arbitrary set of $n$ points $\{x_1,\ldots,x_n\}$ possess $2^n$ subsets, we say that $\mathcal{B}$ \textit{picks out} a certain subset from $\{ x_1, \ldots, x_n\}$ if this can be formed as a set of the form $B\cap \{x_1,\ldots,x_n\}$ for a $B\in \mathcal{B}$. The collection $\mathcal{B}$ \textit{shatters} $\{x_1,\ldots,x_n\}$ if each of its $2^n$ subsets can be picked out in this manner. 
From Definition \ref{def::VC dimension} we see that the VC dimension of the class $\mathcal{B}$ is the smallest $n$ for which no set of size $n$ is shattered by $\mathcal{B}$,  that is,
\begin{align*}
	\mathrm{VC}(\mathcal{B}) =\inf \Bigl\{n:\max_{x_1,\ldots,x_n} \Delta^{\mathcal{B}}(\{ x_1,\ldots,x_n \})\leq 2^n\Bigr\},
\end{align*}
where $\Delta^{\mathcal{B}}(\{ x_1, \ldots,x_n \})=\#\{B\cap \{x_1,\ldots,x_n\}:B\in \mathcal{B}\}$.
Clearly, the more refined $\mathcal{B}$ is, the larger is its index.

\begin{definition}[Covering Number]
	Let $(\mathcal{X}, d)$ be a metric space and $A \subset \mathcal{X}$. For $\varepsilon>0$, the $\varepsilon$-covering number of $A$ is denoted as 
	\begin{align*}
		\mathcal{N}(A, d, \varepsilon) 
		:= \min \biggl\{ n \geq 1 : \exists x_1, \ldots, x_n \in \mathcal{X} \text{ such that } A \subset \bigcup^n_{i=1} B(x_i, \varepsilon) \biggr\},
	\end{align*}
	where $B(x, \varepsilon) := \{ x' \in \mathcal{X} : d(x, x') \leq \varepsilon \}$.
\end{definition}

The following Lemma, which is needed in the proof of Lemma \ref{lem::Rrho}, provides the covering number of the indicator functions on the collection of the balls in $\mathbb{R}^d$.

\begin{lemma}\label{lem::CoveringNumber}
	Let $\mathcal{B} := \{ B(x, r) : x \in \mathbb{R}^d, r > 0 \}$ and $\eins_{\mathcal{B}} := \{ \eins_B : B \in \mathcal{B} \}$. Then for any $\varepsilon \in (0, 1)$, there exists a universal constant $C$ such that 
	\begin{align*}
		\mathcal{N}(\eins_{\mathcal{B}}, \|\cdot\|_{L_1(\mathrm{Q})}, \varepsilon)
		\leq C (d+2) (4e)^{d+2} \varepsilon^{-(d+1)}
	\end{align*}
	holds for any probability measure $\mathrm{Q}$.
\end{lemma}

\begin{proof}[Proof of Lemma \ref{lem::CoveringNumber}] 	
	For the collection $\mathcal{B} := \{ B(x, r) : x \in \mathbb{R}^d, r > 0 \}$,
	\cite{dudley1979balls} shows that for any set $A \in \mathbb{R}^d$ of $d + 2$ points, not all subsets of $A$ can be formed as a set of the form $B \cap A$ for a $B \in \mathcal{B}$. 
	In other words, $\mathcal{B}$ can not pick out all subsets from $A \in \mathbb{R}^d$ of $d + 2$ points.
	Therefore, the collection $\mathcal{B}$ fails to shatter $A$. 
	Consequently, according to Definition \ref{def::VC dimension}, we have $\mathrm{VC}(\mathcal{B}) = d + 2$. Using Theorem 9.2 in \cite{Kosorok2008introduction}, we immediately obtain the assertion.
\end{proof}

To prove Proposition \ref{prop::knnundersample}, we need the following lemma, which provides a high probability uniform bound on the distance between any point and its $k$-th nearest neighbor.

\begin{lemma}\label{lem::Rrho}
	Let $R_{(k)}(x) := \|X_{(k)}(x) - x\|$ denote the distance from $x$ to its $k$-th nearest neighbor, $1 \leq k \leq n$. Moreover, let $f_X$ be the density function of $\mathrm{P}_X$ and suppose that there exist constants $\underline{c}$,  $\overline{c} > 0$ such that $\underline{c} \leq f_X(x) \leq \overline{c}$. Then there exists an $n_2\in \mathbb{N}$ and $c_0=2/\underline{c}>0$ such that for all $n>n_2$, there holds
	\begin{align*}
		\sup_{x\in \mathcal{X}}\sup_{k\geq 48(2d+9)\log n} R_{(k)}(x)^{\alpha}\leq c_0 (k/n)^{\alpha/d}
	\end{align*}
	with probability $\mathrm{P}^n$ at least $1-1/n^3$.
\end{lemma}

\begin{proof}[Proof of Lemma \ref{lem::Rrho}]
	For $x\in \mathcal{X}$ and $\eta\in [0,1]$, we define the $\eta$-quantile diameter 
	\begin{align*}
		\rho_x(\eta) := \inf \bigl\{ r : \mathrm{P}(B(x, r)) \geq \eta \bigr\}.
	\end{align*}
	Let us first consider the set $\mathcal{B}_k^+ :=  \big\{ B \bigl( x, \rho_x \bigl( (k + \sqrt{3 \tau k})/n \bigr) \bigr) : x \in \mathcal{X} \bigr\} \subset \mathcal{B}$. Lemma \ref{lem::CoveringNumber} implies that for any probability $\mathrm{Q}$, there holds
	\begin{align} \label{Bk+CoveringNumber}
		\mathcal{N}(\eins_{\mathcal{B}_k^+}, \|\cdot\|_{L_1(\mathrm{Q})}, \varepsilon)
		\leq \mathcal{N}(\eins_{\mathcal{B}}, \|\cdot\|_{L_1(\mathrm{Q})}, \varepsilon)
		\leq C (d+2) (4e)^{d+2} \varepsilon^{-(d+1)}.
	\end{align}
	By the definition of the covering number, there exists an $\varepsilon$-net $\{A_j\}_{j=1}^J \subset \mathcal{B}_k^+$ with $J := \lfloor C (d+2) (4e)^{d+2} \varepsilon^{-(d+1)} \rfloor$ and for any $x \in \mathcal{X}$, there exists some $j \in \{ 1, \ldots, J \}$ such that 
	\begin{align}\label{eq::approxAj}
		\bigl\| \eins \bigl\{ B \bigl( x, \rho_x \bigl( (k + \sqrt{3 \tau k})/n \bigr) \bigr) \bigr\} - \eins_{A_j} \bigr\|_{L_1(D)} 
		\leq \varepsilon.
	\end{align}
	For any $i = 1, \ldots, n$, let the random variables $\xi_i$ be defined by $\xi_i = \eins_{A_j}(X_i) - (k + \sqrt{3 \tau k})/n$. Then we have $\|\xi_i\|_{\infty} \leq 1$, $\mathbb{E}_{\mathrm{P}}\xi_i = 0$ and $\mathbb{E}_{\mathrm{P}}\xi_i^2 \leq \mathbb{E}_{\mathrm{P}}\xi_i = (k + \sqrt{3 \tau k})/n$. Applying Bernstein's inequality in Lemma \ref{lem::bernstein}, we obtain 
	\begin{align*}
		\frac{1}{n} \sum_{i=1}^n \eins_{A_j}(X_i) - (k + \sqrt{3 \tau k})/n
		\geq - \sqrt{2 \tau (k + \sqrt{3 \tau k})} / n - 2 \tau / (3n)
	\end{align*}
	with probability $\mathrm{P}^n$ at least $1-e^{-\tau}$. Then the union bound together with the covering number estimate \eqref{Bk+CoveringNumber} implies that for any $A_j$, $j = 1, \cdots, J$, there holds 
	\begin{align*}
		& \frac{1}{n} \sum_{i=1}^n \eins_{A_j}(X_i) - (k + \sqrt{3 (\tau + \log J) k}) / n
		\\
		& \geq - \sqrt{2 (\tau + \log J) \bigl( k + \sqrt{3 (\tau + \log J) k} \bigr)} / n 
		- 2 (\tau + \log J) / (3 n)
	\end{align*}
	with probability $\mathrm{P}^n$ at least $1-e^{-\tau}$. This together with \eqref{eq::approxAj} yields that for any $x \in \mathcal{X}$, there holds
	\begin{align*}
		& \frac{1}{n} \sum_{i=1}^n 
		\eins \bigl\{ B \bigl( x, \rho_x \bigl( (k + \sqrt{3 \tau k})/n \bigr) \bigr) \bigr\} (X_i) 
		- (k + \sqrt{3 (\tau + \log J) k}) / n
		\\
		& \geq - \sqrt{2 (\tau + \log J) \bigl( k + \sqrt{3 (\tau + \log J) k} \bigr)} / n 
		- 2 (\tau + \log J) / (3 n) - \varepsilon
	\end{align*}
	with probability $\mathrm{P}^n$ at least $1-e^{-\tau}$.

	Now, if we take $\varepsilon = 1/n$, then for any $n > n_2 := \max \{ 4 e, d + 2, C \}$, there holds
	\begin{align*}
		\log J =  \log C + \log(d+2) +  (d+2)\log(4e) + (d+1) \log n \leq (2d+5)\log n. 
	\end{align*}
	Let $\tau:=\log (n^4)$. A simple calculation yields that if $k \geq 48(2d+9)\log n$, then we have
	\begin{align*}
		\sqrt{2 (\tau + \log J) \bigl( k + \sqrt{3 (\tau + \log J) k} \bigr)} / n \leq \sqrt{5(\tau+\log J)k/2}/n.
	\end{align*}
	Consequently, for all $n > n_2$, there holds
	\begin{align*}
		\sqrt{2 (\tau + \log J) \bigl( k + \sqrt{3 (\tau + \log J) k} \bigr)} / n 
		+ 2 (\tau + \log J) / (3 n)
		+ 1/n
		\leq \sqrt{3(\tau+\log J) k} / n.
	\end{align*}
	Consequently for all $x \in \mathcal{X}$, there holds 
	$\frac{1}{n} \sum_{i=1}^n 
	\eins \bigl\{ B \bigl( x, \rho_x \bigl( (k + \sqrt{3 \tau k})/n \bigr) \bigr) \bigr\} (X_i) 
	\geq k/n$
	with probability $\mathrm{P}^n$ at least $1-1/n^4$. 
	By the definition of $R_{(k)}(x)$, there holds 
	\begin{align}\label{equ::Rrho1}
		R_{(k)}(x) \leq \rho_x \bigl( (k + \sqrt{3 \tau k})/n \bigr)
	\end{align}
	with probability $\mathrm{P}^n$ at least $1-1/n^4$. 
	For any $x \in \mathcal{X}$, we have 
	$\mathrm{P}_X \bigl( B \bigl( x, \rho_x \bigl( (k + \sqrt{3 \tau k})/n \bigr) \bigr) \bigr) 
	= (k + \sqrt{3 \tau k})/n$. 
	Since the density $f_X(x)$ satisfies $\underline{c} \leq f_X(x) \leq \overline{c}$, we have
	$(k + \sqrt{3 \tau k})/n
	\geq \underline{c} \rho_x^d \bigl( (k + \sqrt{3 \tau k})/n \bigr)$
	and  consequently
	\begin{align}\label{eq::cond}
		\rho_x \bigl( (k + \sqrt{3 \tau k})/n \bigr) 
		\leq \bigl( (k + \sqrt{3 \tau k})/(\underline{c} n) \bigr)^{1/d}\leq \bigl( (k + 3 \sqrt{k \log n})/ (\underline{c} n) \bigr)^{1/d}.
	\end{align}
	Combining \eqref{equ::Rrho1} with \eqref{eq::cond}, we obtain that $R_{(k)}(x) \leq \bigl( (k + 3 \sqrt{k \log n})/ (\underline{c} n) \bigr)^{1/d}$ holds for all $x \in \mathcal{X}$ with probability $\mathrm{P}^n$ at least $1-1/n^4$. Therefore, a union bound argument yields that for all $x \in \mathcal{X}$, there holds
	\begin{align*}
		\sup_{k\geq 48(2d+9)\log n} R_{(k)}(x) 
		\leq \bigl( (k + 3 \sqrt{k \log n})/ (\underline{c} n) \bigr)^{1/d} 
		\leq \bigl( 2 k / (\underline{c} n) \bigr)^{1/d}
	\end{align*}
	with probability $\mathrm{P}^n$ at least $1-1/n^3$ for all $n\geq n_2$, which yields to the assertion with $c_0=(2/\underline{c})^{1/d}$.
\end{proof}

\begin{proof}[Proof of Proposition \ref{prop::knnundersample}]
	Let $n_1 \in \mathbb{N}_+$ satisfy $\log n_1 / n_1 \leq \min \{ \overline{\pi}^2/4, \underline{\pi}^2/16\}$. By Lemma \ref{lem::punder}, we see that for $1 \leq m \leq M$, if $n > n_1$, there hold
	\begin{align}\label{equ::npim2}
		n \pi_m / 2 \leq n_{(m)}, 
		\qquad
		1 / (2 M \overline{\pi}) \leq f_X^u(x)  \leq 2 / (M \underline{\pi}),
	\end{align}
	and 
	\begin{align}\label{equ::etamux}
		|\eta_m^u(x') - \eta_m^u(x)|
		\leq 4 c_L \|x - x'\|^{\alpha},
	\end{align}
	with probability $\mathrm{P}^n$ at least $1-2M/n^3$, since \eqref{equ::npim2} implies that $\mathrm{P}^u$ satisfies the assumptions of Lemma \ref{lem::Rrho}, we find that for all $x \in \mathcal{X}$ and $s_u\geq n_2$, there holds
	\begin{align}\label{equ::k482d9}
		\sup_{k\geq 48(2d+9)\log s_u} \|X_{(k)}^u(x)-x\|^{\alpha} 
		\leq (4 M \overline{\pi} k / s_u)^{\alpha/d}
	\end{align}
	with probability $ \mathrm{P}^n\otimes\mathrm{P}_{Z}$ at least $1-1/s_u^3$. By using $s_u\geq n_{(1)}\geq n\pi_1/2$ in \eqref{equ::npim2}, we find that if $n\geq N_2:=\max\{n_1,\lceil 2n_2/\underline{\pi}\rceil, \lceil 8/\underline{\pi}^3+2M\rceil \}$, then  \eqref{equ::k482d9} holds with probability $\mathrm{P}^n\otimes \mathrm{P}_Z$ at least $1-1/s_u^3-2M/n^3\geq 1-(8/\underline{\pi}^3+2M)/n^3\geq 1-1/n^2$.
	
	Let $R_{(i)}^u(x) := \|X_{(i)}^u(x) - x\|$ and $a_n^u := \lceil 48(2d+9)\log s_u\rceil$. 
	Then \eqref{equ::etamux} implies that for any $1 \leq m \leq M$, 
	\begin{align*}
		\sum_{i=1}^k k^{-1} & \bigl| \eta_m^u(X_{(i)}^u(x)) - \eta_m^u(x) \bigr|
		\leq  (4/k) \sum_{i=1}^{a_n^u} c_L (R_{(i)}^u(x))^{\alpha} + (4/k) \sum_{i=a_n^u +1}^k c_L (R_{(i)}^u(x))^{\alpha} 
		\\
		& \leq (4/k) \sum_{i=1}^{a_n^u} c_L (R_{(a_n^u+1)}^u(x))^{\alpha}
		+ (4/k) \sum_{i=a_n^u +1}^k c_L (R_{(i)}^u(x))^{\alpha} 
		\\
		& \leq (4 c_L a_n^u/ k) \bigl( 4 M \overline{\pi} (a_n^u + 1) / s_u \bigr)^{\alpha/d}
		+ \sum_{i=1}^k (4 c_L / k) (4 M \overline{\pi} i / s_u)^{\alpha/d}
		\\
		& \leq 4 c_L (8 M \overline{\pi} k/s_u)^{\alpha/d}
		+ \sum_{i=1}^k (4 c_L / k) (4 M \overline{\pi} i / s_u)^{\alpha/d}.
	\end{align*}
	Since the function $g(t):=t^{\alpha/d}$ is increasing in $(0,\infty)$, we have
	\begin{align*}
		\sum_{i=1}^{k} k^{-1} (i/s_u)^{\alpha/d}
		& \leq (s_u/k) \int_0^{(k+1)/s_u} g(t) \, dt
		\\
		& \leq (s_u/k) (d/(\alpha+d)) \bigl( (k+1) / s_u \bigr)^{(\alpha+d)/d}
		\leq 2^{(\alpha+d)/d} (d/(\alpha+d)) (k/s_u)^{\alpha/d},
	\end{align*}
	which yields $\|\overline{\eta}^{k,u}-\eta^u\|_{\infty} \leq c_2 (k/s_u)^{\alpha/d}$ with constant $c_2 := 4 c_L (8M\overline{\pi})^{\alpha/d}+4c_L(4M\overline{\pi})^{\alpha/d}\cdot 2^{(\alpha+d)/d}d/(\alpha+d)$. Thus we finish the proof.
\end{proof}

\subsubsection{Proofs Related to Section \ref{sec::undersamplingerror}} \label{sec::proofundersampling}

\begin{proof}[Proof of Proposition \ref{thm::punderund}]
	Let the event $E$ be defined by
	\begin{align*}
		E := \Bigl\{ 1 - (2 / \underline{\pi}) \sqrt{\log n/n} \leq n_{(m)} / (n\pi_m) \leq 1 + (2/\overline{\pi}) \sqrt{\log n/n}, \, \forall 1 \leq m \leq M \Bigr\}
	\end{align*}
	and the number $n_1 \in \mathbb{N}_+$ satisfy $\log n_1 / n_1 \leq \min \{ \overline{\pi}^2/4, \underline{\pi}^2/16 \}$. The following arguments will be made on the event $E$ and for $n > n_1$. An elementary calculation implies that for any $1 \leq m \leq M$ and $x \in \mathcal{X}$, there holds
	\begin{align*}
		\bigl| n \eta_m(x) / n_{(m)} - \eta_m(x) / \pi_m \bigr|
		& \leq \bigl| n / n_{(m)} - 1 / \pi_m \bigr|
		= \bigl| (n \pi_m - n_{(m)})/(n_i \pi_m) \bigr|
		\\
		& = |1 - n_{(m)}/(n\pi_m)| / \bigl( \pi_m(n_{(m)} / (n\pi_m) - 1 ) + \pi_m \bigr).	
	\end{align*}
	Thus, on the event $E$, there holds
	\begin{align*}
		\bigl| n \eta_m(x) / n_{(m)} - \eta_m(x) / \pi_m \bigr|
		\leq (2 / \overline{\pi}) \sqrt{\log n / n} \big/ \bigl( \pi_m - (2 / \underline{\pi}) \sqrt{\log n / n} \pi_m \bigr).
	\end{align*}
	Therefore, for $n > n_1$, we have 
	\begin{align}\label{equ::netaini}
		\bigl| n \eta_m(x) / n_{(m)} - \eta_m(x) / \pi_m \bigr|
		\leq (4 / \pi_m \overline{\pi}) \sqrt{\log n / n}
		\leq (4 / (\overline{\pi} \underline{\pi})) \sqrt{\log n / n}.
	\end{align}
	Consequently, for any $1 \leq m \leq M$ and $x \in \mathcal{X}$, there holds
	\begin{align}\label{equ::wideparenetam}
		\bigl| \eta^w_m(x) - \eta_m^u(x) \bigr|
		& \leq \biggl| \frac{n \eta_m(x) / n_{(m)}}{\sum_{i=1}^M n \eta_i(x) / n_i} - \frac{\eta_m(x) / \pi_m}{\sum_{i=1}^M \eta_i(x) / \pi_i} \biggr|
		\nonumber\\
		& \leq \biggl| \frac{n \eta_m(x) / n_{(m)}}{\sum_{i=1}^M n \eta_i(x) / n_i} - \frac{n \eta_m(x) / n_{(m)}}{\sum_{i=1}^M \eta_i(x) / \pi_i} \biggr|
		+ \biggl| \frac{n \eta_m(x) / n_{(m)}}{\sum_{i=1}^M \eta_i(x) / \pi_i} - \frac{\eta_m(x) / \pi_m}{\sum_{i=1}^M \eta_i(x) / \pi_i} \biggr|
		\nonumber\\
		& \leq \frac{|n \eta_m(x) / n_{(m)} - \eta_m(x) / \pi_m|}{\sum_{i=1}^M \eta_i(x) / \pi_i} 
		+ \frac{\sum_{m=1}^M |n \eta_m(x) / n_{(m)} - \eta_m(x) / \pi_m|}{\sum_{i=1}^M \eta_i(x) / \pi_i}.
	\end{align}
	Obviously, we have $\sum_{i=1}^M \eta_i(x) / \pi_i \geq \sum_{i=1}^M \eta_i(x) / \overline{\pi} = 1 / \overline{\pi}$.
	This together with \eqref{equ::netaini} and \eqref{equ::wideparenetam} yields that $\|\eta^w_m - \eta_m^u\|_{\infty} \leq (4 (1 + M) / \underline{\pi}) \sqrt{\log n / n}$ holds for $1 \leq m \leq M$, which together with \eqref{equ::pna} implies that for all $n > n_1$, \eqref{equ::con2} holds with probability $\mathrm{P}^n$ at least $1-2M/n^3$. Therefore, for all $n\geq N_3:=\max\{n_1,2M\}$, there holds \eqref{equ::con2} with probability $\mathrm{P}^n$ at least $1-1/n^2$. Thus we complete the proof of Proposition \ref{thm::punderund}.
\end{proof}

\subsubsection{Proofs Related to Section \ref{sec::knnmulticlassim}}\label{sec::proofknn}

\begin{proof}[Proof of Theorem \ref{thm::knnundersample}]
	Propositions \ref{prop::BnnEstimation} and  \ref{prop::knnundersample} imply that if $n\geq \max\{N_1,N_2\}$, 
	there hold
	$\|\overline{\eta}^{k,u} - \eta^u(x) \|_{\infty} \lesssim (k/s_u)^{\alpha/d}$ and 
	$\|\widehat{\eta}^{k,u}(x)-\overline{\eta}^{k,u}(x)\|_{\infty} \lesssim \sqrt{\log s_u / k}$
	with probability $\mathrm{P}_Z \otimes \mathrm{P^n}$ at least $1-2/n^2$.
	Consequently, we have
	\begin{align*}
		\| \widehat{\eta}_m^{k,u}(x) - \eta_m^u(x) \|_{\infty} 
		& \leq \|\widehat{\eta}_m^{k,u}(x) - \overline{\eta}_m^{k,u}(x) \|_{\infty} 
		+ \|\overline{\eta}_m^{k,u}(x) - \eta_m^u(x) \|_{\infty}
		\\
		& \lesssim (k/s_u)^{\alpha/d}+\sqrt{{\log s_u/k}}
		\lesssim (\log s_u/s_u)^{\alpha/(2\alpha+d)}.
	\end{align*}	
	According to \eqref{equ::con1} in Lemma \ref{lem::punder}, we have $s_u\geq n_{(1)}\geq n\pi_1/2$ with probability $\mathrm{P}^n$ at least $1-2M/n^3\geq 1-1/n^2$. Since $g(x):=\log(x)/x$ is decreasing on $[e,\infty)$,  we have
	\begin{align*}
		\| \widehat{\eta}_m^{k,u}(x) - \eta_m^u(x) \|_{\infty} \lesssim (\log (n\pi_1/2)/(n\pi_1/2))^{\alpha/(2\alpha+d)}\lesssim (\log n/n)^{\alpha/(2\alpha+d)}.
	\end{align*}
	This together with Proposition \ref{thm::punderund} yield that if $n\geq N_1^*:=\max\{N_1,N_2,N_3\}$,
	there holds
	\begin{align*}
		\|\widehat{\eta}^{k,u} -  \eta^w\|_{\infty}
		& = \|\widehat{\eta}^{k,u} - \eta^u\|_{\infty} + \|\eta^u - \eta^w\|_{\infty}
		\\
		& \lesssim (\log n/n)^{\alpha/(2\alpha+d)} + \sqrt{\log n/n}
		\lesssim (\log n/n)^{\alpha/(2\alpha+d)}.
	\end{align*}
	with probability $\mathrm{P}^n$ at least $1-4/n^2$.
	Consequently, Lemma \ref{lem::convergereg} yields that 
	\begin{align*}
		\|\eta^{w,*}_{L_{\mathrm{cl},\mathrm{P}^w}}(x)- \eta^w_{\widehat{f}^{k,u}(x)}(x) \|_{\infty}
		\lesssim (\log n/n)^{\alpha/(2\alpha+d)}
	\end{align*}
	holds with probability $\mathrm{P}_Z \otimes \mathrm{P^n}$ at least $1-4/n^2$, where  $\eta^{w,*}_{L_{\mathrm{cl}},\mathrm{P}^w}(x) = \eta^w_{{f}_{L_{\mathrm{cl}},\mathrm{P}^w}^*(x)}(x)$, i.e,~the Bayes classifier w.r.t.~the classification loss $L_{\mathrm{cl}}$ and the balanced distribution $\mathrm{P}^w$. Using Lemma \ref{lem::regexcessrisk}, we obtain
	\begin{align*}	
		\mathcal{R}_{L_{\mathrm{cl}},\mathrm{P}^w}(\widehat{f}^{k,u}) - \mathcal{R}_{L_{\mathrm{cl}},\mathrm{P}^w}^*
		\lesssim  (\log n/n)^{{\alpha(\beta+1)}/{(2\alpha+d)}}
	\end{align*}
	with probability $ \mathrm{P}_Z\otimes \mathrm{P^n}$ at least $1-4/n^2$. This together with \eqref{equ::R_amf} and Theorem \ref{thm::connection} implies $\mathfrak{R}_{\mathrm{AM}}(\widehat{f}^{k,u}) \lesssim (\log n/n)^{{\alpha(\beta+1)}/{(2\alpha+d)}}$, which completes the proof.
\end{proof}

\begin{proof}[Proof of Theorem \ref{thm::lowerbound}]
	We use Theorem 3.5 in \cite{audibert2007fast} to prove Theorem \ref{thm::lowerbound}. First of all, we verify as follows that for any probability distribution $\mathrm{P}\in \mathcal{P}$, the balanced version of it, $\mathrm{P}^w$ belongs to a certain class of probability distributions ${\mathcal{P}}_{\Sigma}$ in Definition 3.1 in \cite{audibert2007fast}.
	
	\textit{(i)}
	Condition \textit{(i)} in Assumption \ref{ass:imbalance} implies that $\mathrm{P}^w$ satisfies the margin assumption.
	
	\textit{(ii)}
	For any $1 \leq m \leq M$ and $x, x' \in \mathcal{X}$, by \eqref{equ::etamstarx}, we have
	\begin{align*}
		|\eta^w_m(x')-\eta^w_m(x)|=\frac{|\eta_m(x')-\eta_m(x)|}{\pi_m \sum^M_{m=1}\eta_m(x)/\pi_m}
		\leq (\overline{\pi} / \underline{\pi}) \cdot |\eta_m(x')-\eta_m(x)|.
	\end{align*}
	Condition $(ii)$ in Assumption \ref{ass:imbalance} then yields
	$|\eta^w_m(x')-\eta^w_m(x)|
	\leq 4 c_L (\overline{\pi} / \underline{\pi}) \cdot \|x' - x\|^{\alpha}$.
	Therefore, the posterior probability function $\eta^w$ belongs to the H\"{o}lder class.
	
	\textit{(iii)}
	Since $\mathrm{P}_X$ is the uniform distribution on $[0,1]^d$, \eqref{equ::tildefx} yields 
	\begin{align*}
		1 / (M \overline{\pi})
		\leq f^w_X(x)
		= f_X(x) \sum_{m=1}^M \eta_m(x) / (M \pi_m)
		\leq 1 / (M \underline{\pi})
	\end{align*}
	for $x\in [0,1]^d$ and $f^w_X(x)=0$ otherwise, which implies that the strong density assumption on $\mathrm{P}^w_X$ is satisfied. 
	
	Therefore, applying Theorem 3.5 in \cite{audibert2007fast}, there exists a constant $C > 0$ such that for any $f_n \in \mathcal{F}$, there holds
	\begin{align*}
		\sup_{\mathrm{P}^w\in \mathcal{P}_{\Sigma}}
		\mathcal{R}_{L_{\mathrm{cl}}, \mathrm{P}^w}({f_n}) - \mathcal{R}_{L_{\mathrm{cl}}, \mathrm{P}^w}^*\geq C n^{-\alpha(\beta+1)/(2\alpha+d)}.
	\end{align*}
	This together with \eqref{equ::R_amf} and Theorem \ref{thm::connection} implies
	\begin{align*}
		\sup_{\mathrm{P}\in \mathcal{P}}r^*_{\mathrm{AM}} - r_{\mathrm{AM}} (f_n)\geq C n^{-\alpha(\beta+1)/(2\alpha+d)}.
	\end{align*}
	Thus the assertion is proved since $f_n \in \mathcal{F}$ is an arbitrary classifier.
\end{proof}

\subsection{Proofs Related to the Under-bagging $k$-NN  Classifier} \label{sec::proofunderbagging}

In this section, we first present in Sections \ref{sec::proofbagging}-\ref{sec::proofbagsample} the proofs of the theoretical results on bounding the bagging error in Section \ref{sec::bagsample}, the bagged approximation error in Section \ref{sec::bagapprox}, and the bagged sample error in Section \ref{sec::bagsampleerror}, respectively. Then in Section \ref{sec::proofknn}, we prove the main results on the convergence rates of the under-bagging $k$-NN classifier, i.e., Theorem \ref{thm::bagknnclassund} and Corollary \ref{thm::bag1nnund} in Section \ref{sec::knnmulticlassimbag}.

\subsubsection{Proofs Related to Section \ref{sec::bagsample}} \label{sec::proofbagging}

\begin{proof}[Proof of Proposition \ref{pro::fml*funder}]
	By the definition of $\widehat{\eta}^{B,u}(x)$ and $\widetilde{\eta}^{B,u}(x)$, we have
	\begin{align*}
		|\widehat{\eta}^{B,u}(x)-\widetilde{\eta}^{B,u}(x)|
		= \biggl| \frac{1}{B} \sum_{b=1}^B \sum_{i=1}^n V_i^{b,u}(x) \eins \{ Y_{(i)}(x) = m \}
		- \sum_{i=1}^n \overline{V}_i^u(x) \eins \{ Y_{(i)}(x) = m \} \biggr|.
	\end{align*}
	Let $\xi_b:=\sum_{i=1}^n (V_i^{b,u}(x)-\overline{V}_i^u(x)) \eins \{ Y_{(i)}(x) = m \}$. Then have $\|\xi_b\|_{\infty} \leq 1/k$ and 
	\begin{align*}
		\mathbb{E}(\xi_b^2|D_n)
		\leq 2 \sum_{i=1}^n \mathbb{E}((V_i^{b,u}(x)-\overline{V}_i^u(x))^2|D_n)
		= \sum_{i=1}^n \frac{\overline{V}_i^u(x)(1-k\overline{V}_i^u(x))}{k}
		\leq \sum_{i=1}^n \frac{\overline{V}_i^u(x)}{k}
		\leq \frac{1}{k}.
	\end{align*}
	Applying Bernstein's inequality in Lemma \ref{lem::bernstein}, we obtain that for every $\tau > 0$, 
	\begin{align*}
		\mathrm{P}_Z^B\biggl(	|\widehat{\eta}^{B,u}(x)-\widetilde{\eta}^{B,u}(x)|\geq \sqrt{\frac{2\tau}{kB}}+ \frac{2\tau}{3Bk} \bigg|D_n\biggr)\leq e^{-\tau}.
	\end{align*}
	Then with $\tau := (2d + 3) \log n$ we have
	\begin{align*}
		\mathrm{P}_Z^B\biggl(|\widehat{\eta}^{B,u}(x)-\widetilde{\eta}^{B,u}(x)|
		\leq \sqrt{\frac{2(2d+3)\log n}{kB}}+\frac{2(2d+3)\log n}{3kB}\bigg| D_n\biggr)\geq 1 - 1 / n^{2d+3}.
	\end{align*}
	Setting $\varepsilon := \sqrt{8(2d+3)\log n/(kB)}$, then the condition $9kB\geq 2(2d+3)\log n$ implies that
	\begin{align}\label{equ::pointwise2}
		\mathrm{P}_Z^B(|\widehat{\eta}^{B,u}(x)-\widetilde{\eta}^{B,u}(x)|
		\leq \varepsilon| D_n)
		\geq 1 - 1 / n^{2d+3}.
	\end{align}
	In order to derive the uniform upper bound over $\mathcal{X}$, let
	\begin{align*}
		\mathcal{S} := \bigl\{ (\sigma_1, \ldots, \sigma_{n}) : \text{all permutations of } (1, \ldots, n)\text{ obtainable by moving } x \in \mathbb{R}^d \bigr\}.
	\end{align*}
	Then we have
	\begin{align*}
		&	\mathrm{P}_Z^B \biggl( \sup_{x \in \mathbb{R}^d} \biggl( |\widehat{\eta}^{B,u}_m(x) - \widetilde{\eta}_{m}^{B,u}(x)| - \varepsilon\biggr) > 0 \bigg| D_n\biggr)
		\\
		& \leq \mathrm{P}_Z^B \biggl( \bigcup_{(\sigma_1, \ldots, \sigma_n) \in \mathcal{S}}\biggl| \frac{1}{B} \sum_{b=1}^B \sum_{i=1}^n V_{i,\sigma}^{b,u}(x) \eins \{ Y_{\sigma_i}(x) = m \} - \sum_{i=1}^n \overline{V}_{i,\sigma}^u(x) \eins \{ Y_{\sigma_i}(x) = m \} \biggr| > \varepsilon \bigg| D_n \biggr)
		\\
		& \leq \sum_{(\sigma_1, \ldots, \sigma_n) \in \mathcal{S}} \mathrm{P}_Z^B \biggl(\biggl| \frac{1}{B} \sum_{b=1}^B \sum_{i=1}^n V_{i,\sigma}^{b,u}(x) \eins \{ Y_{\sigma_i}(x) = m \} - \sum_{i=1}^n \overline{V}_{i,\sigma}^u(x) \eins \{ Y_{\sigma_i}(x) = m \} \biggr| > \varepsilon \bigg| D_n \biggr),
	\end{align*}
	where 
	$V_{i,\sigma}^{b,u}(x)$ equals $1/k$ if $\sum^i_{j=1} Z^b(X_{\sigma_j}(x),Y_{\sigma_j}(x))\leq k$, and $0$ otherwise, and $\overline{V}_{i,\sigma}^u(x) = k^{-1} \mathrm{P}_Z( \sum^i_{j=1} Z^b(X_{\sigma_j}(x),Y_{\sigma_j}(x))\leq k|\{X_i,Y_i\}_{i=1}^n)$. For any $(\sigma_1,\ldots,\sigma_n)\in \mathcal{S}$, \eqref{equ::pointwise2} implies
	\begin{align*}
		\mathrm{P}_Z^B \biggl(\biggl| \frac{1}{B} \sum_{b=1}^B \sum_{i=1}^n V_{i,\sigma}^{b,u}(x) \eins \{ Y_{\sigma_i}(x) = m \} - \sum_{i=1}^n \overline{V}_{i,\sigma}^u(x) \eins \{ Y_{\sigma_i}(x) = m \} \biggr| > \varepsilon \bigg| D_n \biggr)
		\leq 1 / n^{2d+3}.
	\end{align*}
	This together with Lemma \ref{lem::numberreorder} yields
	\begin{align*}
		\mathrm{P}_Z^B\biggl( \sup_{x \in \mathbb{R}^d} (|\widehat{\eta}^{B,u}_m(x) - \widetilde{\eta}^{B,u}_{m}(x)| - \varepsilon) > 0 \bigg| D_n\bigg)
		\leq (25/d)^d / n^3
	\end{align*}
	for all $n\geq 2d$. Then a union bound argument implies 
	\begin{align*}
		\mathrm{P}_Z^B\bigl(
		\|\widehat{\eta}^{B,u}-\widetilde{\eta}^{B,u}\|_{\infty}
		\leq \sqrt{8(2d+3) \log n / (k B)} \big| D_n \bigr)
		\geq 1 - M (25/d)^d / n^3.
	\end{align*}
	Consequently, if $n\geq N_4:=\max\{\lceil M(25/d)^d\rceil,2d\}$, then we have
	\begin{align*}
		\mathrm{P}_Z^B\otimes \mathrm{P}^n \bigl(	\|\widetilde{\eta}^{B,u}-\overline{\eta}^{B,u}\|_{\infty} \leq \sqrt{8(2d+3) \log n / (k B )} \bigr)
		\geq 1 - M (25/d)^d n^3\geq 1-1/n^2,
	\end{align*}
	which completes the proof.
\end{proof}

\subsubsection{Proofs Related to Section \ref{sec::bagapprox}} \label{sec::proofbagapprox}

To bound the bagged approximation error $\|\overline{\eta}^{B,u}-\eta^u\|_{\infty}$, the theoretical property of $\overline{V}_i^u(x)$ plays a crucial role. In fact, by the definition of $\overline{\eta}^{B,u}$ and $\eta^u$, we have
\begin{align*}
	\|\overline{\eta}^{B,u} - \eta^u\|_{\infty}
	& = \biggl\| \sum_{i=1}^n \overline{V}^u_i(x) \eta_m^u(X_{(i)}(x)) - \eta^u(x) \biggr\|_{\infty}
	\\
	& \leq \biggl\| \sum_{i=1}^n \overline{V}^u_i(x) \bigl( \eta_m^u(X_{(i)}(x)) - \eta^u(x) \bigr) \bigg\|_{\infty} + \biggl\| \sum_{i=1}^n \overline{V}^u_i(x) - 1 \biggr\|_{\infty}.
\end{align*}
Let $\xi_i:=Z^b(X_{(i)}(x),Y_{(i)}(x)), i \geq 1 $. According to \eqref{equ::defbarviu}, $\overline{V}_i^u(x)$ can be re-expressed as
\begin{align}\label{equ::barvdef}
	\overline{V}^{u}_i(x)= \frac{1}{k}  \mathrm{P}_Z \biggl( \sum_{l=1}^i \xi_l
	\leq k,\xi_i=1\bigg|\{ (X_i,Y_i) \}_{i=1}^n \biggr)=\frac{1}{k} \sum_{j=1}^k p^u_{i,j}(x).
\end{align}
where
\begin{align}\label{equ::puijx}
	p^u_{i,j}(x)
	= \mathrm{P}_Z \biggl( \sum^i_{l=1} \xi_l= j, \ \xi_i=1 \bigg| \{ (X_i,Y_i) \}_{i=1}^n \biggr).
\end{align}
As a result, it is necessary to investigate the properties of $p_{i,j}^u(x)$ for all $x\in \mathcal{X}$. Note that $\{ \xi_i, i \geq 1 \}$ can be regarded as a sequence of independent Bernoulli trials with the probabilities of success $a(X_{(i)}(x),Y_{(i)}(x))$, that is,
\begin{align}
	\mathrm{P}(\xi_i = 1) 
	& = a(X_{(i)}(x),Y_{(i)}(x)),
	\label{equ::success}
	\\
	\mathrm{P}(\xi_i = 0) 
	& = 1 - a(X_{(i)}(x),Y_{(i)}(x)).
	\label{equ::failure}
\end{align}
Then by \eqref{equ::puijx}, $p_{i,j}^u(x)$ represents the probability when we observing the sequence $\{\xi_i,1\leq i\leq n\}$ until $j$ successes have occurred, the total number of trials equals to $i$.

Recall that the classical Pascal distribution models the number of successes in a sequence of i.i.d.~Bernoulli trials before a specified number of failures occurs. We refer the readers to \cite{spiegel1992theory,degroot2012probability} for more details. However, \eqref{equ::success} implies that the probabilities of success for the Bernoulli trials $\{ \xi_i, i \geq 1 \}$ are not the same. Therefore, it is necessary to consider a \textit{Generalized Pascal} (GP) distribution where the probabilities of success depend on the location of $x$ in $\mathbb{R}^d$. To this end, we introduce some basic notations. Suppose that $\{ \xi_i, i \geq 1 \}$ is a sequence of independent Bernoulli trials. In each trial, we have the probability of success $\mathrm{P}(\xi_i = 1) = p_i$ and the probability of failure $\mathrm{P}(\xi_i = 0) = 1 - p_i$. The total number of trials we have seen until $j$ successes have occurred, namely $X$, are said to have the Generalized Pascal distribution with parameters $j$ and $p = \{ p_i \}_{i=1}^{\infty}$, that is, $X \sim \mathrm{GP}(j, p)$. For $i \geq j$, let 
$\Omega(j, i)
:= \{ \omega = \{ \omega_1, \ldots, \omega_j \} : 1 \leq \omega_1 < \omega_2 < \cdots < \omega_{j-1} < \omega_j = i\}$.
Then the probability mass function of the Generalized Pascal distribution is 
\begin{align}\label{equ::massgp}
	\mathrm{P}_{\mathrm{GP}}(X=i)=f_{\mathrm{GP}}(i; j, p)
	= \sum_{\omega \in \Omega(j, i)} p_{\ell} \prod_{i=1}^{i-1} p_i^{\eins \{ i \in \omega \}} (1 - p_i)^{\eins \{ i \notin \omega \}}, 
	\quad 
	i \geq j, \,  
	i \in \mathbb{N}_+.
\end{align}
In fact, according to \eqref{equ::puijx}, $p_{i,j}^u(x)$ can be re-expressed as
\begin{align}\label{equ::fgplrp}
	p_{i,j}^u(x)=f_{\mathrm{GP}}(i,j,p(x)),
\end{align}
where $p(x)=(p_1(x),\ldots,p_n(x),\ldots)$ with elements defined by
$p_i(x)=a(X_{(i)}(x),Y_{(i)}(x))$, $1 \leq i \leq n$.
According to the acceptance probability \eqref{equ::axymn} in Section \ref{equ::underbaggingknn}, it is easy to see that for all $x\in \mathcal{X}$, we have
\begin{align}\label{equ::defstildep}
	\begin{split}
		p(x)\in S_{\nu} := \biggl\{ p = \{ p_i \}_{i=1}^{\infty} : 
		& \,  (p_1, \ldots, p_n ) = (\nu_{\sigma_1}, \ldots, \nu_{\sigma_n}) \text{ and } p_i = \nu_n \text{ for } i > n,
		\\
		& \text{ where } \{\sigma_1,\ldots,\sigma_n\} \text{ is a permutation of } \{1,\ldots,n\} \biggr\}.	
	\end{split}
\end{align}
where $\nu = (\nu_1, \ldots, \nu_n, \ldots)$ is an infinite-dimensional vector with elements defined by 
\begin{align}\label{equ::defnu}
	\nu_i=
	\begin{cases}
		\frac{s}{Mn_{(M)}}, & \text{ if } 1 \leq i \leq n_{(M)},
		\\
		\frac{s}{Mn_{(m)}}, & \text{ if } \sum_{\ell=m+1}^M n_{\ell} < i \leq \sum_{\ell=m}^M n_\ell \text{ and } 1 \leq m \leq M - 1,
		\\
		\frac{s}{Mn_{(1)}}, & \text{ if } i > n. 
	\end{cases}
\end{align}
Therefore, in our analysis on $p_{i,j}^u(x)$, it suffices to study the property of the distribution function $f_{\mathrm{GP}}(i,j,p)$ with $p$ restricted on the set $S_{\nu}$.

In the following, we present some results on the Generalized Pascal distribution, which is later crucial for the proof of Lemmas \ref{pro::detererror1} and \ref{lem::sumviund}. The first lemma gives a uniform upper bound of the tail probability of the Generalized Pascal distribution.

\begin{lemma}\label{thm::Tailgp}
	Let $p\in S_\nu$ and suppose that $\sum^{\ell}_{i=1}p_i\geq j$. Then we have
	\begin{align*}
		\sum^{\infty}_{i=\ell+1} f_{\mathrm{GP}}(i; j, p)
		\leq \exp \biggl( - \frac{1}{2\ell} \biggl( \sum_{i=1}^{\ell} p_i - j \biggr)^2 \biggr).
	\end{align*}
\end{lemma}

\begin{proof}[Proof of Lemma \ref{thm::Tailgp}]
	Let $\{\xi_i$, $i \geq 1\}$ be a sequence of independent Bernoulli trials such that $\mathrm{P}(\xi_i = 1) = p_i$ and $\mathrm{P}(\xi_i = 0) = 1 - p_i$. By the definition of the Generalized Pascal distribution,  we have
	$\sum^{\infty}_{i=\ell+1} 	f_{\mathrm{GP}}(i; j, p)
	= \mathrm{P} \bigl( \sum^\ell_{i=1} \xi_i < j \bigr)$.
	Let $\xi_i' = 1 - \xi_i$ for $1 \leq i \leq \ell$. Then we have 
	\begin{align}\label{equ::pgpx}
		\sum^{\infty}_{i=\ell+1} 	f_{\mathrm{GP}}(i; j, p)
		= \mathrm{P} \biggl( \sum^\ell_{i=1} \xi'_i > \ell - j \biggr)
		= \mathrm{P} \biggl( \frac{1}{\ell} \sum^\ell_{i=1} \xi'_i > 1 - \frac{j}{\ell} \biggr).
	\end{align}
	Moreover, there hold $\mathrm{P}(\xi_i' = 1) = 1 - p_i$ and $\mathrm{P}(\xi_i' = 0) = p_i$. Consequently we have $\mathbb{E}[\xi_i'] = 1 - p_i$ and $\mathrm{Var}[\xi_i'] \leq 1/4$. Using Bernstein's inequality in Lemma \ref{lem::bernstein}, for any $\tau > 0$, there holds
	\begin{align}\label{equ::sumxii}
		\frac{1}{\ell} \sum^\ell_{i=1} \xi'_i 
		\geq \sqrt{\frac{\tau}{2\ell}} + \frac{2\tau}{3\ell} + 1 - \frac{1}{\ell}\sum^\ell_{i=1} p_i
	\end{align}
	with probability at most $e^{-\tau}$. Let $\tau := \frac{1}{2\ell} ( \sum_{i=1}^{\ell} p_i - j )^2$. Then we have $\tau \leq (\ell - j)^2 / (2 \ell) \leq \ell / 2$ and consequently \eqref{equ::sumxii} implies
	\begin{align*}
		\frac{1}{\ell}\sum^\ell_{i=1}\xi'_i
		\geq \sqrt{\frac{2\tau}{\ell}}+1-\frac{1}{\ell}\sum^\ell_{i=1}p_i
		\geq \frac{1}{\ell}\biggl(\sum^\ell_{i=1}p_i-j\biggr)+1-\frac{1}{\ell}\sum^\ell_{i=1}p_i
		= 1-\frac{j}{\ell}.
	\end{align*}
	Therefore, we have
	\begin{align}\label{equ::pxiigeq}
		\mathrm{P}\biggl(\frac{1}{\ell}\sum^\ell_{i=1}\xi'_i>1-\frac{j}{\ell}\biggr)\leq e^{-\tau}.
	\end{align}
	Combining \eqref{equ::pgpx} with \eqref{equ::pxiigeq}, we obtain the assertion.
\end{proof}

The next lemma provides the effect of changing the position of entries of $p \in S_{\nu}$ on the Generalized Pascal distribution.

\begin{lemma}\label{lem::xxqq1}
	Given $1 \leq q \leq n - 1$, let $p \in S_{\nu}$ such that $p_q \geq p_{q+1}$. We define a map $h_q : S_{\nu} \to S_{\nu}$ such that the elements of $p' = h_q(p)$  satisfying
	\begin{align}\label{equ::pmcases}
		p_i' :=
		\begin{cases}
			p_i, & \text{ if } i < q \text{ or } i > q + 1,
			\\
			p_{q+1}, & \text{ if } i = q, 
			\\
			p_q, & \text{ if } i=q+1.
		\end{cases}
	\end{align}
	Then for any $i\geq j$, the following statements hold:
	\begin{enumerate}
		\item[(i)] 
		If $i < q$ or $i > q + 1$, then we have
		\begin{align}\label{equ::pgpcase1}
			f_{\mathrm{GP}}(i; j, p)
			=	f_{\mathrm{GP}}(i; j, p');
		\end{align}
		\item[(ii)] 
		If $i = q$, then we have 
		\begin{align}\label{equ::pgpcase2}
			f_{\mathrm{GP}}(i; j, p)
			\geq	f_{\mathrm{GP}}(i; j, p');
		\end{align}
		\item[(iii)] 
		If $i = q + 1$, then we have
		\begin{align}\label{equ::pgpcase5}
			f_{\mathrm{GP}}(i; j, p)
			\leq	f_{\mathrm{GP}}(i; j, p').
		\end{align}
	\end{enumerate}
\end{lemma}

\begin{proof}[Proof of Lemma \ref{lem::xxqq1}]
	\textit{(i)} 
	If $i < q$, by the definition of $p'$ in \eqref{equ::pmcases}, we have $p_i' = p_i$ for $i \leq i$. Thus, by \eqref{equ::massgp}, there holds
	\begin{align*}
		f_{\mathrm{GP}}(i; j, p)
		& = \sum_{\omega \in \Omega(j, i)} p_i \prod_{\ell=1}^{i-1} p_i^{\eins \{ \ell \in \omega \}} (1 - p_i)^{\eins \{ \ell \notin \omega \}}
		\\
		& = \sum_{\omega \in \Omega(j, i)} p_i' \prod_{\ell=1}^{i-1} p_i'^{\eins \{ \ell \in \omega \}} (1 - p_i')^{\eins \{ \ell \notin \omega \}}
		= \mathrm{P}_{\mathrm{GP}}(X' = i).
	\end{align*}

	If $i > q + 1$, then we define the map $g: \Omega(j,i)\to  \Omega(j,i)$ by
	\begin{align*}
		g(\omega) :=
		\begin{cases}
			\omega, 
			& \text{ if } \{ q, q + 1 \} \subset \omega \text{ or }  \{q,q+1\} \nsubseteq \omega,
			\\
			\omega \setminus \{ q \} \cup \{ q + 1 \}, 
			& \text{ if } q \in \omega \text{ and } q + 1 \notin \omega,
			\\
			\omega \setminus \{ q + 1 \} \cup \{ q \}, 
			& \text{ if } q + 1 \in \omega \text{ and } q \notin \omega.
		\end{cases}
	\end{align*}
	By the definition of $p'$ in \eqref{equ::pmcases}, for every $\omega \in \Omega(j, i)$, there holds
	\begin{align*}
		p_{i} \prod_{\ell=1}^{i-1} p_i^{\eins \{ \ell \in \omega \}} (1 - p_i)^{\eins \{ \ell \notin \omega \}}
		= p_{i}' \prod_{\ell=1}^{i-1} {p_i'}^{\eins \{ \ell \in g(\omega) \}} (1 - p_i')^{\eins \{ \ell \notin g(\omega) \}}.
	\end{align*}
	Taking the sum over all possible elements in $\Omega(j, i)$, we obtain
	\begin{align*}
		f_{\mathrm{GP}}(i; j, p)
		= \sum_{\omega \in \Omega(j, i)} p_{i} \prod_{\ell=1}^{i-1} p_i^{\eins \{ \ell \in \omega \}} (1 - p_i)^{\eins \{ \ell \notin \omega \}}
		= \sum_{\omega \in \Omega(j, i)} p_{i}' \prod_{\ell=1}^{i-1} {p_i'}^{\eins \{ \ell \in g(\omega) \}} (1 - p_i')^{\eins \{ \ell \notin g(\omega) \}}.
	\end{align*}
	It can be verified that $g : \Omega(j,i) \to \Omega(j,i)$ is a one-to-one map. Therefore, we have
	\begin{align*}
		f_{\mathrm{GP}}(i; j, p)
		= \sum_{\omega \in \Omega(j, i)} p_{i}' \prod_{\ell=1}^{i-1} {p_i'}^{\eins \{ \ell \in \omega \}} (1 - p_i')^{\eins \{ \ell \notin \omega \}}=
		f_{\mathrm{GP}}(i; j, p').
	\end{align*}

	\textit{(ii)} 
	If $i = q$, then by \eqref{equ::pmcases}, we have $p_q' = p_{q+1} \leq p_q$ and $p_i' = p_i$ for $\ell < q$. Consequently we obtain
	\begin{align*}
		f_{\mathrm{GP}}(q; j, p)
		& = \sum_{\omega \in \Omega(j, q)} p_q \prod_{\ell=1}^{q-1} p_i^{\eins \{ \ell \in \omega \}} (1 - p_i)^{\eins \{ \ell \notin \omega \}}
		\\
		& \geq \sum_{\omega \in \Omega(j, q)} p_q' \prod_{\ell=1}^{q-1} {p_i'}^{\eins \{ \ell \in \omega \}} (1 - p_i')^{\eins \{ \ell \notin \omega \}}
		= f_{\mathrm{GP}}(q; j, p').
	\end{align*}

	\textit{(iii)} 
	If $i=q+1$,  we consider two specific cases: $i = j$ and $i > j$. Suppose that $i = q + 1 = j$. Obviously, there holds
	\begin{align}\label{equ::suminfty}
		1
		= \sum_{i=j}^{\infty} 	f_{\mathrm{GP}}(i; j, p)
		= \sum_{i=j}^{\infty} 	f_{\mathrm{GP}}(i; j, p').
	\end{align}
	Then \eqref{equ::suminfty} together with \eqref{equ::pgpcase1} yields 
	$f_{\mathrm{GP}}(q+1; j, p) = f_{\mathrm{GP}}(q+1; j, p')$.
	If $i = q + 1 > j$, then we have $q \geq j$.
	Combining \eqref{equ::pgpcase1}, \eqref{equ::pgpcase2}, and \eqref{equ::suminfty}, we obtain
	$f_{\mathrm{GP}}(q+1; j, p) \geq f_{\mathrm{GP}}(q+1; j, p')$, which
	completes the proof.
\end{proof}

\begin{lemma}\label{lem::expecationincrease}
	Given $1 \leq q \leq n - 1$, let $p \in S_{\nu}$ such that $p_q \geq p_{q+1}$. Moreover, let $p' = h_q(p)$ be as in \eqref{equ::pmcases}. Then we have
	\begin{align*}
		\sum_{i=j}^n i f_{\mathrm{GP}}(i; j, p)
		\leq \sum_{i=j}^n i f_{\mathrm{GP}}(i; j, p').
	\end{align*}
\end{lemma}

\begin{proof}[Proof of Lemma \ref{lem::expecationincrease}]
	The proof can be divided into the following three cases. 
	
	If $j \leq q$, then Lemma \ref{lem::xxqq1} together with \eqref{equ::suminfty} yields
	\begin{align}\label{equ::pgpxq}
		f_{\mathrm{GP}}(q; j, p)+ f_{\mathrm{GP}}(q+1; j, p)
		= f_{\mathrm{GP}}(q; j, p')+ f_{\mathrm{GP}}(q+1; j, p').
	\end{align}
	By \eqref{equ::pgpcase1}, we have
	\begin{align*}
		& \sum_{i=j}^n i f_{\mathrm{GP}}(i; j, p)- \sum_{i=j}^n i f_{\mathrm{GP}}(i; j, p')
		\\
		& = (q + 1) (f_{\mathrm{GP}}(q+1; j, p) - f_{\mathrm{GP}}(q+1; j, p'))   
		+ q (f_{\mathrm{GP}}(q; j, p)- f_{\mathrm{GP}}(q; j, p'))).
	\end{align*}
	This together with \eqref{equ::pgpxq} implies
	\begin{align*}
		\sum_{i=j}^n i f_{\mathrm{GP}}(i; j, p)- \sum_{i=j}^n i  f_{\mathrm{GP}}(i; j, p')
		=  f_{\mathrm{GP}}(q; j, p')-  f_{\mathrm{GP}}(q; j, p)
		\leq 0,
	\end{align*}
	where the last inequality follows from \eqref{equ::pgpcase2}. 
	
	If $j = q + 1$, \eqref{equ::pgpcase1} together with \eqref{equ::pgpcase5} yields
	\begin{align*}
		\sum_{i=j}^n i  f_{\mathrm{GP}}(i; j, p)- \sum_{i=j}^n i  f_{\mathrm{GP}}(i; j, p')
		= (q + 1) ( f_{\mathrm{GP}}(q+1; j, p) -  f_{\mathrm{GP}}(q+1; j, p'))
		\leq 0.
	\end{align*}
	
	Otherwise if $j > q + 1$, then \eqref{equ::pgpcase1} yields
	$\sum_{i=j}^n i  f_{\mathrm{GP}}(i; j, p) = \sum_{i=j}^n i  f_{\mathrm{GP}}(i; j, p')$, which
	completes the proof.
\end{proof}

The next theorem provides the upper bound of the expectation of the truncated Generalized Pascal distribution, which is needed to prove Lemma \ref{pro::detererror1}.

\begin{theorem}\label{thm::genepasexp}
	Let $p \in S_{\nu}$. Then for any $j \leq u \leq n$ satisfying $\sum_{\ell=1}^{{u}} \nu_{\ell} > j$, there holds
	\begin{align*}
		\sum_{i=j}^n i  f_{\mathrm{GP}}(i; j, p)
		\leq \frac{j}{\nu_1} \biggl( \frac{\nu_u}{\nu_1} \biggr)^j 
		+ n \exp \biggl( - \frac{1}{2u} \biggl( \sum_{\ell=1}^{{u}} \nu_{\ell} - j \biggr)^2 \biggr).
	\end{align*}
\end{theorem}

\begin{proof}[Proof of Theorem \ref{thm::genepasexp}]
	For any $p \in S_{\nu}$, we can exchange the positions of two adjacent entries of $(p_1, \ldots, p_n)$ successively to arrange the entry with larger value behind. It is easy to see that a finite number of such operations can change $p$ to $\nu$, that is, $(p_1, \ldots, p_n)$ is rearranged in ascending order. Therefore, there exists a series of probability sequences $\{ p^t, 1 \leq t \leq T \}$ such that $p^{(1)} = p$, $p^{(T)} = \nu$, and $p^{(t+1)} = h_{q_t}(p^{(t)})$ for $2 \leq t \leq T - 1$, that is, we exchange $p_{q_t}^{(t)}$ and $p_{q_{t}+1}^{(t)}$ in the $t$-th operation. By Lemma \ref{lem::expecationincrease}, we have
	\begin{align*}
		\sum_{i=j}^n i  f_{\mathrm{GP}}(i; j, p)
		\leq \sum_{i=j}^n i  f_{\mathrm{GP}}(i; j, \nu).
	\end{align*}
	\eqref{equ::massgp} implies that for $j \leq i \leq u$, there holds
	\begin{align*}
		\sum_{i=j}^n i  f_{\mathrm{GP}}(i; j, \nu)
		& \leq \sum_{\omega \in \Omega(j, i)} \nu_{i} \prod_{\ell=1}^{i-1} \nu_i^{\eins \{ \ell \in \omega \}} (1 - \nu_i)^{\eins \{ \ell \notin \omega \}}
		\leq \sum_{\omega \in \Omega(j, i)} \nu_u \prod_{\ell=1}^{i-1} \nu_u^{\eins \{ \ell \in \omega \}} (1 - \nu_1)^{\eins \{ \ell \notin \omega \}}
		\\
		& \leq \biggl( \frac{\nu_u}{\nu_1} \biggr)^j \sum_{\omega \in \Omega(j, i)} \nu_1 \prod_{\ell=1}^{i-1} \nu_1^{\eins \{ \ell \in \omega \}} (1 - \nu_1)^{\eins \{ \ell \notin \omega \}}
		= \binom{i-1}{j-1} \biggl( \frac{\nu_u}{\nu_1} \biggr)^j \nu_1^j (1 - \nu_1)^{i-j}
	\end{align*}
	and consequently we have
	\begin{align}\label{equ::pgpxl1}
		\sum_{i=j}^u i f_{\mathrm{GP}}(i; j, \nu)
		& \leq \biggl( \frac{\nu_u}{\nu_1} \biggr)^j \sum_{i=j}^u i \binom{i-1}{j-1} \nu_1^j (1 - \nu_1)^{i-j}
		\nonumber\\
		& \leq \biggl( \frac{\nu_u}{\nu_1} \biggr)^j \sum_{i=j}^{\infty} i \binom{i-1}{j-1} \nu_1^j (1 - \nu_1)^{i-j}
		= \frac{j}{\nu_1} \biggl( \frac{\nu_u}{\nu_1} \biggr)^j.
	\end{align}
	Under the assumption $\sum_{\ell=1}^{{u}} \nu_\ell > j$, Theorem \ref{thm::Tailgp} yields
	\begin{align*}
		\sum^{\infty}_{i=u+1} f_{\mathrm{GP}}(i; j, \nu)
		< \exp \biggl( - \frac{1}{2u} \biggl( \sum_{\ell=1}^{{u}} \nu_\ell - j \biggr)^2 \biggr).
	\end{align*}
	Consequently we have
	\begin{align}\label{equ::pgpxl2}
		\sum_{i=u+1}^n i f_{\mathrm{GP}}(i; j, \nu)
		\leq n \sum_{i=u+1}^n f_{\mathrm{GP}}(i; j, \nu)
		\leq n \sum^{\infty}_{i=u+1} f_{\mathrm{GP}}(i; j, \nu)
		\leq n \exp \biggl( - \frac{1}{2u} \biggl( \sum_{\ell=1}^{{u}} \nu_\ell - j \biggr)^2 \biggr).
	\end{align}
	Combining \eqref{equ::pgpxl1} and \eqref{equ::pgpxl2}, we obtain
	\begin{align*}
		\sum_{i=j}^n i f_{\mathrm{GP}}(i; j, p )
		\leq \frac{j}{\nu_1} \biggl( \frac{\nu_u}{\nu_1} \biggr)^j
		+ n \exp \biggl( - \frac{1}{2u} \biggl( \sum_{\ell=1}^{{u}} \nu_\ell - j \biggr)^2 \biggr),
	\end{align*}
	which finishes the proof.
\end{proof}

To prove Proposition \ref{pro::tildefmlund}, we need the following two lemmas. Lemma \ref{pro::detererror1} provides the uniform upper bound of the weighted sum of the $i$-th nearest neighbor distance $R_{(i)}(x)$, which supplies the key to the proof of the bagged approximation error term. Lemma \ref{lem::sumviund} bounds the sum of the bagged weights $\overline{V}_i^u(x)$ uniformly.

\begin{lemma}\label{pro::detererror1}
	Let $\overline{V}_i^u(x)$ be defined as in \eqref{equ::defbarviu} and $R_{(i)}(x) := \|X_{(i)}(x) - x\|$. 
	Suppose that $s\exp( - (s/M-k)^2/(2n)) \leq M\underline{\pi}/2$. Then there exists a constant $c_3 > 0$ and  an $n_3\in \mathbb{N}$ such that for all $n\geq n_3$, with probability $\mathrm{P}^n$ at least $1-(2M+1)/n^3$, for all $x\in \mathcal{X}$, there holds
	\begin{align*}
		\sum_{i=1}^n \overline{V}_i^u(x)R_{(i)}^{\alpha}(x)
		\leq  c_3 (k/s)^{\alpha/d}.
	\end{align*}
\end{lemma}

\begin{proof}[Proof of Lemma \ref{pro::detererror1}]
	Let $a_n:= \lceil 48(2d+9)\log n \rceil$.
	Lemma \ref{lem::Rrho} implies that if $n>n_2$, then for all $x\in \mathcal{X}$, there holds
	$\sup_{i \geq a_n} R_{(i)}(x) \leq (2i/n)^{1/d}$ with probability $\mathrm{P}^n$ at least $1-1/n^3$. 
	Then we have
	\begin{align}
		\sum_{i=1}^n \overline{V}^u_i(x) R_{(i)}^{\alpha}(x) 
		& = \sum_{i=1}^{a_n} \overline{V}_i^u(x) R_{(i)}^{\alpha}(x) + \sum_{i=a_n}^n \overline{V}_i^u(x) R_{(i)}^{\alpha}(x)
		\nonumber\\
		& \leq R_{(a_n)}^{\alpha}(x) + \sum_{i=a_n}^n \overline{V}_i^u(x) R_{(i)}^{\alpha}(x)
		\leq (2a_n/n)^{\alpha/d} + \sum_{i=a_n}^n \overline{V}_i^u(x) R_{(i)}^{\alpha}(x)
		\nonumber\\
		& \leq (2k/n)^{\alpha/d} + \sum_{i=a_n}^n \overline{V}_i^u(x) R_{(i)}^{\alpha}(x)
		\leq (2k/s)^{\alpha/d} + \sum_{i=a_n}^n \overline{V}^u_i(x) R_{(i)}^{\alpha}(x).
		\label{equ::viriund}
	\end{align}
	For the second term in \eqref{equ::viriund}, there holds
	\begin{align}\label{equ::BnnApproximationund}
		\sum_{i = a_n}^n \overline{V}_i^u(x) R_{(i)}^{\alpha}(x)
		& = \sum_{i=a_n}^n k^{-1} \sum_{j=1}^k p_{i,j}^u(x) R_{(i)}^{\alpha}(x)
		\nonumber\\
		& \leq \sum_{i=1}^n k^{-1} \sum_{j=1}^k p_{i,j}^u(x) (2i/n)^{\alpha/d}
		= k^{-1} (2/n)^{\alpha/d} \sum_{j=1}^k \sum_{i=1}^n i^{\alpha/d} p^u_{i,j}(x).
	\end{align}
	By Jensen's inequality, we have
	\begin{align}\label{sumni1ialphad}
		\sum_{i=1}^n i^{\alpha/d} p^u_{i,j}(x)
		\leq \biggl( \sum_{i=1}^n p^u_{i,j}(x) \biggr)^{1-\alpha/d} \biggl( \sum_{i=1}^n i p^u_{i,j}(x) \biggr)^{\alpha/d}\leq \biggl( \sum_{i=1}^n i p^u_{i,j}(x) \biggr)^{\alpha/d}.
	\end{align}
	Since for all $x\in \mathcal{X}$, we have $p_{i,j}^u(x)=f_{\mathrm{GP}}(i,j,p(x))$ by \eqref{equ::fgplrp}, where $p(x)\in S_{\nu}$ with $S_{\nu}$ defined by \eqref{equ::defnu}. Applying Theorem \ref{thm::genepasexp} with $u = n_{(M)}$, we have
	\begin{align*}
		\sum_{i=1}^n i p_{i,j}^u(x)
		= \sum_{i=1}^n i f_{\mathrm{GP}}(i;j,p(x))
		\leq M j n_{(M)} / s + n \exp \bigl( - (s/M - j)^2 / (2n_{(M)}) \bigr)
	\end{align*}
	for all $x\in \mathcal{X}$. 
	According to \eqref{equ::con1} in Lemma \ref{lem::punder} for all $n>n_1$, with probability $\mathrm{P}^n$ at least $1-2M/n^3$, there holds $n_{(M)}\geq n\pi_m/2$. 
	This together with the condition $s\exp( - (s/M-k)^2/(2n)) \leq M\underline{\pi}/2$ and $n_{(M)}\geq n\pi_M/2$ yields that for all $n\geq n_3:=\max\{n_1,n_2\}$, there holds $\sum_{i=1}^n i p_{i,j}^u(x) \leq 2 M j n_{(M)} / s$ with probability $\mathrm{P}^n$ at least $1-(2M+1)/n^3$. Combining this with \eqref{sumni1ialphad}, we obtain $\sum_{i=1}^n i^{\alpha/d} p_{i,j}^u(x) \leq (2 M j n_{(M)}/s)^{\alpha/d}$, which together with \eqref{equ::BnnApproximationund} implies
	\begin{align*}
		\sum_{i = n_1}^n \overline{V}_i^u(x) R_{(i)}^{\alpha}(x)
		\leq k^{-1} (4 M / n)^{\alpha/d}
		\sum_{j=1}^k (j n_{(M)}/s)^{\alpha/d}.
	\end{align*}
	Since $g(t) = t^{\alpha/d}$ is increasing in $[0, 1]$, we have
	$k^{-1} \sum_{j=1}^k (j/k)^{\alpha/d}
	\leq 2 \int_0^1 x^{\alpha/d} \, dx  
	= 2 (1 + \alpha/d)$.
	Consequently we obtain
	\begin{align*}
		\sum_{i=n_1}^n \overline{V}^u_i(x) R_{(i)}^{\alpha}(x)
		= k^{-1} \bigl( 4 M n_{(M)} k /(n s) \bigr)^{\alpha/d}  \sum_{j=1}^k (j/k)^{\alpha/d}
		\leq 2 (1 + \alpha/d) (4 M k / s)^{\alpha/d}.
	\end{align*}
	Combining this with \eqref{equ::viriund}, we find that for all $x\in \mathbb{R}^d$, there holds
	\begin{align*}
		\sum_{i=1}^n \overline{V}_i^u(x) R_{(i)}^{\alpha}(x)
		\leq (2 k/s)^{\alpha/d} + 2 (1 + \alpha/d) (4 M k/s)^{\alpha/d}
		\leq c_3 (k/s)^{\alpha/d},
	\end{align*}
	where the constant $c_3 := 2^{\alpha/d}+2(1+\alpha/d)(4M)^{\alpha/d}$. Thus, we finish the proof.
\end{proof}

The following lemma is needed in the proof of Proposition \ref{pro::tildefmlund}.

\begin{lemma}\label{lem::sumviund}
	Let $\overline{V}^u_i(x)$ be defined by \eqref{equ::defbarviu} and suppose $k \leq s$. 
	Then for all $x\in \mathcal{X}$, we have
	\begin{align*}
		1 - \sum_{i=1}^n \overline{V}^u_i(x)
		\leq \exp \bigl( - (s - k)^2 / (2n) \bigr).
	\end{align*}
\end{lemma}

\begin{proof}[Proof of Lemma \ref{lem::sumviund}]
	By \eqref{equ::defbarviu}, we have
	\begin{align} \label{equ::barVpuij}
		\sum_{i=1}^n \overline{V}^u_i(x)
		= \sum_{i=1}^n \frac{1}{k} \sum_{j=1}^k p^u_{i,j}(x)
		= \frac{1}{k} \sum_{j=1}^k \sum_{i=1}^n p^u_{i,j}(x).
	\end{align}
	By \eqref{equ::fgplrp}, we have $\sum_{i=1}^n p^u_{i,j}(x) = \sum_{i=1}^n f_{\mathrm{GP}}(i;j,p(x))$ where $p(x)\in S_{\nu}$ with $S_{\nu}$ defined as in \eqref{equ::defnu}. Since $\sum_{i=1}^n p_i(x) = \sum_{m=1}^M n_{(m)} \cdot \frac{s}{Mn_{(M)}} = s$, Theorem \ref{thm::Tailgp} implies 
	\begin{align*}
		\sum_{i=1}^n p_{i,j}^u(x)
		= \sum_{i=1}^n f_{\mathrm{GP}}(i;j,p(x))
		\geq 1 - \exp \bigl( - (s - j)^2 / (2n) \bigr)
		\geq 1 - \exp \bigl( - (s - k)^2/ (2n) \bigr).
	\end{align*}
	Combining this with \eqref{equ::barVpuij}, we obtain
	$\sum_{i=1}^n \overline{V}_i^u(x) \geq 1 - \exp \bigl( - (s - k)^2 / (2n) \bigr)$,
	which yields the assertion.
\end{proof}

\begin{proof}[Proof of Proposition \ref{pro::tildefmlund}]
	By the definition of $\overline{\eta}^{B,u}$ and $\eta^u$, we have
	\begin{align*}
		\|\overline{\eta}^{B,u} - \eta^u\|_{\infty}
		& = \biggl\| \sum_{i=1}^n \overline{V}^u_i(x) \eta_m^u(X_{(i)}(x)) - \eta^u(x) \biggr\|_{\infty}
		\\
		& \leq \biggl\| \sum_{i=1}^n \overline{V}^u_i(x) \bigl( \eta_m^u(X_{(i)}(x)) - \eta^u(x) \bigr) \bigg\|_{\infty} + \biggl\| \sum_{i=1}^n \overline{V}^u_i(x) - 1 \biggr\|_{\infty}.
	\end{align*}
	Lemma \ref{lem::punder} implies that for all $n\geq n_1$, with probability at least $1 - 2 M / n^3$, there holds
	\begin{align*} 
		\|\overline{\eta}^{B,u} - \eta^u\|_{\infty}
		\leq 4 c_L \sup_{x \in \mathcal{X}} \biggl( \sum_{i=1}^n \overline{V}_i^u(x) \|X_{(i)}(x) - x\|^{\alpha} \biggr) 
		+ \biggl\| \sum_{i=1}^n \overline{V}_i^u(x) - 1 \biggr\|_{\infty}.
	\end{align*}
	Applying Lemma \ref{pro::detererror1}, we obtain 
	\begin{align}\label{equ::etambux}
		\|\overline{\eta}_m^{B,u}(x) - \eta^u_m(x)\|_{\infty}
		\leq c_3 (k/s)^{\alpha/d} + \exp \bigl( - (s - k)^2 /(2n) \bigr)
	\end{align}
	for all $n>n_3$
	with probability $\mathrm{P}^n$ at least $1-(4M+1)/n^3$. 
	Consequently, if $n\geq N_5:=\max\{4M+1,n_3\}$, then \eqref{equ::etambux} holds with probability $\mathrm{P}^n$ at least $1-1/n^2$. This completes the proof.
\end{proof}

\subsubsection{Proofs Related to Section \ref{sec::bagsampleerror}} \label{sec::proofbagsample}

To prove Proposition \ref{pro::tildeffunder}, we need the following lemma, which bounds the maximum value of the bagged weights $\overline{V}_i^u(x)$ defined by \eqref{equ::defbarviu}.

\begin{lemma}\label{lem::pijunder}
	Let $\overline{V}_i^u(x)$ be defined by \eqref{equ::defbarviu}. Then for any $x\in \mathbb{R}^d$, there holds
	\begin{align*}
		\max_{1 \leq i \leq n}  \overline{V}_i^u(x)
		\leq s / (k M n_{(1)}).
	\end{align*}
\end{lemma}

\begin{proof}[Proof of Lemma \ref{lem::pijunder}]
	By \eqref{equ::barvdef} and \eqref{equ::puijx}, we have
	\begin{align*}
		\overline{V}_i^u(x)
		& = \frac{1}{k} \sum^k_{j=1} p_{i,j}^u(x) 
		\\
		& = \frac{1}{k} \sum^k_{j=1} \mathrm{P}_{Z} \biggl( \sum^i_{\ell=1} Z^b(X_{(\ell)}(x),Y_{(\ell)}(x)) = j,  Z^b(X_{(i)}(x),Y_{(i)}(x)) = 1 \biggl| \{ (X_i,Y_i) \}_{i=1}^n \biggr)
		\\
		& \leq k^{-1} \mathrm{P}_{Z} \bigl( Z^b(X_{(i)}(x),Y_{(i)}(x)) = 1 \bigl| \{ (X_i,Y_i) \}_{i=1}^n \bigr)
		\\
		& = k^{-1} a(X_{(i)}(x),Y_{(i)}(x)) 
		\leq s / (M k n_{(1)}),
	\end{align*}
	which finishes the proof.
\end{proof}

\begin{proof}[Proof of Proposition \ref{pro::tildeffunder}]
	By the definition of $\widetilde{\eta}^{B,u}$ and $\overline{\eta}^{B,u}$, we have
	\begin{align*}
		\bigl| \widetilde{\eta}_m^{B,u} - \overline{\eta}_m^{B,u} \bigr|
		=\sum_{i=1}^n \overline{V}_i^u(x) \bigl( \eins \{ Y_{(i)}(x) = m \} - \eta_m^u(X_{(i)})(x) \bigr).
	\end{align*}
	For any fixed $x\in \mathcal{X}$, Lemmas \ref{lem::lemconcen} and \ref{lem::pijunder} yield
	\begin{align*}
		(\mathrm{P}^u)_{Y|X}^n \bigl( \bigl| \widetilde{\eta}_m^{B,u}(x)-\overline{\eta}_m^{B,u}(x) \bigr| \geq \varepsilon \big| D_n \bigr)
		\leq  2 \exp \bigl( - \varepsilon^2 k M n_{(1)} / (2s) \bigr).
	\end{align*}
	Setting $\varepsilon := \sqrt{2(2d+3)s\log n/(kMn_{(1)})}$, we get
	\begin{align}\label{equ::pointwise1}
		(\mathrm{P}^u)_{Y|X}^n \bigl( \bigl| \widetilde{\eta}_m^{B,u}(x)-\overline{\eta}_m^{B,u}(x) \bigr| \geq \varepsilon \big| D_n \bigr) \leq 2n^{-(2d+3)}.
	\end{align}
	Note that this inequality holds only for fixed $x$. In order to derive the uniform upper bound over $\mathcal{X}$, let
	$\mathcal{S} := \bigl\{ (\sigma_1, \ldots, \sigma_{n}) : \text{all permutations of } (1, \ldots, n)\text{ obtainable by moving } x \in \mathbb{R}^d \bigr\}$.
	Then we have
	\begin{align*}
		&	(\mathrm{P}^u)^{n}_{Y|X} \biggl( \sup_{x \in \mathbb{R}^d} \bigl( |\widetilde{\eta}^{B,u}_m(x) - \overline{\eta}_{m}^{B,u}(x)| - \varepsilon \bigr) > 0 \bigg| D_n\biggr)
		\\
		& \leq(\mathrm{P}^u)^n_{Y|X} \biggl( \bigcup_{(\sigma_1, \ldots, \sigma_n) \in \mathcal{S}} \biggl| \sum_{i=1}^n \overline{V}_{i,\sigma}^u (\eins \{ Y_{\sigma_i} = m \} - \eta_m^u(X_{\sigma_i})) \biggr| > \varepsilon \bigg| D_n\biggr)
		\\
		& \leq \sum_{(\sigma_1, \ldots, \sigma_n) \in \mathcal{S}} (\mathrm{P}^u)^{n}_{Y|X} \biggl( \biggl| \sum_{i=1}^n \overline{V}_{i,\sigma}^u (\eins \{ Y_{\sigma_i} = m \} - \eta^u_m(X_{\sigma_i})) \biggr| > \varepsilon \bigg| D_n \biggr),
	\end{align*}
	where $\overline{V}_{i,\sigma}^u(x) = k^{-1} \mathrm{P}_Z (\sum^i_{j=1} Z^b(X_{\sigma_j}(x),Y_{\sigma_j}(x))\leq k|\{X_i,Y_i\}_{i=1}^n)$.
	For any $(\sigma_1,\ldots,\sigma_n)\in \mathcal{S}$, \eqref{equ::pointwise1} implies
	\begin{align*}
		(\mathrm{P}^u)^{n}_{Y|X} \biggl( \biggl| \sum_{i=1}^n \overline{V}_{i,\sigma}^u (\eins \{ Y_{\sigma_i} = m \} - \eta_m^u(X_{\sigma_i})) \biggr| > \varepsilon \bigg| D_n\biggr)
		\leq 2 / n^{2d+3}.
	\end{align*}
	Combining this with Lemma \ref{lem::numberreorder}, we obtain
	\begin{align*}
		(\mathrm{P}^u)^{n}_{Y|X} \Bigl( \sup_{x \in \mathbb{R}^d} (|\widetilde{\eta}^{B,u}_m(x) - \overline{\eta}^{B,u}_{m}(x)| - \varepsilon) > 0 \Big| D_n\Bigr)
		\leq 2 (25/d)^d / n^3
	\end{align*}
	for $n\geq 2d$. Then a union bound argument yields 
	\begin{align*}
		(\mathrm{P}^u)^{n}_{Y|X}\bigl(
		\|\widetilde{\eta}^{B,u}-\overline{\eta}^{B,u}\|_{\infty}\leq \sqrt{2(2d+3) s \log n / (k M n_{(1)})} \big| D_n \bigr)
		\geq 1 - 2 M (25/d)^d / n^3.
	\end{align*}
	Consequently, by the law of total probability, we have
	\begin{align*}
		\mathrm{P}_Z^B\otimes \mathrm{P}^n\bigl(	\|\widetilde{\eta}^{B,u}-\overline{\eta}^{B,u}\|_{\infty}\leq \sqrt{2(2d+3)s \log n / ( k M n_{(1)})} \bigr) 
		\geq 1 - 2 M (25/d)^d / n^3.
	\end{align*}
	Therefore, if $n\geq N_6:=\max\{\lceil 2M(25/d)^d\rceil\}$, there holds
	\begin{align*}
		\mathrm{P}_Z^B\otimes \mathrm{P}^n\bigl(	\|\widetilde{\eta}^{B,u}-\overline{\eta}^{B,u}\|_{\infty}\leq \sqrt{2(2d+3)s \log n / ( k M n_{(1)})} \bigr) 
		\geq 1 - 1/n^2.
	\end{align*}
	Thus we complete the proof of Proposition \ref{pro::tildeffunder}.
\end{proof}

\subsubsection{Proofs Related to Section \ref{sec::knnmulticlassimbag}}

\begin{proof}[Proof of Theorem \ref{thm::bagknnclassund}] 
	Choosing $s$, $B$,  and $k$ according to \eqref{s::underbagging}, \eqref{B::underbagging}, and \eqref{k::underbagging}, respectively, 
	Propositions \ref{pro::fml*funder} , \ref{pro::tildefmlund}, \ref{pro::tildeffunder}  yield that if $n\geq \max\{N_4,N_5,N_6\}=\max\{N_5,N_6\}$, there holds
	\begin{align*}
		\|\widehat{\eta}^{B,u} - \eta^u\|_{\infty}
		& \lesssim \sqrt{\log n/(kB)} + (k/s)^{\alpha/d} + \exp \bigl( - (s - k)^2 / (2n) \bigr)+ \sqrt{s \log n/(kMn_{(1)})}
		\\
		& \lesssim \bigl( \log (M n_{(1)}) / (M n_{(1)}) \bigr)^{\alpha/(2\alpha+d)}
	\end{align*}
	with probability $\mathrm{P}_{Z}^B \otimes \mathrm{P}^n$ at least $1 -3/ n^2$.
	According to \eqref{equ::con1} in Lemma \ref{lem::punder}, we have $Mn_{(1)}\geq Mn\pi_1/2$. 
	Note that $g(x):=\log (x)/x$ is decreasing on $[e,\infty)$. 
	Consequently, we see that if $n \geq \max \{ N_5, N_6, 2 M, \lceil 2 e / (M \underline{\pi}) \rceil\}$, there holds
	\begin{align*}
		\|\widehat{\eta}^{B,u} - \eta^u\|_{\infty}\lesssim \bigl( \log (Mn\pi_1/2) / (Mn\pi_1/2) \bigr)^{\alpha/(2\alpha+d)}\lesssim (\log n/n)^{\alpha/(2\alpha+d)}.
	\end{align*}
	with probability $\mathrm{P}_{Z}^B \otimes \mathrm{P}^n$ at least $1-4/n^2$.
	This together with Proposition \ref{thm::punderund} yields that
	for all $n \geq N_2^* := \max \{ N_3, N_5, N_6, 2 M, \lceil 2 e / (M \underline{\pi}) \rceil\}$, there holds
	\begin{align*}
		\|\widehat{\eta}^{B,u} - \eta^w\|_{\infty}
		& = \|\widehat{\eta}^{B,u} - \eta^u\|_{\infty} + \|\eta^u - \eta^w\|_{\infty}
		\\
		& \lesssim (\log n/n)^{\alpha/(2\alpha+d)}+\sqrt{\log n/n}
		\lesssim (\log n/n)^{\alpha/(2\alpha+d)}
	\end{align*}
	with probability $\mathrm{P}_{Z}^B \otimes \mathrm{P}^n$ at least $1 -5/n^2$.
	Lemma \ref{lem::convergereg} yields that
	\begin{align*}
		\|\eta^{w,*}_{L_{\mathrm{cl}},\mathrm{P}^w}(x) - \eta^w_{\widehat{f}^{B,u}(x)}(x)\|_{\infty}
		\lesssim (\log n/n)^{\alpha/(2\alpha+d)}
	\end{align*}
	holds with probability $\mathrm{P}_{Z}^B \otimes \mathrm{P}^n$ at least $1 - 5 / n^2$, where $\eta^{w,*}_{L_{\mathrm{cl}},\mathrm{P}^w}(x) = {\eta}_{{f}_{L_{\mathrm{cl}},\mathrm{P}^w}^*(x)}(x)$, i.e,~the Bayes classifier w.r.t.~the classification loss $L_{\mathrm{cl}}$ and the balanced distribution $\mathrm{P}^w$. Consequently, Lemma \ref{lem::regexcessrisk} implies that
	\begin{align*}	
		\mathcal{R}_{L_{\mathrm{cl}},\mathrm{P}^w}(\widehat{f}^{B,u}) - \mathcal{R}_{L_{\mathrm{cl}},\mathrm{P}^w}^*
		\lesssim  (\log n/n)^{{\alpha(\beta+1)}/{(2\alpha+d)}}
	\end{align*}
	holds with probability $\mathrm{P}_Z^B\otimes \mathrm{P^n}$ at least $1 -5/ n^2$. By \eqref{equ::R_amf} and Theorem \ref{thm::connection}, we have $\mathfrak{R}_{\mathrm{AM}}(\widehat{f}^{B,u}) \lesssim  (\log n/n)^{{\alpha(\beta+1)}/{(2\alpha+d)}}$, which finishes the proof.
\end{proof}

\begin{proof}[Proof of Corollary \ref{thm::bag1nnund}]
	Taking $k = 1$ in Proposition \ref{pro::tildefmlund}, we get
	\begin{align*}
		\|\overline{\eta}^{b,u}_m(x)-\eta^u_m(x)\|_{\infty}
		& \leq \sum_{i=1}^n \overline{V}^u_i(x) R_{(i)}^{\alpha}(x) + \exp \bigl( - (s - 1)^2 / (2 n) \bigr)
		\\
		& \leq (2 n_1/n)^{\alpha/d} + c_3 s^{-\alpha/d} + \exp \bigl( - (s - 1)^2/(2 n)\bigr)
		\\
		& \leq c_3' (\log s/s)^{\alpha/d} + \exp \bigl( - (s - 1)^2/(2 n) \bigr),
	\end{align*}
	where the constant $c_3':=(12d+32)^{\alpha/d}+c_3$. 
	This together with
	Propositions \ref{pro::fml*funder} and \ref{pro::tildeffunder} yields that
	if $n\geq \max\{N_4,N_5,N_6\}=\max\{N_5,N_6\}$,
	there holds
	\begin{align*}
		\|\widehat{\eta}^{B,u} - \eta^u\|_{\infty}
		& \lesssim \sqrt{\log n/B} + (\log s/s)^{\alpha/d} + \exp \bigl( - (s - 1)^2/(2n) \bigr)+ \sqrt{s \log n / (Mn_{(1)})}
	\end{align*}
	with probability $\mathrm{P}_{Z}^B \otimes \mathrm{P}^n$ at least $1 - 3 / n^2$.
	
	If $d > 2 \alpha$, with $s = (Mn_{(1)})^{\frac{d}{2\alpha+d}} (\log (Mn_{(1)}))^{\frac{2\alpha-d}{2\alpha+d}}$, $B = (Mn_{(1)})^{\frac{2\alpha}{2\alpha+d}}(\log (Mn_{(1)}))^{\frac{d-2\alpha}{2\alpha+d}}$ we get
	\begin{align*}
		\|\widehat{\eta}^{B,u} - \eta^u\|_{\infty}
		\lesssim (\log^2 (Mn_{(1)})/{Mn_{(1)}})^{\alpha/(2\alpha+d)}.
	\end{align*}
	According to \eqref{equ::con1} in Lemma \ref{lem::punder}, we have $Mn_{(1)}\geq Mn\pi_1/2$. 
	Note that $g(x):=\log^2 (x)/x$ is decreasing on $[e^2,\infty)$. 
	Consequently,  if $n \geq \max \{ N_5, N_6, 2 M, \lceil 2 e^2 / (M \underline{\pi}) \rceil\}$, there holds
	\begin{align*}
		\|\widehat{\eta}^{B,u} - \eta^u\|_{\infty}\lesssim \bigl( \log^2 (Mn\pi_1/2) / (Mn\pi_1/2) \bigr)^{\alpha/(2\alpha+d)}\lesssim (\log n/n)^{\alpha/(2\alpha+d)}
	\end{align*}
	with probability $\mathrm{P}_{Z}^B \otimes \mathrm{P}^n$ at least $1-4/n^2$.
	This together with Proposition \ref{thm::punderund} yields that
	for all $n \geq N_3^*:=\max\{N_3,N_5,N_6,2M,\lceil 2e^3/(M\underline{\pi})\rceil\}$,
	there holds
	\begin{align*}
		\|\widehat{\eta}^{B,u} - \eta^w\|_{\infty}
		& = \|\widehat{\eta}^{B,u} - \eta^u\|_{\infty} + \|\eta^u - \eta^w\|_{\infty}
		\\
		& \lesssim (\log n/n)^{\alpha/(2\alpha+d)} + \sqrt{\log n/n}
		\lesssim (\log n/n)^{\alpha/(2\alpha+d)}
	\end{align*}
	with probability $\mathrm{P}_{Z}^B \otimes \mathrm{P}^n$ at least $1 -5/n^2$.
	
	Otherwise if $d \leq 2 \alpha$, with $s = (M n_{(1)} \log (M n_{(1)}))^{1/2}$ and $B = (M n_{(1)} /\log (M n_{(1)}))^{1/2}$, by similar arguments as above, for all $n \geq N_3^*$, there holds
	\begin{align*}
		\|\widehat{\eta}^{B,u} - \eta^w\|_{\infty}
		= \|\widehat{\eta}^{B,u} - \eta^u\|_{\infty} + \|\eta^u - \eta^w\|_{\infty}
		\lesssim \max \bigl\{ (\log n/n)^{\alpha/(2d)}, (\log^3 n/n)^{1/4} \bigr\}
	\end{align*}
	with probability $\mathrm{P}_{Z}^B \otimes \mathrm{P}^n$ at least $1 -5/n^2$.
	By exploiting similar arguments as that in the proof of Theorem \ref{thm::bagknnclassund}, we obtain the assertion.
\end{proof}

\section{Conclusion}\label{sec::conclusion}

In this paper, we conduct a learning theory analysis of the under-bagging $k$-NN algorithm for the imbalanced classification problem. By assuming the H\"{o}lder smoothness and margin condition, we establish optimal convergence rates for under-bagging $k$-NN classifier w.r.t.~the AM measure, a frequently used performance for imbalanced classification, based on the proposed statistical learning treatment. Through our theoretical analysis we show that with proper parameter selections, lower time complexity are required for the under-bagging $k$-NN compared with the standard $k$-NN. These findings in return unveil the working mechanism of under-bagging for imbalanced classification. Therefore, we believe that our work sheds light on developing learning theory analysis of under-bagging algorithms with other base classifiers such as support vector machines, decision trees, and neural networks.

\bibliographystyle{plain}
\small{\bibliography{UNNIC}}
\end{document}